\documentclass[11pt]{article}
\usepackage{CJK}
\usepackage[utf8]{inputenc}

\usepackage{amssymb,verbatim}
\usepackage{epsfig}
\usepackage{fancyhdr}
\usepackage{pdflscape}
\usepackage{bm}
\usepackage{booktabs}
\usepackage{enumerate}
\usepackage{graphicx}
\usepackage{array}
\usepackage{amstext}
\usepackage{amsthm}
\usepackage{amsopn}
\usepackage{amsfonts}
\usepackage{amsmath}
\usepackage{mathrsfs}
\usepackage{amsbsy}
\usepackage{tablefootnote}
\usepackage{mathtools}
\usepackage{multicol}
\usepackage{multirow}
\usepackage{colortbl}
\usepackage{xcolor}
\usepackage{appendix}
\usepackage{caption}
\usepackage{bbm}
\usepackage{algorithm,algpseudocode}
\usepackage{threeparttable}
\usepackage[final]{pdfpages}

\usepackage[authoryear,round]{natbib}
\usepackage[resetlabels]{multibib}
\usepackage{setspace}
\renewcommand{\baselinestretch}{2}

\theoremstyle{plain}
\newtheorem{theorem}{Theorem}
\newtheorem{lemma}{Lemma}
\newtheorem{corollary}{Corollary}
\newtheorem{remark}{Remark}
\newtheorem{definition}{Definition}

\newtheorem{proposition}{Proposition}
\newtheorem{condition}{Condition}
\newtheorem{model}{M}

\newcommand{\bbR}{\mathbb{R}}
\newcommand{\bbE}{\mathbb{E}}
\newcommand{\bbN}{\mathbb{N}}
\newcommand{\btheta}{\bm{\theta}}
\newcommand{\bphi}{\bm{\phi}}
\newcommand{\bgamma}{\bm{\gamma}}

\newcommand{\calX}{\mathcal{X}}
\newcommand{\calZ}{\mathcal{Z}}
\newcommand{\calW}{\mathcal{W}}
\newcommand{\calL}{\mathcal{L}}
\newcommand{\calG}{\mathcal{G}}
\newcommand{\calF}{\mathcal{F}}

\newcommand{\bX}{\bm{X}}
\newcommand{\bx}{\bm{x}}
\newcommand{\bZ}{\bm{Z}}
\newcommand{\bz}{\bm{z}}

\title{Penalized Generative Variable Selection}
\author{Tong Wang\thanks{Department of Biostatistics, Yale School of Public Health, New Haven, Connecticut, USA}
\ Jian Huang
\thanks{Department of Applied Mathematics, The Hong Kong Polytechnic University, Hong Kong SAR, China}
\ Shuangge Ma
\thanks{Department of Biostatistics, Yale School of Public Health, New Haven, Connecticut, USA}}

\date{}

\begin{document}

\maketitle

\begin{abstract}
Deep networks are increasingly applied to a wide variety of data, including data with high-dimensional predictors. In such analysis, variable selection can be needed along with estimation/model building. Many of the existing deep network studies that incorporate variable selection have been limited to methodological and numerical developments. In this study, we consider modeling/estimation using the conditional Wasserstein Generative Adversarial networks. Group Lasso penalization is applied for variable selection, which may improve model estimation/prediction, interpretability, stability, etc. Significantly advancing from the existing literature, the analysis of censored survival data is also considered. We establish the convergence rate for variable selection while considering the approximation error, and obtain a more efficient distribution estimation. Simulations and the analysis of real experimental data demonstrate satisfactory practical utility of the proposed analysis.     
\end{abstract}


\section{Introduction}\label{sec:intro}
In data analysis, an essential task is to establish the relationship between a response $Y$ and predictors $\bX$. Numerous regression techniques have been developed for this. It has been increasingly recognized that regression sometimes can be too rigid, and over the past decade, we have witnessed a surge in deep network techniques, which have demonstrated great power in a wide range of domains, such as computer vision, speech recognition, biomedical diagnosis, natural language processing, and others. ``Traditional’’ deep networks have often been criticized for being ``black-boxes’’ with somewhat limited statistical principles, theoretical guarantees, and interpretability. 

The fast development of data collection techniques in many domains has made the analysis of data with high-dimensional predictors routine. With such data, model sparsity is often assumed, under which only a subset of the predictors is relevant. For many specific problems, this assumption has been suggested as sensible. Accordingly, variable selection has been  conducted, which can lead to more reliable, interpretable, and efficient estimation. Under the regression framework, various penalization techniques have been developed and shown to have satisfactory theoretical and numerical properties \citep{zou2005regularization, SpAM2007, ZhangCH2010}. Alternative selection techniques include Bayesian, thresholding, boosting, and others. 

Deep network techniques are now increasingly applied to data with high-dimensional predictors. A natural question is whether variable selection can be incorporated in such analysis. Compared to their regression counterparts, deep networks have more complicated model structures and many more parameters, leading to a stronger demand for variable selection. Here it is first noted that some pruning techniques, such as random dropout, can generate sparser networks but cannot conduct variable selection. In the literature, there have been a few efforts. For example, one strategy is to conduct variable selection via constructing a sparse network. 
With imaging data, \cite{zhao2015heterogeneous} and \cite{scardapane2017group} applied group Lasso penalization to either a specific module or the entire network. \cite{feng2017sparse,dinh2020consistent} and \cite{luo2023sparseinput} studied the sparse input neural network and developed the corresponding algorithms. \cite{lemhadri2021lassonet} developed LassoNet by adding a skip (residual) layer, which can control the participation of a feature in any hidden layer. As an alternative strategy, researchers have developed new network architectures for variable selection. For example, \cite{li2016deep} performed selection by adding a sparse one-to-one layer to networks, and \cite{chen2021nonlinear} established its selection consistency under some conditions with a greedy algorithm they proposed. \cite{lu2018deeppink} designed a new deep neural network architecture by integrating knockoff filters to achieve feature selection with a controlled false discovery rate. \cite{liang2018bayesian} developed a neural network variable selection method under the Bayesian paradigm. 

The aforementioned and other studies have demonstrated promising performance of deep networks, without and with variable selection. Many of them have been limited to methodological and numerical developments, while insufficient theoretical guarantee has been provided. Nevertheless, theoretical investigation has been conducted in a handful of studies. For example, when variable selection is not incorporated, convergence properties of the ReLU neural networks have been studied under least squares regression \citep{Bauer2019,Schmidt-Hieber2020}, quantile regression \citep{Padilla2020QuantileRW}, robust nonparametric regression \citep{shen2021robust}, and a general class of nonparametric regression-type loss functions \citep{Farrell2021DNN}. Effort has also been made when variable selection is incorporated. \cite{Yan2022} established selection consistency for Bayesian sparse deep neural networks based on marginal posterior inclusion probabilities. \cite{chen2021nonlinear} established selection consistency with  approximation error for the network with a one-to-one selection layer under a regression framework. The aforementioned and some other studies have examined selection consistency with respect to certain specifically designed network architectures (and hence may not be generalizable). There are also studies on the commonly adopted fully connected networks. For example, \cite{feng2017sparse} examined variable selection consistency for single-layer neural networks. Additionally, \cite{dinh2020consistent} examined selection properties from a parametric perspective through setting the underlying regression function as the neural network.

The goal of this study is to further advance the research on variable selection with deep networks. This study is connected to but also significantly differs from the existing literature in multiple important aspects. Specifically, to more comprehensively describe the response-predictor relationship, we consider the Wasserstein generative adversarial network (WGAN), which significantly differ from the networks that are limited to certain specific quantities such as conditional mean. 
It makes milder assumptions on data structures and can be more suitable for modeling complex and unknown relationships. 
For variable selection, we adopt penalization, which has been demonstrated as effective in multiple studies. Different from some studies that have an additional sparse layer, group penalization is imposed to the first-layer weights. Additionally, there are multiple notable differences in the statistical development. For example, we consider deep neural networks whose width and depth can increase with sample size – this significantly differs from those studies with fixed networks. Different from \cite{dinh2020consistent} and some others, no assumption is imposed on the underlying regression function, making the proposed networks more flexible. Approximation error, which has a crucial role in nonparametric estimation but has been ignored in \cite{feng2017sparse, dinh2020consistent} and some other studies, is taken into consideration. Most of the existing studies, in particular those incorporating variable selection, have been limited to continuous and categorical responses. This study further advances by also considering censored survival data. Overall, this study can deliver a practically useful tool with a strong statistical ground, advancing deep network methodological and theoretical research.

\section{Methods}\label{sec:method}

Let $\mathcal{A}=\{(\bx_i,y_i):i=1,\ldots,n\}$ denote a dataset of $n$ i.i.d. observations, where the predictor $\bx_i\in\bbR^p$ and the response  $y_i\in\bbR$. %$\bx_i$ can have both continuous and categorical components. 
To mitigate data-induced correlations, we split $\mathcal{A}$ into two subsets: $\mathcal{A}_1=\{(\bx_i,y_i):i=1,\ldots,n_1\}$ and $\mathcal{A}_2=\{(\bx_i,y_i):i=n_1+1,\ldots,n_1+n_2\}$ with $n_1+n_2=n$.

\subsection{Penalized WGAN}
\label{sec:P-WGAN}

The proposed analysis consists of two stages/steps. In both stages, we adopt WGAN to describe the response-predictor relationship. Compared to those focused on certain specific quantities (for example, conditional mean), GWAN is built on distributions and can more comprehensively measure model fitting. Its advantages have been increasingly recognized in the recent literature. In the first stage, the first subset of data is analyzed, and the main goal is to identify the predictors that are important for the response. Penalization is applied for variable selection. In the second stage, the second subset of data is analyzed. With only the selected variables, a refined estimation is conducted. Here no variable selection is needed, and hence penalization is not imposed. 

\noindent{\bf Stage 1}  WGAN corresponds to a minimax two-player game, and two networks are involved in training: a generator network $g_{\btheta_1}$ with predictors $\bx$ and a noise vector $\bz$ as the input, and a discriminator network $f_{\bphi_1}$ with $(\bx,y)$ as the input, where $\btheta_1$ and $\bphi_1$ are the network parameters (weights and biases). We denote the first-layer weights in $g_{\btheta_1}$ by $\bm{w}_{0}(\btheta_1)\in\bbR^{p_1\times(p+m)}$, where $p_1$ is the width of the first hidden layer and $m$ is the dimension of $\bz$, which is generated from  the standard normal distribution (noting that some other distributions are also applicable).

Using only $\mathcal{A}_1$, we propose the penalized objective function:
\begin{align}\label{eq:emp-stage1}
     l_1(g_{\btheta_1},f_{\bphi_1})=
        \frac{1}{n_1}\sum_{i=1}^{n_1}f_{\bphi_1}(\bx_i,g_{\btheta_1}(\bx_i,\bz_i))-\frac{1}{n_1}\sum_{i=1}^{n_1}f_{\bphi_1}(\bx_i,y_i)+{\lambda_n}\sum_{j=1}^{p}\|\bm{w}_{0}(\btheta_1)^{[,j]}\|_2,
\end{align}
where ${\lambda_n}$ is a data-dependent tuning parameter and $\bm{w}_{0}(\btheta_1)^{[,j]}$ is the $j$-th column of $\bm{w}_{0}(\btheta_1)$, which connects the $j$-th component of input $\bx_i$. The proposed estimation is defined as:
\begin{align}\label{eq:est-stage1}
(\hat{\btheta}_1,\hat{\bphi_1})=\arg\min_{\btheta_1}\max_{\bphi_1}l_1(g_{\btheta_1},f_{\bphi_1}).
\end{align}
With the estimated weights, the network structures can be obtained. The input variables with the estimated first-layer weights being nonzero are identified as important for the response -- this is consistent with penalized regression and some existing penalized deep network selection studies. 
Overall, the proposed training process is sketched in Figure 
\ref{fig:P-WGAN}.

\begin{figure}[h]
    \centering
\begin{tabular}{@{}c@{}}
    \includegraphics[width=0.55\textheight]{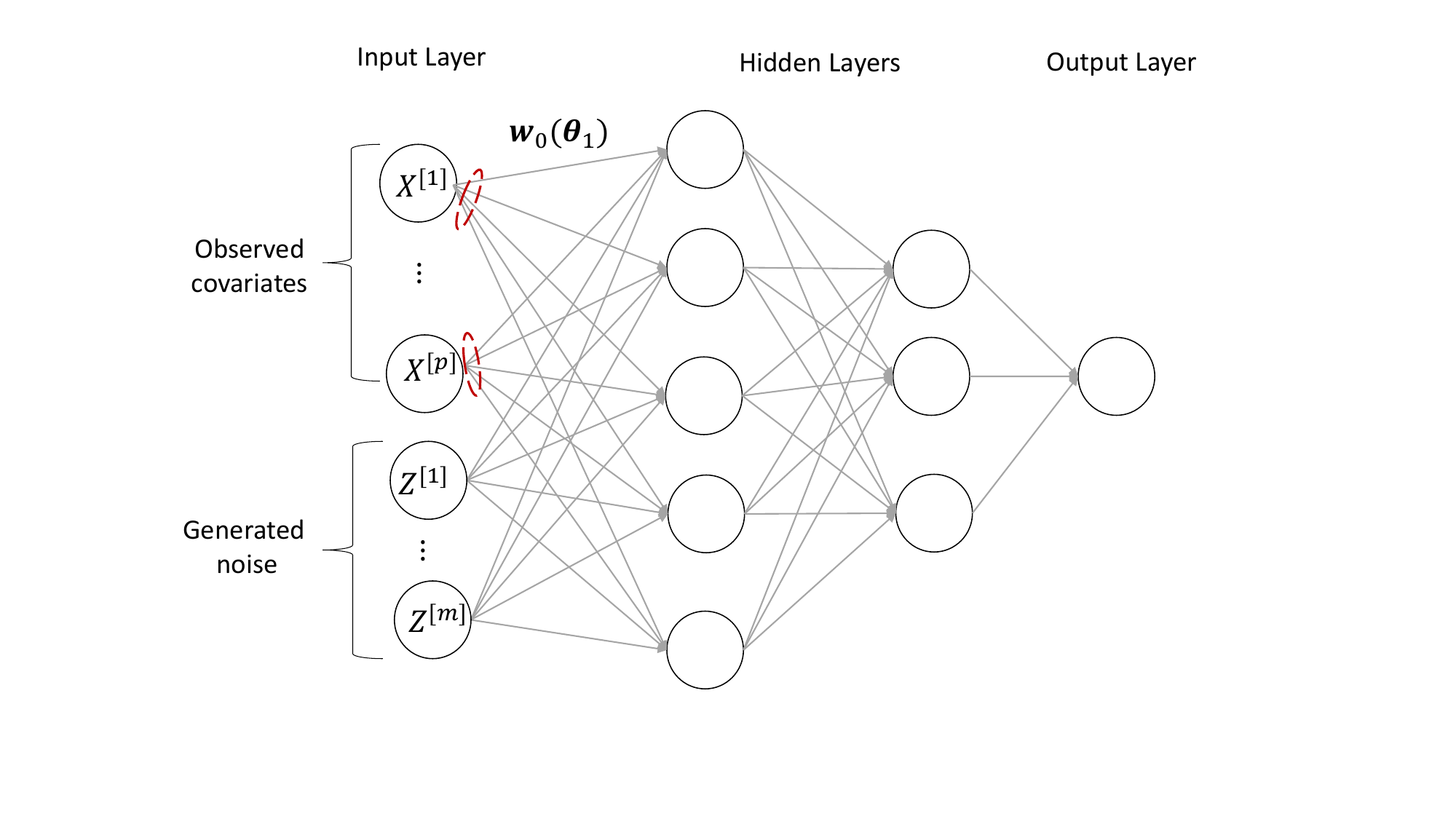}
\end{tabular}
\begin{tabular}{@{}c@{}}
    \includegraphics[width=0.55\textheight]{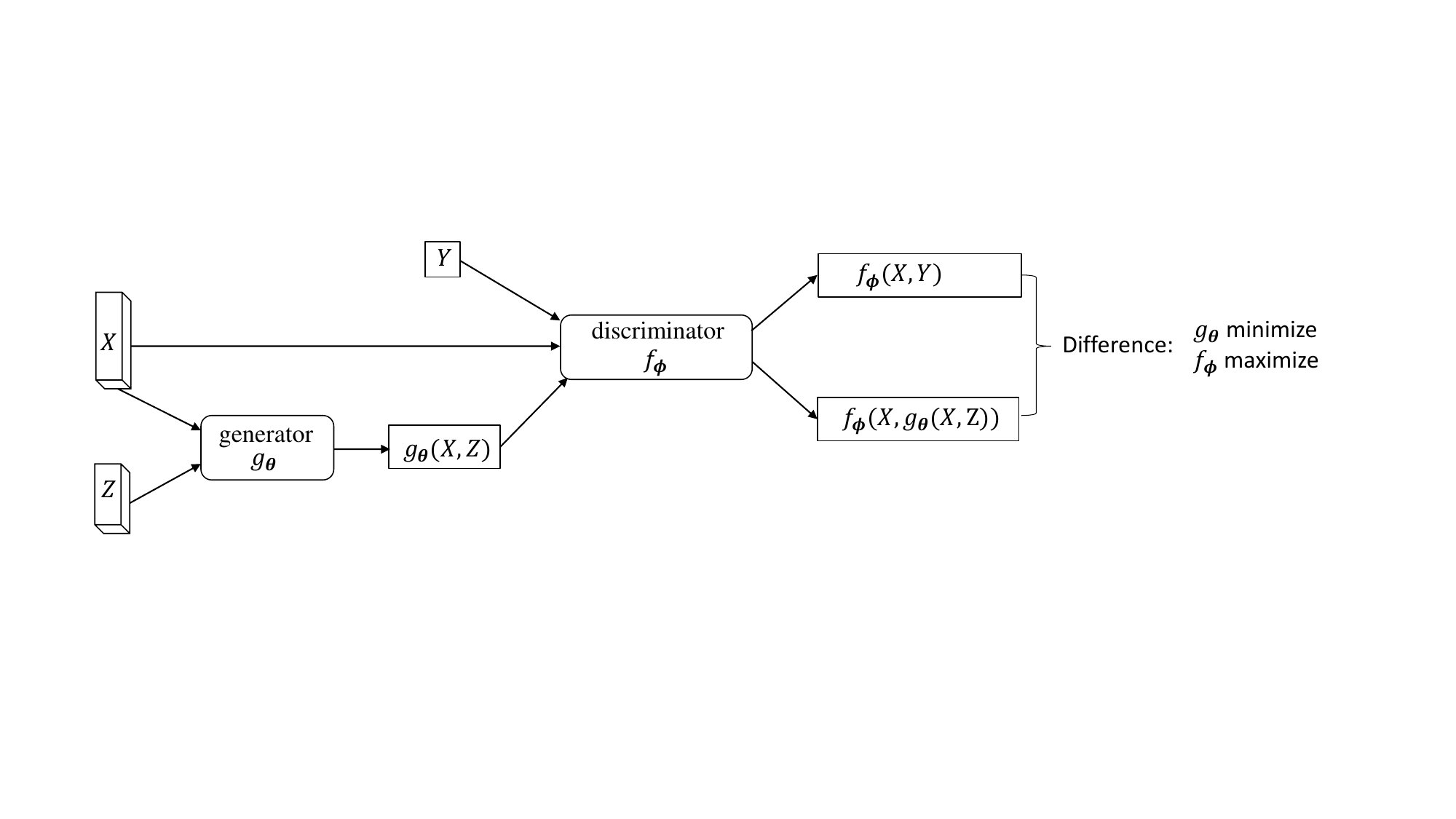}
\end{tabular}
\caption{Upper: Architecture of the penalized generator network $g_{\btheta_1}$. Lower: Training process of the proposed method.}
\label{fig:P-WGAN}
\end{figure}

\noindent\underline{Rationale}
The first two terms in (\ref{eq:emp-stage1}) empirically measure the 1-Wasserstein distance between the joint distributions of $(X,g_{\btheta_1}(\bX,\bZ))$ and $(\bX,Y)$, denoted by $P_{\bX,g_{\btheta_1}}$ and $P_{\bX,Y}$, respectively. It can be seen from (\ref{eq:est-stage1}), traning a WGAN amounts to finding the Nash equilibrium between $g_{\btheta_1}$ and $f_{\bphi_1}$, with $g_{\btheta_1}$ mimicking $P_{Y|\bX}$ and $f_{\bphi_1}$ discriminating the difference between $P_{\bX,g_{\btheta_1}}$ and $P_{\bX,Y}$.
The minimum of this distance is achieved when $P_{\bX,g_{\btheta_1}}$ and $P_{\bX,Y}$ are identical, which indicates that $g_{\btheta_1}(\bx,\bZ)$ follows the conditional distribution of $Y$ given $\bX=\bx$. 

The last term in (\ref{eq:emp-stage1}) is a group Lasso penalty, designed for identifying important input variables. Penalization has been adopted in multiple deep network studies for variable selection. Although sharing a similar spirit, different studies have different penalization designs. Here, the adopted penalty treats all the first-layer weights corresponding to an input variable as a group. With its all-in-or-all-out property, if a group is estimated as zero, the corresponding variable can be dropped from the model. This cannot be achieved using for example Lasso. On the other hand, if a group is estimated as nonzero, the corresponding input variable has all the first-layer weights nonzero, allowing for sufficient contributions. This strategy differs from that involving an additional sparse layer. Our preliminary examination does not suggest the dominance of one strategy over the other. As group penalization is adopted for the purpose of variable selection, it is only applied to the first layer of $g_{\btheta_1}$. Some published studies have also imposed penalization  to the other layers to regularize estimation and/or achieve sparsity. If desirable, this can be coupled with the proposed approach too.

\noindent{\bf Stage 2}
To further refine estimation, network retraining is conducted using only the identified important input variables. Such retraining has been commonly conducted with penalized regression. Denote $\hat{\bX}^s$. With the low dimensionality of $\hat{\bX}^s$ and the focus on estimation, no penalization is imposed in this stage. 
Using $\mathcal{A}_2$ only, we consider the un-penalized objective function: 
\begin{align} \label{eq:emp-stage2}  
     l_2(g_{\btheta_2},f_{\bphi_2})=
        \frac{1}{n_2}\sum_{i=1}^{n_2}f_{\bphi_2}(\tilde{\bx}^{s}_{n_1+i},g_{\btheta_2}(\tilde{\bx}^{s}_{n_1+i},\bz_i))-\frac{1}{n_2}\sum_{i=1}^{n_2}f_{\bphi_2}(\tilde{\bx}^{s}_{n_1+i},y_{n_1+i}).
       \end{align}
where $g_{\btheta_2}$ and $f_{\bphi_2}$ are a generator network and a discriminator network with parameters $\btheta_2$ and $\bphi_2$, respectively. The proposed estimation is defined as: $(\hat{\btheta}_2,\hat{\bphi}_2)=\arg\min_{\btheta_2}\max_{\bphi_2}l_2(g_{\btheta_2},f_{\bphi_2}).$

\begin{remark}
Splitting data leads to a reduction in sample size. In principle, it is possible to use the whole dataset for both stages. However, this may make theoretical development much more challenging. Similar data splitting has been adopted in multiple published studies.
\end{remark}

\subsection{Implementation}\label{sec:imple}

Here, we present the implementation details including the network architecture and the computational algorithm. Throughout this article, we consider the feedforward neural network (FNN), which can be expressed as a composition of a series of functions: 
\begin{align*}
\zeta(\bx)=\bm{w}_{\calL}\sigma(\bm{w}_{\calL-1}\sigma(\ldots\bm{w}_1\sigma(\bm{w}_0\bx+\bm{b}_0)+\bm{b}_1)+\ldots)+\bm{b}_{\calL-1})+\bm{b}_{\calL},\quad \bx\in\bbR^{p_0}.
\end{align*}
Here $\sigma(\cdot)$ is the activation function, $\calL$ is depth of the network,  $p_0$ is the dimension of the input vector, 
$p_{i}$ is the width of the $i$-th layer for $i=1,\ldots,\calL$,  
and $\bm{w}_i\in\bbR^{p_{i+1}\times p_{i}}$ and $\bm{b}_i\in\bbR^{p_{i+1}}$ are the weight matrix and bias vector of the $i$-th layer, respectively.  The width of $\zeta(\cdot)$ is described as $\calW=\max\{p_1,\ldots,p_{\calL}\}$.

To emphasize penalization in the objective function, the first-layer weights of the generator network $g_{\btheta_1}$ are denoted by $\bm{w}_{0}(\btheta_1)$. $\bm{w}_{0}(\btheta_1)\in\bbR^{p_1\times (p+m)}$ includes two parts: the first $p$ columns connecting the predictors $\bX$ and the rest connecting the generated noise $\bZ$. 
Algorithm \ref{alg: P-WGAN} presents the most essential component of the computation. Weight clipping is applied during the training of the discriminator network $f_{\bphi_1}$ to satisfy the Lipschitz constraint required by the Wasserstein distance. Although \cite{gulrajani2017improved} showed that a gradient penalty can enforce this Lipschitz constraint better in training WGANs, weight clipping performs comparably and is computationally easier for the present problem. The second step computation is omitted as it is a simplified version of Algorithm \ref{alg: P-WGAN}. 

 \begin{algorithm}[h]
	\caption{Penalized WGAN}
	
	\begin{algorithmic}[1]
	\Require{ Tuning parameter ${\lambda_n}$;  Minibatch size $n_b\leq n_1$; Clipping parameter $c$; Noise dimension $m$; Learning rates $r_f$, $r_g$. %; $\tilde{n}$, the batch size.
	}
	\For {number of training iterations in stage 1}
        \State Sample $\{(\bx_{bj},y_{bj})\}_{j=1}^{n_b}$ from $\{(\bx_i,y_i)\}_{i=1}^{n_1}$, and $\{\bz_j\}_{j=1}^{n_b}$ from $N(\bm{0},\bm{I}_{m})$.
		\State Update the discriminator $f_{\bphi_1}$ by ascending its stochastic gradient:
	\begin{align*}%\label{Alg-D}
	 & f_{\bphi_1}\gets\nabla_{\bphi_1}\left[ \frac{1}{n_b} \sum_{j=1}^{n_b} f_{\bphi_1}\left(\bx_{bj}, {g}_{\btheta_1}\left( \bx_{bj},\bz_{j}\right)\right)-\frac{1}{n_b} \sum_{j=1}^{n_b}f_{\bphi_1}\left(\bx_{bj}, y_{bj}\right)\right],\\
    &\bphi_1 \gets \bphi_1 + r_f\cdot \text{RMSProp}(\bphi_1,f_{\bphi_1}), \quad \bphi_1 \gets \text{clip}(\bphi_1,-c,c).
	\end{align*}
		\State Update the generator $g_{\btheta_1}$ by descending its stochastic gradient:
	\begin{align*}%\label{Alg-F}
	 &g_{\btheta_1}\gets \nabla_{\btheta} \left[\frac{1}{n_b} \sum_{j=1}^{n_b} {f}_{\bphi}(\bx_{bj}, g_{\boldsymbol{\btheta}}( \bx_{bj},\bz_{j}))+{\lambda_n}\sum_{l=1}^p\|\bm{w}_{0}(\btheta_1)^{[,l]}\|_2 \right],\\
    &\btheta_1 \gets \btheta_1 - r_g \cdot \text{RMSProp}(\btheta,g_{\btheta_1}).
	\end{align*}
	\EndFor
	\end{algorithmic}\label{alg: P-WGAN}
 \footnotesize
\end{algorithm}

\section{Statistical Properties}\label{sec:thm}

\subsection{Definition of importance}\label{sec:DefSign}
Let $Y\in\mathcal{Y}\subseteq\bbR$ and $\bX\in\mathcal{X}\subseteq\bbR^p$.
Since the conditional distribution $P_{Y|\bX}$ can fully describe the relationship between $\bX$ and $Y$, it is natural to define 
the important variables as the subset of $\bX$ that $P_{Y|\bX}$ only depends on.

\begin{definition}[Importance]\label{def:significance}
    Let $\bX^{[j]}$ be the $j$-th component of $\bX$. Then, $X^{[j]}$ is defined as unimportant if and only if
    \begin{align}\label{eq:def-sign}
         P_{Y|\bX^{[-j]}=\bx^{[-j]}}\sim P_{Y|\bX=\bx}, \quad \forall \bx\in\calX,
    \end{align}
    where $\bX^{[-j]}=(X^{[1]},\ldots,X^{[j-1]},X^{[j+1]},\ldots,X^{[p]})$ and $\bx^{[-j]}=(x^{[1]},\ldots,x^{[j-1]},x^{[j+1]},\ldots,x^{[p]})$, and ``$\sim$" means the equivalence of two distributions. 
\end{definition}

We separate $\bX$ into $\bX^{s}\in\bbR^{p_s}$ and $\bX^{c}\in\bbR^{p_c}$, with $p_s+p_c=p$, denoting the important and unimportant variables, respectively. Without loss of generality, assume that the first $p_s$ components of $\bX$ are important. For any $\bx\in\bX$, $P_{Y|\bX=\bx}\sim P_{Y|\bX^{s}=\bx^{s}}$. Note that this definition of importance allows for correlations among the predictors. It is a general definition of sparsity and fits many commonly adopted models. For example, for the mean regression model, $\bX^{[j]}$ is unimportant if and only if $\bbE[Y|\bX^{[j]}]=\bbE[Y|\bX]$. For the quantile regression model, $\bX^{[j]}$ is unimportant at a particular quantile level $\tau$ if and only if $Q_{Y|\bX^{[j]}}(\tau)=Q_{Y|\bX}(\tau)$. For the Cox regression model, $\bX^{[j]}$ is unimportant if and only if the hazard function $\lambda({Y|\bX^{[j]}})=\lambda({Y|\bX})$. Definition \ref{def:significance} is a sufficient condition for all the above.

\subsection{Connection between the important variables and target parameters}\label{subsec:THM-connection}   

To identify the important subset $\bX^{s}$, we aim to estimate a function $g(\bX=\bx,\bZ)$ which follows $P_{Y|\bX=\bx}$ for any $\bx\in\calX$, where $\bZ$ is a known random  vector.
The noise-outsourcing lemma \citep{kallenberg2002foundations} guarantees the existence of such a function $g(\bX,\bZ)$, as long as $\mathcal{Y}$ is a standard Borel space and $\bZ\in\calZ\subseteq \bbR^{m}$ is an independent $m$-dimensional noise vector sampled from $N(\bm{0},\bm{I}_{m})$ with $m\geq 1$.
Motivated by \cite{liu2021}, the estimation of  $g$ can be achieved through matching $P_{\bX,g}$ and $P_{\bX,Y}$ using the dual form of the 1-Wasserstein distance \citep{villani2009optimal}:
\begin{align}\label{eq:W1}
    W\left(P_{\bX, g}, P_{\bX, Y}\right)=\sup _{f\in\calF_{\text{Lip}}^1}\left\{\bbE_{\bX, \bZ} f(\bX, g(\bX,\bZ))-\bbE_{\bX, Y} f(\bX, Y)\right\}.
\end{align}
where $\calF_{\text{Lip}}^1$ is the 1-Lipschitz function class.
Let $g^{\ast}$ be the minimizer of $ W\left(P_{\bX, g}, P_{\bX, Y}\right)$ over the Borel space, so that $W\left(P_{\bX, g^{\ast}}, P_{\bX, Y}\right)=0$. $g^{\ast}$ satisfies
\begin{align}\label{eq:g-star}
    g^{\ast}(\bx,\bZ)\sim P_{Y|\bX=\bx}, \forall  \bx\in\mathcal{X},
\end{align}
as it is the sufficient and necessary condition for $W\left(P_{\bX, g}, P_{\bX, Y}\right)=0$. Together with (\ref{eq:g-star}) and Definition \ref{def:significance}, $g^{\ast}(\bX,\bZ)$ depends on $\bX^s$ only. We construct an estimator of $g^{\ast}$ within a function class consisting of feedforward neural networks, 
denoted by $\calG_{\btheta_1}$ with network parameter $\btheta_1$. 
To simplify presentation, we divide the first-layer weights of $g_{\btheta_1}$ into three parts:
\begin{align}
    \bm{w}_{0}(\btheta_1)=(\bm{\mu}_1,\bm{\nu}_1,\bm{\upsilon}_1),
\end{align}
where $\bm{\mu}_1\in\mathbb{R}^{p_1\times p_s}$ connects $\bX^{s}$, $\bm{\nu}_1\in\mathbb{R}^{p_1\times p_c}$ connects $\bX^{c}$, and $\bm{\upsilon}_1\in\mathbb{R}^{p_1\times m}$ connects $\bZ$.
Then, referring to the optimal network solution 
$g_{\btheta^{\ast}_1}\in \arg\min_{g_{\btheta_1}\in\calG_{\btheta_1}}W(P_{\bX,g_{\btheta_1}},P_{\bX,Y}),$
we define a set  of $\btheta$ as  
\begin{align}\label{eq:Theta-def}
     \Theta=\{\btheta, W(P_{\bX,g_{\btheta_1}(\bX,\bZ)},P_{\bX, Y})=W(P_{\bX, g_{\btheta^{\ast}_1}(\bX,\bZ)},P_{\bX, Y})\}.
\end{align}
With proper conditions on the data structure and network class, $ g_{\btheta^{\ast}_1}$ can approach $g^{\ast}$ close enough such that, within $ g_{\btheta^{\ast}_1}$'s first-layer, the sub-components connecting $\bX^{s}$ are non-zero, and those connecting $\bX^{c}$ are zero. Then, for any $\btheta\in\Theta$ with first-layer weights $\bm{w}_0(\btheta)=(\bm{\mu},\bm{\nu},\bm{\upsilon})$, there exists a positive constant $c_{\mu}>0$ such that
$  \min_{1\leq j\leq p_s}\|\bm{\mu}^{[,j]}\|_2\geq c_{\mu}.
$
Further, if we set $\bm{\nu}=\bm{0}$,  $\btheta\in\Theta$ still holds. These properties of $ g_{\btheta^{\ast}_1}$ and $\Theta$ are given in Lemmas \ref{lem:para-g-theta-ast}-\ref{lem:Dinh2020-expand} in Appendix. 

In practice, the distribution of $(\bX,Y)$ is unknown, and the network parameter $\hat{\btheta}_1$ obtained in (\ref{eq:est-stage1}) lies outside of $\Theta$. To measure the distance between $\btheta$ and $\Theta$, we define metric $d(\btheta,\Theta)$:
\begin{align}\label{d-theta-1}
    d(\btheta,\Theta)\coloneqq \min_{\bm{\vartheta}\in\Theta}\|\btheta-\bm{\bm{\vartheta}}\|_2,
\end{align}
where $\|\cdot\|_2$ is the Euclidean norm. Subsequently, $d(\hat{\btheta}_1,\Theta)\to 0$ is sufficient for $\hat{\bm{\mu}}$ to be bounded away from zero. Thus, we can obtain the selection results via establishing the convergence rate of the proposed penalized generator network $ g_{\hat{\btheta}_1}$ in terms of its parameters.

\subsection{Convergence rate}\label{subsec:THM-P-WGAN} 
Here, we describe the main theoretical results. The assumed conditions and detailed proofs are provided in Appendix. In what follows, the asymptotic notation $A\precsim B$ means that 
$A\leq CB$ for some constant $C>0$. 

\begin{theorem}\label{thm:vary-p}
Suppose that Conditions \ref{C:XY}-\ref{C:Lojasewica} hold. If $p=o(\log^{c} n_1)$ with $0<c<1$ and ${\lambda_n}=O({n_1}^{-\frac{1}{2(p+1)}})$,  $\hat{\btheta}_1$ defined in (\ref{eq:est-stage1}) with first-layer weights $\bm{w}_0(\hat{\btheta}_1)=(\hat{\bm{\mu}}_1,\hat{\bm{\nu}}_1,\hat{\bm{\upsilon}}_1)$ satisfies 
   \begin{align*}
       \bbE[d(\hat{\btheta}_1,\Theta)]\precsim\left({n_1}^{-\frac{1}{p+1}} \log {n_1}\right)^{\frac{1}{2(a-1)}}, \quad \mathbb{E}\|\hat{\bm{\nu}}_1\|_2\precsim  {n_1}^{-\frac{1}{2(a-1)(p+1)}} \log {n_1},
   \end{align*} 
where $a>2$ is a positive constant, and $\mathbb{E}$ is taken with respect to $\{(\bX_i,Y_i)\}_{i=1}^{n_1}$ and $\{\bZ_i\}_{i=1}^{n_1}$. 
\end{theorem}
%Both $\bbE[d(\hat{\btheta}_1,\Theta)]$ and $\mathbb{E}\|\hat{\bm{\nu}}_1\|_2$ converge to zero as ${n_1}\to \infty$.

\begin{corollary}\label{cor:fixed-p}
Suppose that Conditions \ref{C:XY}-\ref{C:Lojasewica} hold. If $p$ is fixed and  ${\lambda_n}=O({n_1}^{-\frac{1}{2(p+1)}})$, there exists a positive constant $a> 2$ such that 
\begin{align*}
\mathbb{E}[d(\hat{\btheta}_1,\Theta)]\precsim \left(\frac{\log n_1}{n_1}\right)^{\frac{1}{2(p+1)(a-1)}}, \quad
\mathbb{E}\|\hat{\bm{\nu}}_1\|_2\precsim  \left(\frac{\log n_1}{n_1}\right)^{\frac{1}{2(p+1)(a-1)}}.
\end{align*}
%both of which converge to zero as ${n_1}\to \infty$.
\end{corollary}

In the establishment of the convergence rate, Theorem \ref{thm:vary-p} is more general, while Corollary \ref{cor:fixed-p} is more specific to the case of a finite $p$. They show that $\bbE[d(\hat{\btheta}_1,\Theta)]\to 0$, which leads to that the important variables can be identified with a high probability. It is noted that the convergence rate does not depend on $m$. 

\begin{corollary}\label{cor:threshold}
Suppose that Conditions \ref{C:XY}-\ref{C:Lojasewica} hold. If $p=o(\log^c {n_1})$ with $c\in(0,1)$ and $\lambda_n=O({n_1}^{-\frac{1}{2(p+1)}})$, there exists a positive threshold $c_{p_s} $ depending on $p_s$ such that, for any $j=1,\ldots,p_s$ and $k=1,\ldots,p_c$,
\begin{align*}
   \Pr(\{\|\hat{\bm{\mu}}_1^{[,j]}\|_2\geq c_{p_s}\} \cap \{\|\hat{\bm{\nu}}_1^{[,k]}\|_2\leq c_{p_s}\}) \to 1.
\end{align*}
\end{corollary}

Define the empirically selected variables as $\hat{\bX}^s=\{\bX^{[j]}:\|\bm{w}_0(\hat{\btheta}_1)^{[,j]}\|_2\geq  c_{p_s}, j=1,\ldots,p\}$.
Corollary \ref{cor:threshold} suggests that there exists a threshold with which the penalized WGAN can select a model that only involves the important variables with probability converging to 1.  
Then, the convergence rate of the second stage generator network $g_{\hat{\btheta}_2}$ can be established in term of the Wasserstein distance. 

\begin{theorem}\label{thm2:convergence-rate-stage2}
Suppose that Conditions \ref{C:XY}-\ref{C:Lojasewica} hold. If $p=o(\log^c {n_1})$ with $c\in(0,1)$ and $\lambda_n=O({n_1}^{-\frac{1}{2(p+1)}})$, then with probability $1-c_{p_s}{n_1}^{-\frac{1}{2(a-1)(p+1)}} \log {n_1}$,   
\begin{align}
\mathbb{E}W(P_{\hat{\bX}^s,g_{\hat{\btheta}_2}},P_{\bX^s,Y})\precsim n_2^{-\frac{2}{p_s +1}}\log n_2,
\end{align}
 where $c_{p_s}$ is a positive constant depending on $p_s$.
\end{theorem}

\section{Accommodating survival data}\label{sec:P-WGAN-AFT}

Here, we extend the proposed approach to accommodate right censored survival response. Relatively, development of deep networks for survival data has been limited, especially under the WGAN framework and/or when variable selection is needed.

Let $T$ be the event time and $C$ be the censoring time. We observe $(Y,\Delta)$, where $Y=\min(T,C)$, $\Delta=\mathbbm{1}(T\leq C)$, and $\mathbbm{1}(\cdot)$ is the indicator function. Let $\bX\in\mathcal{X}\subset\bbR^p$ be the $p$-dimensional predictor vector. Coherent with the above methodological development, we propose the penalized objective function at the population level as: 
\begin{align}\label{pop-obj:AFT}
    L(g_{\btheta_1},f_{\bphi_1})=\mathbb{E}_{\bX,\bZ}f_{\bphi_1}(\bX,g_{\btheta_1}(\bX,\bZ))-\mathbb{E}_{\bX,T}f_{\bphi_1}(\bX,T)+{\lambda_n}\sum_{j=1}^p\left\|\bm{w}_{0}(\btheta_1)^{[,j]}\right\|_2.
\end{align}
The next step is to construct the corresponding empirical objective function based on the observed censored data $(\bX,Y,\Delta)$, for which we resort to the Kaplan-Meier weighting. 

Consider a random sample $\{(\bx_i,y_i,\Delta_i),i=1,\ldots,n\}$.
Studies such as \cite{stute1996} suggest that the joint distribution function of $(\bX,T)$, denoted by $F_{\bX,T}$, can be consistently estimated using the Kaplan-Meier estimator $\hat{F}_{\bX,T}$ defined as:
\begin{align}\label{eq:KM-estimator}
    \hat{F}_{\bX,T}(\bX,Y)=\sum_{i=1}^n \omega_{(i)}\mathbbm{1}\{\bX_{(i)}\leq \bx,Y_{(i)}\leq y\},
\end{align}
where $Y_{(1)}\leq \ldots\leq Y_{(n)}$ are the order statistics of $Y_i$'s, $\bX_{(1)},\ldots,\bX_{(n)}$ are the associated predictors, and $\omega_{(i)}$'s are the jumps in the Kaplan-Meier estimator and can be computed as:
\begin{align*}
\omega_{(1)}=\frac{\Delta_{(1)}}{n}, \quad \omega_{(i)}=\frac{\Delta_{(i)}}{n-i+1} \prod_{j=1}^{i-1}\left(\frac{n-j}{n-j+1}\right)^{\Delta_{(j)}},\ 
i=2, \ldots, n.
\end{align*}
Here $\Delta_{(1)},\ldots,\Delta_{(n)}$ are the associated indicators of the ordered $Y_i$'s.
Accordingly, the empirical version of (\ref{pop-obj:AFT}) is: 
\begin{align}\label{emp-obj:AFT}
        l_1(g_{\btheta_1},f_{\bphi_1})=\frac{1}{n_1}\sum_{i=1}^{n_1} f_{\bphi_1}(\bx_i,g_{\btheta_1}(\bx_i,\bz_i))-\sum_{i=1}^{n_1} \omega_{(i)} f_{\bphi_1}(\bx_{(i)},y_{(i)})+{\lambda_n}\sum_{j=1}^d\left\|\bm{w}_{0}(\btheta_1)^{[,j]}\right\|_2,
\end{align}
where $\mathbb{E}_{\bX,T}f_{\bphi_1}(\bX,T)$ is estimated by the $\omega_{(ni)}$-weighted mean. Here notations have similar implications as in the previous section. The identification of the important input variables can be conducted in the same manner as above. 

In the second stage of network retraining, we consider the objective function:
\begin{align*}
    l_2(g_{\btheta_2},f_{\bphi_2})=
        \frac{1}{n_2}\sum_{i=1}^{n_2}f_{\bphi_2}(\hat{\bx}_{n_1+i}^s,g_{\btheta_2}(\hat{\bx}_{n_1+i}^s,\bz_i))-\sum_{i=1}^{n_2}\omega_{(i)} f_{\bphi_2}(\hat{\bx}_{n_1+i}^s,y_{n_1+i}).
\end{align*}
The training process is slightly different from the previous section. Details are provided in Algorithm \ref{alg:survival} in Appendix.

To establish convergence rate as in Section \ref{sec:thm}, we first establish Proposition \ref{prop:surivival} to provide the upper bound of the Kaplan-Meier estimator $\hat{F}_{\bX,T}$ and the distribution of interest $F_{\bX,T}$ in terms of the 1-Wasserstein distance.

\begin{proposition}\label{prop:surivival}
Suppose that Conditions \ref{C6:censor-indep}-\ref{C7:integrability} hold and that $(\bX,T)$ satisfies Condition \ref{C:XY}.
Let $\hat{F}_{\bX,T}(\bx,t)=\sum_{i=1}^n w_i\mathbbm{1}\{\bX_{(i)}\leq \bx,Y_{(i)}\leq t\}$ be the corresponding empirical measure. Then, there exists a constant $C$ depending on $M_3$ such that 
   \begin{align}\label{eq:survival-rate}
       \mathbb{E}W(\hat{F}_{\bX,T},F_{\bX,T})\leq C n^{-\frac{1}{p+1}}.
   \end{align}
\end{proposition}
%{\color{red} $n$ or $n_1$: either way is fine. but needs clarification}

Proposition \ref{prop:surivival} is developed based on Proposition 3.1 in \cite{NEURIPS2020_Lu}, which is for data without censoring. Comparable to the result in \cite{NEURIPS2020_Lu}, (\ref{eq:survival-rate}) indicates that the rate, at which the Kaplan-Meier estimator converges to the truth in distribution, is the same as the empirical distribution estimator. Finally, with Proposition \ref{prop:surivival}, Theorem \ref{thm:vary-p} and Theorem \ref{thm2:convergence-rate-stage2} hold for the weighted penalized WGAN for survival data analysis.

\section{Simulation}\label{sec:simulation}

We conduct simulation to gauge performance of the proposed approach. 
For comparison, we consider the following highly relevant alternatives using neural networks: least squares with group Lasso penalization \citep[P-LS,][]{yuan2006model}, deep feature selection with a selection layer \citep[DFS,][]{chen2021nonlinear}, LassoNet \citep{lemhadri2021lassonet}, and the group concave regularization method \citep[GCRNet,][]{luo2023sparseinput}. Specifically, with P-LS, neural network is adopted for nonparametric estimation, and the group Lasso penalization is applied to the first-layer weights. GCRNet applies the minimax concave penalty (MCP) to the first layer-weights for feature selection. Both DFS and LassoNet have special network architectures. In particular, DFS adds an one-to-one selection layer to the network. LassoNet adds a single residual layer to the network, connecting predictors and response, and selection is achieved by allowing a feature to participate in any hidden layer only if its residual-layer representative is active.
Among the alternatives, P-LS can be applied to survival data by changing the least squares objective function to the weighted form corresponding to the AFT model. LassoNet and GCRNet can be applied to survival data by building on the Cox regression. DFS cannot be directly applied to survival data. 
%{\color{red} how selection is done with the alternatives? DFS requires a pre-fixed number of selected variables? need to clarify. and need to say which is not feasible in practice}
We acknowledge that the field is moving fast and there are other applicable alternatives. The above ones are argubly more popular, have competitive performance, and have been adopted in multiple studies. We implement the proposed method in Pytorch and use the stochastic gradient descent algorithm RMSprop in the training process. For the alternatives, computation follows the original developments. We simulate data from the following models:

\begin{model}\label{M1} A linear model with an additive standard normal error:
$Y=\bX^{\top}\bm{\beta}+\epsilon,\quad \text{ where }\bm{\beta}=(\bm{1}^{\top}_{30},\bm{0}^{\top}_{p_c})^{\top}\text{ and }\epsilon\sim N(0,1).$ 
\end{model}

\begin{model}\label{M2}
A linear model with an additive heavy-tail error:
$Y=\bX^{\top}\bm{\beta}+\epsilon,\quad \text{ where }\bm{\beta}=(\bm{1}^{\top}_{30},\bm{0}^{\top}_{p_c})^{\top}\text{ and }\epsilon\sim t(1).$
\end{model}

\begin{model}\label{M3}
A nonlinear model with an additive standard normal error:
    $Y=\bX^{\top}\bm{\beta}+\exp(X_2+\frac{1}{3}X_3)+\sin(X_4X_5)\epsilon,\quad \text{ where }\bm{\beta}=(\bm{1}^{\top}_{30},\bm{0}^{\top}_{p_c})^{\top}\text{ and }\epsilon\sim N(0,1).$
\end{model}

\begin{model}\label{M4}
    A nonlinear model in the form of a fully connected network with a product standard normal error:
    $Y=(\bm{w}_1\sigma(\bm{w}_0\bX+\bm{b}_0)+b_1)\epsilon,$
    where function $\sigma$ is component-wise defined as  $\sigma(X)=\max(0.01X,X)$, $\bm{w}_0\in\bbR^{16\times d}$ is a fixed matrix with the first 30 columns non-zero, $\bm{w}_1\in\bbR^{16}$, $\bm{b}_0\in\bbR^{16}$ and $b_1\in\bbR$ are fixed vectors -- specifically, they are generated from a standard normal distribution and then kept fixed, and $\epsilon\sim N(0,1)$.
\end{model}

\begin{model}\label{M5}
A nonlinear AFT (accelerated failure time) model with a predictor-dependent censoring mechanism:
$ T=|\bX^{\top}\bm{\beta}(1-\bX^{\top}\bm{\beta})|^{1/2}+|\bX^{\top}\bm{\beta}||\epsilon|,$    
where $C=4\exp(\bX^{\top}\bm{\beta})$, $\epsilon\sim N(0,1)$, and $\bm{\beta}=(\bm{1}^{\top}_{30}, \bm{0}^{\top}_{p_c})^{\top}$.
\end{model}

\begin{model}\label{M6}
    A nonlinear AFT model with a predictor-dependent censoring mechanism:
    \begin{align*}
        T=\left\{
        \begin{array}{lcl}
            \exp(\sqrt{0.1|\bX^{\top}\bm{\beta}}|)-1.1+0.3|\epsilon|  & & 2\bX^{\top}\bm{\beta} \in(-\infty, -6.5] \\
            |0.7X_1^3+0.2X_2^2+0.3X_3+\epsilon| & & 2\bX^{\top}\bm{\beta}\in(-6.5 ,0] \\
            \exp(0.4\bX^{\top}\bm{\beta}+\epsilon) & & 2\bX^{\top}\bm{\beta}\in(0,6.5 ]\\
            |\log(3\bX^{\top}\bm{\beta}+\epsilon)| & & 2\bX^{\top}\bm{\beta}\in(6.5,\infty )
        \end{array}
        \right.,
    \end{align*}
   where $C=4\exp(\bX^{\top}\bm{\beta})$, $\epsilon\sim N(0,1)$, and $\bm{\beta}=(\bm{1}^{\top}_{30}, \bm{0}^{\top}_{p_c})^{\top}$.
\end{model}

For all the models, the predictors are generated from a multivariate standard normal distribution with dimension $p\in\{100,300,500,1000\}$. The censoring rates for M\ref{M5} and M\ref{M6} are about 45$\%$ and 40$\%$, respectively. Both the generator network and discriminator network have two hidden layers with widths (64, 32). The noise vector $\bZ\sim N(\bm{0},\bm{I}_{5})$. 
For the training datasets, we generate 1,000 samples for M\ref{M1} and M\ref{M3}, 10,000 samples for M\ref{M2} and M\ref{M4}, and 5,000 samples for M\ref{M5} and M\ref{M6}. The differences in sample size have been motivated by the differences in model complexity. The sizes of the validation datasets (which are used for network parameter tuning) and testing datasets (which are used for prediction performance evaluation) are 100 and 1,000, respectively. For each model, we simulate 100 replicates. 

Variable selection performance is evaluated using true positive rate (TPR) and false positive rate (FPR) which are defined as:
$\text{TPR} = \frac{|\hat{S}\cap S^{\star}|}{|S^{\star}|}, \text{ and } \text{FPR} = \frac{|\hat{S}\cap S^{c}|}{|S^{c}|},$
where $|\cdot|$ is the cardinality of a set, $\hat{S}$, $S^{\star}$, and $S^{c}$ are the sets of the empirically selected variables, the important variables, and the unimportant variables, respectively. 
Prediction is evaluated using the mean square error between the predicted and observed responses on the testing data, denoted by MSE. For continuous response, for example for the proposed approach, it is computed as $T^{-1}\sum_{i=1}^{T}\|Y_i-J^{-1}\sum_{j=1}^Jg_{\hat{\btheta}_2}(\bX_i,\bZ_j)\|_2^2$, where $g_{\hat{\btheta}_2}$ is the generator network function trained in the second stage, { $T$} is the size of the testing data, and $J$ is the size of the noise samples $\bz$.
For survival response, prediction is evaluated using the Concordance index (C-idx) computed as:
$    \text{C-idx} = \frac{\sum_{i\neq j} \Delta_i \mathbbm{1}(\hat{Y}_i\leq \hat{Y}_j)\{1-2\mathbbm{1}(Y_i\leq Y_i)\} }{\sum_{i\neq j} \Delta_i \mathbbm{1}(Y_i\leq Y_j)},$
where $\hat{Y}_i$ is the predicted survival time calculated as $\hat{Y}_i=J^{-1}\sum_{j=1}^{J}g_{\hat{\btheta}_2}(\bX_i,\bZ_j)$ with $J=50$.

\begin{table}[t]
    \centering
    \caption{Simulation results for variable selection and prediction with $p_s =30$.}
    \begin{threeparttable}
    \resizebox{\textwidth}{!}{\begin{minipage}{\textwidth}
    \begin{tabular}{cc ccc c ccc c ccc c ccc c ccc }
      \hline
      &&\multicolumn{3}{c}{Proposed} & &\multicolumn{3}{c}{P-LS}& & \multicolumn{3}{c}{DFS} & & \multicolumn{3}{c}{LassoNet}& & \multicolumn{3}{c}{GCRNet}\\
      \cline{3-5}\cline{7-9} \cline{11-13} \cline{15-17} \cline{19-21}
      % \cline{4-6}\cline{8-10} \cline{12-14} \cline{15-18} \cline{20-22}
      & $p$ & MSE/C-idx & TPR & FPR && MSE/C-idx & TPR & FPR & & MSE/C-idx & TPR & FPR  && MSE/C-idx & TPR & FPR  && MSE/C-idx & TPR & FPR\\
      \hline
      \multirow{4}{*}{M\ref{M1}} & 100 & 1.50(0.06) & 1.00 & 0.00 & & 1.13(0.05) & 1.00 & 0.00 & & 1.61(0.47) & 0.99 & 0.00& & 1.08(0.04) & 1.00 & 0.00 & & 1.59(0.09) & 1.00 & 0.11 \\
      & 300 & 1.51(0.13) & 1.00 & 0.00 & & 1.23(0.12) & 1.00 & 0.00 & & 1.71(0.52) & 0.98 & 0.00& & 1.13(0.06) & 1.00 & 0.00 & & 1.62(0.10) & 1.00 & 0.05\\
      & 500 & 1.54(0.24) & 1.00 & 0.00 & & 1.35(0.08)& 1.00 & 0.00 & & 1.85(0.55) & 0.98 & 0.00& & 1.17(0.04)& 1.00 & 0.00 & & 1.66(0.08) & 1.00 & 0.03\\
      & 1000 & 1.57(0.16) & 1.00 & 0.00 & & 1.42(0.17) & 1.00 & 0.00 & &1.98(0.57) & 0.97 & 0.00& & 1.18(0.05)& 1.00 & 0.00 & & 1.68(0.09) & 1.00 & 0.01\\
      \hline
      \multirow{4}{*}{M\ref{M2}} & 100 & - & 1.00 & 0.00 & & - & 0.58 & 0.46 & & - & 0.36 & 0.27 & & - & 1.00 & 0.32 & & - & 0.06 & 0.05 \\
      & 300 & -  & 0.98 & 0.03 & & - & 0.49 & 0.48 & & - & 0.18 & 0.09 & & - & 1.00 & 0.12 & & - & 0.09 & 0.08 \\
      & 500 & -  & 0.96 & 0.01 & & - & 0.45 & 0.50 & & - & 0.13 & 0.01 & & - & 1.00 & 0.19 & & - & 0.00 & 0.00\\
      & 1000 & - & 0.95 & 0.02 & & - & 0.60 & 0.48 & & - & 0.04 & 0.03 & & - & 1.00 & 0.25 & & - & 0.00 & 0.00 \\
      \hline
      \multirow{4}{*}{M\ref{M3}} & 100 & 1.46(0.27) & 1.00 & 0.00 & & 0.61(0.28) & 1.00 & 0.00 & & 1.55(0.22)& 1.00 & 0.00 & & 1.07(0.49) & 1.00 & 0.00 & & 1.67(0.36) & 1.00 & 0.07 \\
      & 300 & 1.44(0.27)  & 1.00 & 0.00 & & 1.11(0.80) & 1.00 & 0.00 & & 1.55(0.27) & 0.98 & 0.00 & & 1.25(0.65) & 1.00 & 0.00 & & 1.76(0.45) & 1.00 & 0.03\\
      & 500 & 1.45(0.32) & 1.00 & 0.00 & & 1.32(0.69) & 1.00 & 0.00 & & 1.72(0.48) & 0.97 & 0.00 & &1.29(0.53) & 1.00 & 0.00 & & 1.99(0.64) & 1.00 & 0.01\\
      & 1000 & 1.49(0.36) & 1.00 & 0.00 & & 1.45(0.49) & 1.00 & 0.00 & &1.96(0.46) & 0.97 & 0.00 & &1.66(0.40)& 1.00 & 0.00 & & 2.10(0.67) & 1.00 & 0.00 \\
      \hline
      \multirow{4}{*}{M\ref{M4}} & 100 & 8.69(1.03) & 1.00 & 0.00 & & 8.53(1.04) & 0.47 & 0.48 & & 8.97(0.91) & 0.28 & 0.31 & & 8.51(0.87) & 1.00 & 0.84 & & 8.47(1.02) & 0.00 & 0.00\\
      & 300 & 8.66(0.93) & 1.00 & 0.00 & & 7.55(0.67) & 0.52 & 0.51 & & 8.95(0.83) & 0.13 & 0.10 & & 8.71(0.89) & 1.00 & 1.00 & & 8.53(0.81) & 0.00 & 0.00\\
      & 500 & 7.72(0.73) & 0.98 & 0.00 & & 9.21(0.81) & 0.48 & 0.43 & & 8.82(0.93) & 0.07 & 0.06 & & 8.57(0.86) & 1.00 & 1.00 & & 8.94(0.92) & 0.00 & 0.00 \\
      & 1000 & 8.01(0.76) & 0.99 & 0.00 & & 9.65(0.97) & 0.52 & 0.50 & & 8.17(0.66) & 0.05 & 0.03 & & 9.72(0.84) & 1.00 & 1.00 & & 9.13(0.72) & 0.00 & 0.00 \\
    \hline
    \multirow{4}{*}{M\ref{M5}} & 100 & 0.90(0.01) & 1.00 & 0.00 & &  0.91(0.01) & 1.00 & 0.00 & & - & - &  - & & 0.83(0.01) & 0.95 & 0.00 & & 0.79(0.01) & 1.00 & 0.66\\
    & 300 & 0.91(0.01) & 1.00 & 0.00 & & 0.91(0.01) & 1.00 & 0.00 & &  - & - &  - & & 0.82(0.01) & 0.97 & 0.01 && 0.79(0.01) & 1.00 & 0.52\\
    & 500 & 0.91(0.01) & 1.00 & 0.00 & & 0.91(0.01) & 1.00 & 0.00 & & - & - &  - & & 0.82(0.01) & 0.96 & 0.00 && 0.79(0.01) & 1.00 & 0.35 \\
    & 1000 & 0.91(0.01) & 1.00 & 0.00 & &  0.91(0.01) & 1.00 & 0.00 & &  - & - &  - & & 0.80(0.02) & 1.00 & 0.63 && 0.79(0.01) & 1.00 & 0.17\\
    \hline
    \multirow{4}{*}{M\ref{M6}} & 100 & 0.83(0.01) & 0.99 & 0.00 & & 0.79(0.02) & 0.85 & 0.00 & & - & - &  - & & 0.71(0.01) & 0.63 & 0.08 & & 0.49(0.01) & 0.00 & 0.00 \\
    & 300 & 0.83(0.01) & 0.99 & 0.00 & & 0.78(0.02) & 0.81 & 0.06 & & - & - &  - & & 0.68(0.01) & 0.54 & 0.01 & & 0.50(0.00) & 0.00 & 0.00\\
    & 500 & 0.82(0.01) & 0.99 & 0.00 & & 0.78(0.02) & 0.77 & 0.09 & & - & - &  - & & 0.67(0.01) & 0.44 &0.27 & & 0.50(0.01) & 0.00 & 0.00 \\
    & 1000 & 0.82(0.01) & 0.99 & 0.00 & & 0.77(0.01) & 0.73 & 0.01 && - & - &  - & & 0.67(0.01) & 0.45 & 0.48 & & 0.50(0.00) & 0.00 & 0.00\\
    \hline
    \end{tabular}
    %\footnotesize
    Note: 
    The standard errors are given in parentheses. With DFS, the number of variables to be selected is set as 30, which is required during the training process. %For the AFT models, DFS is not applicable.
    \end{minipage}}
    \end{threeparttable}
    \label{tab:d30}
\end{table}

The variable selection and prediction results are summarized in Table \ref{tab:d30}. It is observed that the proposed approach has satisfactory variable selection performance with high TPR and low FPR values, while the alternatives fail with Models M\ref{M2}, M\ref{M4}, and M\ref{M6}. In particular, for Model M\ref{M6} which is especially complicated, the proposed approach is the only one that can satisfactorily identify the important variables under all dimensions of $\bX$. In terms of prediction, it has smaller MSEs than DFS and GCRNet, but may have slightly larger values than P-LS and LassoNet in some cases. This observation aligns with expectation, as the objective functions of the latter methods primarily aim at minimizing mean squared error, whereas the proposed emphasizes distribution matching. The differences diminish as the dimension of $\bX$ increases, suggesting that the proposed approach is less affected by dimensionality. This can be attributed to the proposed two-stage strategy. 
%As long as the first stage correctly selects the significant variables, the second stage can effectively conduct estimation/prediction. Further, with the flexibility of this two-stage approach, the better prediction performance can be obtained via replacing the Wassersterin distance in the second stage by the mean square error. 
It is noted that MSE for M\ref{M2} is not reported because the mean of $t(1)$ is not well defined. Additional numerical results are provided in Appendix, including the evaluation of the distribution estimation and the results with $\bX^{s}\in \bbR^5$.

\section{Data Analysis}

To demonstrate practical utility, we analyze three datasets, namely The Cancer Genome Atlas Lung Adenocarcinoma (TCGA-LUAD) data, the Medical Information Mart for Intensive Care III (MIMIC-III) data \citep{johnson2016mimic}, and the HIV-1 data \citep{rhee2006genotypic}. In the TCGA-LUAD data, both the predictors and response are continuous. In the MIMIC-III data, we conduct two analyses: the first is to identify important biomarkers for albumin level, where both the predictors and response are continuous, and the second is to identify important biomarkers for survival (which is subject to right censoring). In the HIV-1 data, the predictors are binary, and the response is continuous. For the proposed and alternative methods, the implementation details are summarized in Section XX in Appendix.

\subsection{TCGA-LUAD data}\label{subsec: TCGA}
TCGA is an authoritative cancer omics database. LUAD is a major subtype of lung cancer. This dataset contains 460 samples. For the response variable, we first extract 226 features of tumor pathological images, which describe tumors’ micro properties and microenviroment. We then conduct principal component analysis (PCA) and extract the first PC, which can describe the most important pathological properties. For predictors, we consider 21,080 gene expressions and apply marginal screening to reduce the dimension to 900. Overall, this analysis examines how tumor properties are affected by gene expressions. 

Nine genes are identified by the proposed approach, which are {\it ST6GAL1, ISL2, OR10G8, WDR89, OR4N5, TRIM16L, C1orf126, TMEM159, PIK3C2A}. Data is also analyzed using the alternatives considered in simulation. It is observed that different methods lead to different selection results. Out of the four compared methods, P-LS and LassoNet have an overlap in their selection of important genes with the proposed approach. More details are available in Table \ref{tab:TCGA-selection} in Appendix. To evaluate the stability of selection, we resort to a random sampling-based approach with 30 replications. In each replication, we randomly pick 360 samples for selection/estimation and the rest 100 samples for testing. For the proposed approach, P-LS, and LassoNet, several genes are frequently identified, while DFS selects highly different genes, across the replications. GCRNet identifies zero important genes in 27 replications. The details of the selected genes with their frequencies are given in Table \ref{tab:TCGA-gene} in Appendix. Additionally, we also evaluate prediction performance. With the estimation results (using the 360 samples), prediction performance is evaluated for the 100 testing samples using MSE. The results are summarized in Table \ref{tab:real-data-prediction}. It is observed that the proposed approach has the best prediction performance. 

\begin{table}[]
    \centering
    \caption{Data analysis: Prediction performance.}
    \begin{threeparttable}
    \resizebox{\textwidth}{!}{\begin{minipage}{\textwidth}
    \begin{tabular}{c|ccccc}
    \hline
         &Proposed & P-LS & DFS & LassoNet & GCRNet\\
         \hline
       TCGA & 0.917(0.355) & 1.061(0.354) & 1.513(0.090) & 1.516(0.102) & 2.278(0.538)\\
       MIMIC-III (Albumin) & 0.082(0.018) & 0.104(0.020) & 0.084(0.020)& 0.363(0.058)& 0.378(0.054) \\
       MIMIC-III (Survival) & 0.646(0.014) & 0.646(0.022) & - & 0.630(0.029) & 0.563(0.013) \\
        HIV-1 (APV) & 0.462(0.049) & 0.503(0.055) & 0.430(0.042) & 0.484(0.049) & 0.529(0.063)\\
      HIV-1 (SQV) & 0.667(0.059) & 0.662(0.056) & 0.674(0.064)& 0.676(0.053) & 0.731(0.066)\\
       \hline
    \end{tabular}  
    %\footnotesize
    Note: Standard errors are given in parentheses. For MIMIC-III (survival data), c-idx is computed instead of MSE. DFS is not applicable to survival data.
    \end{minipage}}
    \end{threeparttable}
    \label{tab:real-data-prediction}
\end{table}

\subsection{MIMIC-III data on albumin level}\label{subsec: MIMIC-III-albumin}
MIMIC-III dataset is a large openly available electronic health records database, comprising rich information related to patients admitted to critical care units between 2001 and 2012.  It contains data on 38,597 distinct adult patients and covers 12,587 charted observations in total. We refer to the literature \citep{deng2020optimal} for information on data processing. For response, we consider albumin level, which is associated with many diseases such as liver cirrhosis. For predictors, we consider basic information of patients such as gender as well as the charted observations. For each charted observation, the mean and max values are considered (as longitudinal observations are available). The total dimensionality is 155. Those samples with missing response and predictors with missing rates exceeding 40$\%$ are excluded from analysis. The remaining missing values are imputed using the sample means. After this processing, the sample size is 2,233. 

The proposed approach identifies seven important variables, including one on basic patient information (number of admissions) and six charted observations. Among the charted observations, three pertain to different aspects of white blood cells, which are closely related to albumin level \citep{tate2019albumin}. Among the remaining three, two represent the mean and max values of Platelets,  and the third one is Differential-laymphs. In certain situations such as the diagnosis of liver and kidney diseases, Platelets and Differential-laymphs are considered in conjunction with albumin level \citep{jorgensen1980inhibitory}. Among the seven identified variables by the proposed approach, DFS identifies five, P-LS identifies four, and LassoNet and GCRNet identify two. More details are provided in Table \ref{tab:MIMIC-selection} in Appendix. The same evaluations as for the TCGA data are conducted. The splitted sets have sizes 1,933 and 300, respectively. The same four variables are consistently identified by the proposed approach across replications, suggesting a high level of stability. The frequently identified variables of each method are given in Table \ref{tab:high-MIMIC-selected} in Appendix. Table \ref{tab:real-data-prediction} shows that the prediction errors of the proposed approach and DFS are comparable. Additionally, the proposed approach exhibits a prediction error smaller than that of P-LS and significantly smaller than those of LassoNet and GCRNet.

\subsection{MIMIC-III data on survival}\label{subsec: MIMIC-III survival}
The same data processing as above is conducted. The analyzed dataset contains 8,323 patients, and the survival times of 5,517 patients are censored. There are a total of 379 predictors, which describe the basic information of the patients as well as the min, mean, and max values of the charted observations. The proposed approach identifies 11 variables, and LassoNet identifies 15 variables. All of those variables are the charted observations. P-LS identifies 14 variables, among which one is on the patients' basic information. GCRNet does not identify any variable. Additional results are provided in Table \ref{tab:MIMIC-III-survival-selection} in Appendix. In the evaluations, the two splitted sets have sizes 5,323 and 3,000, respectively. Table \ref{tab:real-data-prediction} shows that the c-idx values of the proposed approach and P-LS are the same and slightly larger than LassoNet. For GCRNet, as no significant variable is identified, we use all the variables to compute c-idx, which is much worse than the others.

\subsection{HIV-1 drug resistance data}\label{subsec: HIV}
This dataset contains information on 16 drugs from three classes. We specifically consider two drugs that belong to the protease inhibitors class to identify mutations that are associated with them. The response variable is the log-transformed drug resistance level. Each component of the predictor vector $\bX$ is a binary variable, representing the presence or absence of a mutation. The samples with missing drug resistance information and the mutations that appear less than three times are removed from analysis. The dimension of $\bX$ is 205. There are 768 samples for Amprenavir (APV) and 826 samples for Saquinavir (SQV).

For both APV and SQV, the proposed approach identifies 24 mutations. P-LS and DFS identify 25 mutations. LassoNet identifies 23 mutations for APV and 25 mutations for SQV. GCRNet identifies 18 mutations for APV and 31 mutations for SQV. The detailed list of the identified mutations can be found in Table \ref{tab:HIV-selection} in Appendix. To evaluate identification accuracy, we compare against the expert panel given in \cite{rhee2006genotypic}. It can be seen from Table \ref{tab:HIV} that the proposed approach identifies a larger number of truly important mutations while having the lowest false discoveries in most cases. The prediction results given in Table \ref{tab:real-data-prediction} show that the prediction error of the proposed approach is comparable to or slightly smaller than most of the others.

% \section{Discussions}

% \section*{Acknowledgements}
% This study is partly supported by ...

% \bibliographystyle{plain}
\bibliography{Ref-main}

\begin{thebibliography}{}

\bibitem[Bauer and Kohler, 2019]{Bauer2019}
Bauer, B. and Kohler, M. (2019).
\newblock On deep learning as a remedy for the curse of dimensionality in nonparametric regression.
\newblock {\em The Annals of Statistics}, 47(4):2261 -- 2285.

\bibitem[Bolte et~al., 2010]{bolte2010characterizations}
Bolte, J., Daniilidis, A., Ley, O., and Mazet, L. (2010).
\newblock Characterizations of {\l}ojasiewicz inequalities: subgradient flows, talweg, convexity.
\newblock {\em Transactions of the American Mathematical Society}, 362(6):3319--3363.

\bibitem[Chen et~al., 2021]{chen2021nonlinear}
Chen, Y., Gao, Q., Liang, F., and Wang, X. (2021).
\newblock Nonlinear variable selection via deep neural networks.
\newblock {\em Journal of Computational and Graphical Statistics}, 30(2):484--492.

\bibitem[Deng et~al., 2023]{deng2020optimal}
Deng, S., Ning, Y., Zhao, J., and Zhang, H. (2023).
\newblock Optimal and safe estimation for high-dimensional semi-supervised learning.
\newblock {\em Journal of the American Statistical Association}, 0(0):1--12.

\bibitem[Dinh and Ho, 2020]{dinh2020consistent}
Dinh, V.~C. and Ho, L.~S. (2020).
\newblock Consistent feature selection for analytic deep neural networks.
\newblock {\em Advances in Neural Information Processing Systems}, 33:2420--2431.

\bibitem[Farrell et~al., 2021]{Farrell2021DNN}
Farrell, M.~H., Liang, T., and Misra, S. (2021).
\newblock Deep neural networks for estimation and inference.
\newblock {\em Econometrica}, 89(1):181--213.

\bibitem[Feng and Simon, 2017]{feng2017sparse}
Feng, J. and Simon, N. (2017).
\newblock Sparse-input neural networks for high-dimensional nonparametric regression and classification.
\newblock {\em arXiv preprint arXiv:1711.07592}.

\bibitem[Gulrajani et~al., 2017]{gulrajani2017improved}
Gulrajani, I., Ahmed, F., Arjovsky, M., Dumoulin, V., and Courville, A.~C. (2017).
\newblock Improved training of wasserstein gans.
\newblock {\em Advances in neural information processing systems}, 30.

\bibitem[He et~al., 2018]{He2018}
He, X., Wang, J., and Lv, S. (2018).
\newblock Gradient-induced model-free variable selection with composite quantile regression.
\newblock {\em Statistica Sinica}, 28(3):1521--1538.

\bibitem[Huang et~al., 2010]{huang2010}
Huang, J., Horowitz, J.~L., and Wei, F. (2010).
\newblock {Variable selection in nonparametric additive models}.
\newblock {\em The Annals of Statistics}, 38(4):2282--2313.

\bibitem[Jiao et~al., 2023]{jiao2023AOS}
Jiao, Y., Shen, G., Lin, Y., and Huang, J. (2023).
\newblock {Deep nonparametric regression on approximate manifolds: Nonasymptotic error bounds with polynomial prefactors}.
\newblock {\em The Annals of Statistics}, 51(2):691 -- 716.

\bibitem[Johnson et~al., 2016]{johnson2016mimic}
Johnson, A.~E., Pollard, T.~J., Shen, L., Lehman, L.-w.~H., Feng, M., Ghassemi, M., Moody, B., Szolovits, P., Anthony~Celi, L., and Mark, R.~G. (2016).
\newblock Mimic-iii, a freely accessible critical care database.
\newblock {\em Scientific data}, 3(1):1--9.

\bibitem[J{\o}rgensen and Stoffersen, 1980]{jorgensen1980inhibitory}
J{\o}rgensen, K.~A. and Stoffersen, E. (1980).
\newblock On the inhibitory effect of albumin on platelet aggregation.
\newblock {\em Thrombosis research}, 17(1):13--18.

\bibitem[Kallenberg, 2002]{kallenberg2002foundations}
Kallenberg, O. (2002).
\newblock {\em Foundations of Modern Probability}.
\newblock Springer.

\bibitem[Lei, 2020]{lei2020}
Lei, J. (2020).
\newblock {Convergence and concentration of empirical measures under Wasserstein distance in unbounded functional spaces}.
\newblock {\em Bernoulli}, 26(1):767 -- 798.

\bibitem[Lemhadri et~al., 2021]{lemhadri2021lassonet}
Lemhadri, I., Ruan, F., Abraham, L., and Tibshirani, R. (2021).
\newblock Lassonet: A neural network with feature sparsity.
\newblock {\em The Journal of Machine Learning Research}, 22(1):5633--5661.

\bibitem[Li et~al., 2019]{Li2019BetterAO}
Li, B., Tang, S., and Yu, H. (2019).
\newblock Better approximations of high dimensional smooth functions by deep neural networks with rectified power units.
\newblock {\em Communications in Computational Physics}.

\bibitem[Li et~al., 2020]{li2020powernet}
Li, B., Tang, S., and Yu, H. (2020).
\newblock Powernet: Efficient representations of polynomials and smooth functions by deep neural networks with rectified power units.
\newblock {\em Journal of Mathematical Study}, 53(2):159--191.

\bibitem[Li et~al., 2016]{li2016deep}
Li, Y., Chen, C.-Y., and Wasserman, W.~W. (2016).
\newblock Deep feature selection: theory and application to identify enhancers and promoters.
\newblock {\em Journal of Computational Biology}, 23(5):322--336.

\bibitem[Liang et~al., 2018]{liang2018bayesian}
Liang, F., Li, Q., and Zhou, L. (2018).
\newblock Bayesian neural networks for selection of drug sensitive genes.
\newblock {\em Journal of the American Statistical Association}, 113(523):955--972.

\bibitem[Liu et~al., 2007]{SpAM2007}
Liu, H., Wasserman, L., Lafferty, J., and Ravikumar, P. (2007).
\newblock Spam: Sparse additive models.
\newblock In {\em Advances in Neural Information Processing Systems}, volume~20.

\bibitem[Liu et~al., 2021]{liu2021}
Liu, S., Zhou, X., Jiao, Y., and Huang, J. (2021).
\newblock Wasserstein generative learning of conditional distribution.
\newblock {\em arXiv}, 2112.10039.

\bibitem[Lu et~al., 2018]{lu2018deeppink}
Lu, Y., Fan, Y., Lv, J., and Stafford~Noble, W. (2018).
\newblock Deeppink: reproducible feature selection in deep neural networks.
\newblock {\em Advances in neural information processing systems}, 31.

\bibitem[Lu and Lu, 2020]{NEURIPS2020_Lu}
Lu, Y. and Lu, J. (2020).
\newblock A universal approximation theorem of deep neural networks for expressing probability distributions.
\newblock volume~33, pages 3094--3105.

\bibitem[Luo and Halabi, 2023]{luo2023sparseinput}
Luo, B. and Halabi, S. (2023).
\newblock Sparse-input neural network using group concave regularization.
\newblock {\em arXiv preprint arXiv:2307.00344}.

\bibitem[Padilla et~al., 2020]{Padilla2020QuantileRW}
Padilla, O. H.~M., Tansey, W., and Chen, Y. (2020).
\newblock Quantile regression with relu networks: Estimators and minimax rates.
\newblock {\em Journal of Machine Learning Research}, 23:247:1--247:42.

\bibitem[Pham, 2012]{Lojasiewicz2012}
Pham, T.~S. (2012).
\newblock {An explicit bound for the \L ojasiewicz exponent of real polynomials}.
\newblock {\em Kodai Mathematical Journal}, 35(2):311 -- 319.

\bibitem[Rhee et~al., 2006]{rhee2006genotypic}
Rhee, S.-Y., Taylor, J., Wadhera, G., Ben-Hur, A., Brutlag, D.~L., and Shafer, R.~W. (2006).
\newblock Genotypic predictors of human immunodeficiency virus type 1 drug resistance.
\newblock {\em Proceedings of the National Academy of Sciences}, 103(46):17355--17360.

\bibitem[Scardapane et~al., 2017]{scardapane2017group}
Scardapane, S., Comminiello, D., Hussain, A., and Uncini, A. (2017).
\newblock Group sparse regularization for deep neural networks.
\newblock {\em Neurocomputing}, 241:81--89.

\bibitem[Schmidt-Hieber, 2020]{Schmidt-Hieber2020}
Schmidt-Hieber, J. (2020).
\newblock Nonparametric regression using deep neural networks with relu activation function.
\newblock {\em Annals of statistics}, 48:1875--1897.

\bibitem[Shen et~al., 2022]{shen2022estimation}
Shen, G., Jiao, Y., Lin, Y., Horowitz, J.~L., and Huang, J. (2022).
\newblock Estimation of non-crossing quantile regression process with deep requ neural networks.
\newblock {\em arXiv preprint arXiv:2207.10442}.

\bibitem[Shen et~al., 2021]{shen2021robust}
Shen, G., Jiao, Y., Lin, Y., and Huang, J. (2021).
\newblock Robust nonparametric regression with deep neural networks.
\newblock {\em arXiv preprint arXiv:2107.10343}.

\bibitem[Stute, 1996]{stute1996}
Stute, W. (1996).
\newblock Distributional convergence under random censorship when covariables are present.
\newblock {\em Scandinavian Journal of Statistics}, 23(4):461--471.

\bibitem[Sun et~al., 2022]{Yan2022}
Sun, Y., Song, Q., and Liang, F. (2022).
\newblock Consistent sparse deep learning: Theory and computation.
\newblock {\em Journal of the American Statistical Association}, 117(540):1981--1995.

\bibitem[Tate et~al., 2019]{tate2019albumin}
Tate, J.~P., Sterne, J.~A., Justice, A.~C., Study, V. A.~C., (ART-CC, A. T. C.~C., et~al. (2019).
\newblock Albumin, white blood cell count, and body mass index improve discrimination of mortality in hiv-positive individuals.
\newblock {\em AIDS (London, England)}, 33(5):903.

\bibitem[Van Der~Vaart et~al., 1996]{van1996weak}
Van Der~Vaart, A.~W., Wellner, J.~A., van~der Vaart, A.~W., and Wellner, J.~A. (1996).
\newblock {\em Weak convergence}.
\newblock Springer.

\bibitem[Villani, 2009]{villani2009optimal}
Villani, C. (2009).
\newblock {\em Optimal Transport: Old and New}.
\newblock Springer.

\bibitem[Yang et~al., 2022]{Yang2021}
Yang, Y., Li, Z., and Wang, Y. (2022).
\newblock On the capacity of deep generative networks for approximating distributions.
\newblock {\em Neural Networks}, 145:144–154.

\bibitem[Yuan and Lin, 2006]{yuan2006model}
Yuan, M. and Lin, Y. (2006).
\newblock Model selection and estimation in regression with grouped variables.
\newblock {\em Journal of the Royal Statistical Society Series B: Statistical Methodology}, 68(1):49--67.

\bibitem[Zhang, 2010]{ZhangCH2010}
Zhang, C.-H. (2010).
\newblock {Nearly unbiased variable selection under minimax concave penalty}.
\newblock {\em The Annals of Statistics}, 38(2):894 -- 942.

\bibitem[Zhao et~al., 2015]{zhao2015heterogeneous}
Zhao, L., Hu, Q., and Wang, W. (2015).
\newblock Heterogeneous feature selection with multi-modal deep neural networks and sparse group lasso.
\newblock {\em IEEE Transactions on Multimedia}, 17(11):1936--1948.

\bibitem[Zou and Hastie, 2005]{zou2005regularization}
Zou, H. and Hastie, T. (2005).
\newblock Regularization and variable selection via the elastic net.
\newblock {\em Journal of the royal statistical society: series B (statistical methodology)}, 67(2):301--320.

\end{thebibliography}

%%%%%%%%%%%%%%%%%%%%%%%%%% Appendix %%%%%%%%%%%%%%%%%%%%%%%%%%
%%%%%%%%%%%%%%%%%%%%%%%%%% Appendix %%%%%%%%%%%%%%%%%%%%%%%%%%
\clearpage

\setcounter{subsection}{0}

\renewcommand\thetable{A.\arabic{table}}
\renewcommand\thesubsection{A.\arabic{subsection}}
\renewcommand\thefigure{A.\arabic{figure}}

\section*{Appendix}

\subsection{Conditions}
Denote the class of FNNs with input dimension $p_0$, width $\calW$, and depth $\calL$ by $\mathcal{NN}(p_0,\calW,\calL)$.  Let $\bbN_0$ and $\bbN^{+}$ be the set of non-negative integers and of positive integers, respectively. For a positive constant $c$, $\lfloor c\rfloor$ denotes the largest integer strictly smaller than $c$.

\subsubsection{Assumed Conditions}% for Theorems in Section \ref{sec:thm}}
\begin{condition}\label{C:XY}
    Let $\bm{V}=(\bX,Y)\sim P_{\bX,Y}$. For some $\delta>0$, $\bm{V}$ satisfies the first moment tail condition $\bbE\|\bm{V}\| \mathbbm{1}\{\|\bm{V}\|>\log t\}=O(t^{-(\log t)^\delta /(p+1)} )$, for any $t \geq 1$.
\end{condition}
\begin{condition}\label{C:significant}
    The dimension of important  features $p_s$ is fixed, and there exists a constant $c_s>0$ such that %$\min_{1\leq j\leq p_s}\bbE[g^{\ast}(\bX,\bZ)|\bX^{[-j]},\bZ]\geq c_s.$
    $\min_{1\leq j\leq p_s}|\bbE_{\bX,Y}[Y|\bX^{[j]}(0)]-\bbE_{\bX,Y}[Y|\bX]|\geq c_s,$ where $\bX^{[j]}(0)=(X^{[1]},\ldots,X^{[j-1]},0,X^{[j+1]},\ldots,X^{[p]})$.
\end{condition}
{
\begin{condition}\label{C:g-smooth}
    The function $g^{\ast}$ defined in (\ref{eq:g-star}) belongs to the H\"older class $\mathcal{H}^{\beta}(\mathcal{X}\times{\calZ} ,B_0)$ for a given $\beta>0$ and a finite constant $B_0>0$, which is defined as
    \begin{align*}
& \mathcal{H}^\beta\left(\mathcal{X}\times{\calZ}, B_0\right)=\left\{f:\mathcal{X}\times{\calZ} \rightarrow \mathbb{R}, \max _{\|\alpha\|_1 \leq\lfloor\beta\rfloor}\left\|\partial^\alpha f\right\|_{\infty} \leq B_0, \max _{\|\alpha\|_1=\lfloor\beta\rfloor} \sup _{ \bm{r} \neq \bm{s}} \frac{\left|\partial^\alpha f(\bm{r})-\partial^\alpha f(\bm{s})\right|}{\|\bm{r}-\bm{s}\|_2^{\beta-\lfloor\beta\rfloor}} \leq B_0\right\} \\
& \text {where} \partial^\alpha=\partial^{\alpha} \cdots \partial^{\alpha_{m+p}} \text { with } \alpha=\left(\alpha, \ldots, \alpha_{m+p}\right)^{\top} \in \mathbb{N}_0^{m+p}.
    \end{align*}
\end{condition}

\begin{condition}\label{C:G-class}
    The generator class $\calG_{\btheta_1}$ is a class of FNNs with width $\calW^2_{\calG_{\btheta_1}}\geq 7(p+m)+1$, depth $\calL_{\calG_{\btheta_1}}\geq 2$, and $\calW^2_{\calG_{\btheta_1}}\calL_{\calG_{\btheta_1}}=cn_1$ with $12\leq c\leq 384$.
    Its parameter $\btheta\in[-B_2,B_2]^{\mathcal{S}_{\btheta}}$,
where $0<B_2<\infty$ is a positive constant and $\mathcal{S}_{\btheta}$ is the size of network $g_{\btheta_1}$.
\end{condition}
\begin{condition}\label{C:F-class}
    The class for network function $f_{\bphi_1}$ is  $\calF_{\bphi_1}\equiv\mathcal{NN}(p+1,\calW_{\calF_{\bphi_1}},\calL_{\calF_{\bphi_1}})$, whose width $\calW_{\calF_{\bphi_1}}=152 (p+1)^2 3^{p+1}N_{f_1}\lceil\log_2(8N_{f_1})\rceil$, depth $\calL_{\calF_{\bphi_1}}=84M_{f_1}\lceil \log_2 (8M_{f_1})\rceil+2(p+1)$, with $N_{f_1},M_{f_1}\in\mathbb{N}^{+}$ and $N_{f_1}M_{f_1}\geq \sqrt{n_1}$.
\end{condition}
}

\begin{condition}\label{C:Lojasewica}
There exists a positive constant $a\geq2$ such that 
$$ d(\btheta,\Theta)^{a}\precsim\{ W(P_{\bX,g_{\btheta_1}(\bX,\bZ)},P_{\bX, Y})-W(P_{\bX, g_{\btheta^{\ast}_1}(\bX,\bZ)},P_{\bX, Y})\}.$$ 
\end{condition}

Condition \ref{C:XY} is a technical condition for dealing with the case where the support of $P_{\bX,Y}$ is an unbounded subset of $\bbR^{p+1}$ and is commonly satisfied. A popular special case is sub-Gaussian.
Condition \ref{C:significant} is a condition on  minimal detected signal with respect to Definition \ref{def:significance}. Similar definitions have been commonly used in the variable selection literature \citep{huang2010,He2018}.
Condition \ref{C:g-smooth} is a smoothness condition for the target function $g^{\ast}$ in (\ref{eq:g-star}), which is commonly required in the approximation theory of neural networks \citep{Schmidt-Hieber2020,shen2021robust}.
Conditions \ref{C:G-class}-\ref{C:F-class} describe the size of the network $g_{\btheta}$ and $f_{\bphi}$, respectively.
Condition \ref{C:Lojasewica} is a technical condition and also has roots in the literature.  
For a fixed network, this condition holds when the network's activation function is analytic, and can be demonstrated through \L ojasewica's inequality. Several activation functions satisfy this condition, including the classic ones such as the linear function, tanh function, and sigmoid function, as well as the newly developed ReLU-type activation functions such as GeLU, ELU, and PELU.
In general, $a$ can be upper-bounded by some positive constant, such as in a polynomial function class, which has been proved by \cite{bolte2010characterizations} and \cite{Lojasiewicz2012}.
This condition is sensible, given the close relationship between the polynomial function class and the network class \citep{Li2019BetterAO,li2020powernet, shen2022estimation}.

\subsubsection{Conditions for survival analysis}

Let $F_T$, $F_C$, and $F_Y$ be the distribution functions of the event time $T$, the censoring time $C$, and the observed time $Y$, respectively.
Let $\tau_T$, $\tau_C$, and $\tau_Y$ be the endpoints of the support of $T$, $C$, and $Y$, respectively.
Then, $1-F_{Y}(y)=(1-F_T(y))(1-F_C(y))$, resulting from the independence assumption of $T$ and $C$. In light of $F_{\bX,T}$, the joint distribution of $(\bX,T)$, denote 
\begin{align*}
    \tilde{F}_{\bX,T}(\bx, t)= \begin{cases}F_{\bX,T}(\bx, t), & t<\tau_Y, \\ F_{\bX,T}\left(\bx, \tau_Y-\right)+F_{\bX,T}\left(\bx, \tau_Y\right) 1\left\{\tau_Y \in \Omega\right\}, & t \geqslant \tau_Y,\end{cases}
\end{align*}
with $\Omega$ denoting the set of atoms of $F_Y$. 
Then, two subdistribution functions are defined as follows:
\begin{align*}
    \tilde{F}_{\bX,Y}^{1}(\bx, y)  =P(\bX \leq \bx, Y \leq y, \Delta=1), \quad
\tilde{F}_Y^0(y)  =P(Y \leq y, \Delta=0).
\end{align*}
For $j=1,\ldots,p$, let
\begin{align*}
    &\bgamma_0(y)=\exp \left\{\int_0^{y-} \frac{\tilde{F}_Y^0(d u)}{1-F_Y(u)}\right\},\\
    &\bgamma_{1,j}(y)=\frac{1}{1-F_{Y}(y)}\int \mathbbm{1}\{y<u\}f(\bx,u)\bgamma_0(u)\tilde{F}_{\bX,Y}^1(d\bx,dz),\\
    & \bgamma_{2,j}(y)=\int\int\frac{\mathbbm{1}\{\eta<y,\eta<z\}f(\bx,z)\bgamma_0(z)}{[1-F_Y(\eta)]^2}\tilde{F}_Y^0(d\eta)\tilde{F}_{\bX,Y}^1(d\bx,dz).
\end{align*}
The following conditions are  assumed.

\begin{condition}\label{C6:censor-indep}
    $T$ and $C$ are independent and $P(T\leq C|T,\bX)=P(T\leq C|\bX)$.
\end{condition}
\begin{condition}\label{C7:integrability} $(X,Y,\Delta)$ satisfies
    \begin{align*}
        &\mathbb{E}_{X,T}[f(\bX,Y)\bgamma_0(Y)\Delta]^2<\infty,\\
        \int|f(\bx,z)|K^{1/2}(u) & \tilde{F}_{\bX,T}(d\bx,du)<\infty, \text{ where } K(y)=\int_{0}^{y^{-}}\frac{F_C(du)}{[(1-F_Y(u))(1-F_C(u))]}.
    \end{align*}
\end{condition}
These two conditions have been commonly assumed.

\clearpage
\subsection{Additional Supporting Lemmas}

\begin{lemma}\label{lem:stochastic-bound}
    Suppose that Conditions \ref{C:XY} and \ref{C:G-class} hold. 
    Then, for  $g_{\hat{\btheta}_1}\in\calG_{\btheta_1}$ obtained in (\ref{eq:est-stage1}),  
    \begin{align}\label{eq:stoch-err-bound}
        \bbE\sup_{f\in\calF^1_\text{Lip}}\left\{\bbE_{\bX, \bZ} f(\bX,  g_{\hat{\btheta}_1}(\bX,\bZ))-\frac{1}{n_1}\sum_{i=1}^{n_1} f(\bX_i,  g_{\hat{\btheta}_1}(\bX_i,\bZ_i))\right\}\precsim{n_1}^{-1/(p+1)} \log {n_1},       
    \end{align}
    where $\mathbb{E}$ is the expectation taken with respect to $\{(\bX_i,Y_i)\}_{i=1}^{n_1}$ and $\{\bZ_i\}_{i=1}^{n_1}$.    
\end{lemma}

\begin{proof}
The proof will be done in two steps.

Let $\{(\bX'_i,Y'_i)\}_{i=1}^{n_1}$ and $\{\bZ'_i\}_{i=1}^{n_1}$ be ${n_1}$ i.i.d. copies of the random vector $(\bX,Y)$ and $\bZ$, respectively.
Let $\bm{\xi}=(\xi_1,\cdots,\xi_{n_1})$ be i.i.d. Rademacher variables, i.e., uniform $\{-1,1\}$.
With the symmetrization technique, we can bound the right hand side of (\ref{eq:stoch-err-bound}) as
\begin{align*}
    &\bbE\sup_{f\in\calF^1_\text{Lip}}\left\{\bbE_{\bX, \bZ} f(\bX,  g_{\hat{\btheta}_1}(\bX,\bZ))-\frac{1}{n_1}\sum_{i=1}^{n_1} f(\bX_i,  g_{\hat{\btheta}_1}(\bX_i,\bZ_i))\right\}\\
    =&\bbE\sup_{f\in\calF^1_\text{Lip}}\left\{\bbE_{\bX', \bZ'} \frac{1}{n_1}\sum_{i=1}^{n_1} f(\bX'_i,  g_{\hat{\btheta}_1}(\bX'_i,\bZ'_i))-\frac{1}{n_1}\sum_{i=1}^{n_1} f(\bX_i,  g_{\hat{\btheta}_1}(\bX_i,\bZ_i))\right\}\\
    \leq& \bbE_{\bX,\bZ,\bX',\bZ'}\left[\sup_{f\in\calF^1_\text{Lip}}\frac{1}{n_1}\sum_{i=1}^{n_1} \left\{f(\bX'_i,  g_{\hat{\btheta}_1}(\bX'_i,\bZ'_i))- f(\bX_i,  g_{\hat{\btheta}_1}(\bX_i,\bZ_i))\right\}\right]\\
    =&\bbE_{\bX,\bZ,\bX',\bZ',\bm{\xi}}\left[\sup_{f\in\calF^1_\text{Lip}}\frac{1}{n_1}\sum_{i=1}^{n_1} \xi_i\left\{f(\bX'_i,  g_{\hat{\btheta}_1}(\bX'_i,\bZ'_i))- f(\bX_i,  g_{\hat{\btheta}_1}(\bX_i,\bZ_i))\right\}\right]\\
    \leq&\bbE_{\bX,\bZ,\bm{\xi}}\left\{\sup_{f\in\calF^1_\text{Lip}}\frac{1}{n_1}\sum_{i=1}^{n_1} \xi_i f(\bX_i,  g_{\hat{\btheta}_1}(\bX_i,\bZ_i))\right\}+\bbE_{\bX',\bZ',\bm{\xi}'}\left\{\sup_{f\in\calF^1_\text{Lip}}\frac{1}{n_1}\sum_{i=1}^{n_1} \xi_i f(\bX'_i,  g_{\hat{\btheta}_1}(\bX'_i,\bZ'_i))\right\}\\
    =&2\bbE_{\bX,\bZ}\bbE_{\bm{\xi}}\left\{\sup_{f\in\calF^1_\text{Lip}}\frac{1}{n_1}\sum_{i=1}^{n_1} \xi_i f(\bX_i,  g_{\hat{\btheta}_1}(\bX_i,\bZ_i))\right\}\\
    \coloneqq&2\bbE_{\bX,\bZ}\mathcal{R}_n(\calF^1_{\text{Lip}}|g_{\hat{\btheta}_1}),
\end{align*}
where $\mathcal{R}_n(\calF^1_{\text{Lip}})$ is the conditional Rademacher complexity of $\calF^1_{\text{Lip}}$ given $g_{\hat{\btheta}_1}$.

To bound $\mathcal{R}_n(\calF^1_{\text{Lip}})$, the chaining method is applied. 
Let $\alpha_0=B$ and $\alpha_t=2^{-t}B$ for any $t\in\bbN^{+}$. 
For each $t$, let $\mathcal{C}_t$ be a $\alpha_t$ cover of $\calF^1_{\text{Lip}}$ with respect to $\|\cdot\|_{2}$ such that, for any $f\in\calF^1_{\text{Lip}}$, there exists a function $\hat{f}_t\in\mathcal{C}_t$ satisfying $\|f-\hat{f}_t\|_{2}\leq\alpha_t$. 
Let $\hat{f}_0\equiv0$, and for any $T\in\mathbb{N}^{+}$, we have the following chaining expression for $f$:
\begin{align*}
    f=f-\hat{f}_T+\sum_{t=1}^{T}(\hat{f}_t-\hat{f}_{t-1}).
\end{align*}
Hence, for any $T\in\mathbb{N}^{+}$, we decompose the Rademacher complexity as
\begin{align}\label{eq:stoch-err-bound-1}
    \mathcal{R}_n(\calF^1_{\text{Lip}}|g_{\hat{\btheta}_1})\leq& \bbE_{\bm{\xi}}\sup_{f\in\calF^1_\text{Lip}}\frac{1}{n_1}\sum_{i=1}^{n_1} \xi_i \left\{f(\bX_i,  g_{\hat{\btheta}_1}(\bX_i,\bZ_i))-\hat{f}_T(\bX_i,  g_{\hat{\btheta}_1}(\bX_i,\bZ_i))\right\}\\
    \nonumber&+\bbE_{\bm{\xi}}\sup_{f\in\calF^1_\text{Lip}}\frac{1}{n_1}\sum_{i=1}^{n_1} \xi_i \left[\sum_{t=1}^{T}\left\{\hat{f}_t(\bX_i,  g_{\hat{\btheta}_1}(\bX_i,\bZ_i))-\hat{f}_{t-1}(\bX_i,  g_{\hat{\btheta}_1}(\bX_i,\bZ_i)))\right\}\right].
\end{align}
By Cauchy-Schwarz inequality, the first term in (\ref{eq:stoch-err-bound-1}) can be bounded by
\begin{align*}
    \frac{1}{n_1}\sum_{i=1}^{n_1} \xi_i \left\{f(\bX_i,  g_{\hat{\btheta}_1}(\bX_i,\bZ_i))-\hat{f}_T(\bX_i,  g_{\hat{\btheta}_1}(\bX_i,\bZ_i))\right\}\leq\|\bm{\xi}\|_2\|f-\hat{f}_T\|_2\leq \alpha_T.
\end{align*}
Then, the second term in (\ref{eq:stoch-err-bound-1}) is the summation of the empirical Rademacher chaos w.r.t. the function classes $\mathcal{C}_t-\mathcal{C}_{t-1}=\{f_t-f_{t-1}:f_t\in\mathcal{C}_{t},f_{t-1}\in\mathcal{C}_{t-1}\}$, $t=1,2,\ldots,T$. 
By triangle inequality,
\begin{align*}
    \|\hat{f}_t-\hat{f}_{t-1}\|_2\leq \|\hat{f}_t-f\|_2+\|f-\hat{f}_{t-1}\|_2\leq \alpha_{t}+\alpha_{t-1}=3\alpha_t.
\end{align*}
Thus, applying Massart's lemma to $\mathcal{C}_t-\mathcal{C}_{t-1}$, $t=1,2,\ldots,T$, we can obtain 
\begin{align*}
    \bbE_{\bm{\xi}}\sup_{f\in\calF^1_\text{Lip}}\frac{1}{n_1}\sum_{i=1}^{n_1} \xi_i \left\{\hat{f}_t(\bX_i,  g_{\hat{\btheta}_1}(\bX_i,\bZ_i))-\hat{f}_{t-1}(\bX_i,  g_{\hat{\btheta}_1}(\bX_i,\bZ_i)))\right\}
    \leq\frac{3\alpha_t\sqrt{2\log(|\mathcal{C}_t||\mathcal{C}_{t-1}|)}}{\sqrt{n_1}}.
\end{align*}
Therefore, for any $T\in\bbN^{+}$, 
\begin{align*}
    \mathcal{R}_{n_1}(\calF^1_{\text{Lip}}|g_{\hat{\btheta}_1})\leq &\alpha_T+3\sum_{t=1}^{T}\frac{ \alpha_t\sqrt{2\log(|\mathcal{C}_t| |\mathcal{C}_{t-1}|)}}{\sqrt{n_1}}\\
    \leq &\alpha_T+6\sum_{t=1}^{T}\frac{ \alpha_t\sqrt{\log(|\mathcal{C}_t|)}}{\sqrt{n_1}}\\
    = &\alpha_T+6\sum_{t=1}^{T}\frac{ (\alpha_t-\alpha_{t-1})\sqrt{\log(|\mathcal{C}_t|)}}{\sqrt{n_1}}\\
    \leq &\alpha_T+12\int_{\alpha_{T+1}}^{\alpha_0}\sqrt{\frac{\log N(\delta,\calF^1_\text{Lip},\|\cdot\|_2)}{n_1}}d\delta,
\end{align*}
where $N(\delta,\calF^1_\text{Lip},\|\cdot\|_2)$ is the covering number of $\calF^1_\text{Lip}$ under $\|\cdot\|_2$.
For any arbitrary small $\delta>0$, we can choose $T$ such that $\alpha_{T+1}\leq\delta<\alpha_{T}$.
Thus,
\begin{align*}
     \mathcal{R}_{n_1}(\calF^1_{\text{Lip}}|g_{\hat{\btheta}_1})\leq &2\delta+12\int_{\delta/2}^{B}\sqrt{\frac{\log N(\delta,\calF^1_\text{Lip},\|\cdot\|_2)}{n_1}}d\delta\\
     \leq&\inf_{0<\delta<B}\left(4\delta+12\int_{\delta}^{B}\sqrt{\frac{\log N(\delta,\calF^1_\text{Lip},\|\cdot\|_2)}{n_1}}d\delta\right)\\
     \leq&\inf_{0<\delta<B}\left(4\delta+12\int_{\delta}^{B}\sqrt{\frac{\log N(\delta,\calF^1_\text{Lip},\|\cdot\|_{\infty})}{n_1}}d\delta\right).
\end{align*}
To find the upper bound of the covering number, we utilize Theorem 2.7.1 in \cite{van1996weak}. Under Condition \ref{C:XY},  there exists a positive constant $C$ such that
\begin{align*}
    \log N(\delta,\calF^1_\text{Lip}(\calX\times\mathcal{Y}),\|\cdot\|_{\infty})\leq C\left(\frac{\log {n_1} }{\delta}\right)^{p+1}.
\end{align*}
Finally, we can establish the upper bound in (\ref{eq:stoch-err-bound}) as
\begin{align*}
    \bbE_{\bX,\bZ}\mathcal{R}_n(\calF^1_{\text{Lip}}|g_{\hat{\btheta}_1})\leq & 4\delta +   \frac{12C}{\sqrt{n_1}}\int_{\delta}^{B}\left(\frac{\log {n_1} }{\epsilon}\right)^{\frac{p+1}{2}}d\epsilon\\
    \precsim& \delta + {n_1}^{-\frac{1}{2}}(\log {n_1})^{\frac{p+1}{2}}\delta^{1-\frac{p+1}{2}}\\
    = &{n_1}^{-\frac{1}{p+1}}\log {n_1} + {n_1}^{-\frac{1}{2}}(\log {n_1})^{\frac{p+1}{2}}({n_1}^{-\frac{1}{p+1}}\log {n_1})^{1-\frac{p+1}{2}}\\
    =&2{n_1}^{-1/(p+1)} \log {n_1},
\end{align*}
where the third line holds by taking $\delta=C{n_1}^{-\frac{1}{p+1}}\log {n_1}$. The proof is accomplished.
\end{proof}

\begin{lemma}\label{lem:stoch-f-XY}
    Suppose that Condition \ref{C:XY} holds. Then, 
    \begin{align}\label{eq:stoch-err-bound-XY}
        \bbE_{\bX_i,Y_i}\sup_{f\in\calF^1_\text{Lip}}\left\{\bbE_{\bX, \bZ} f(\bX,  Y)-\frac{1}{n}\sum_{i=1}^n f(\bX_i,  Y_i)\right\}\precsim \sqrt{p+1}n^{-1/(p+1)}.
    \end{align} 
\end{lemma}

\begin{proof}
When the third moment of $(\bX,Y)$ exists, the result of (\ref{eq:stoch-err-bound-XY}) can be directly obtained from Proposition 3.1 in \cite{NEURIPS2020_Lu}. Let $\bm{V}=(\bX,Y)\sim P_{\bX,Y}$. Suppose that Condition \ref{C:XY} holds. 
By Markov inequality,
\begin{align*}
         \Pr(\|\bm{V}\|>\log n) \leq \frac{\bbE\|V\| \mathbbm{1}\{\|\bm{V}\|>\log n\}}{\log n}=O\left(n^{-\frac{(\log n)^\delta}{p+1}} / \log n\right).
\end{align*}
Thus,
\begin{align*}
\bbE|\bm{V}|^3=\int_{0}^{\infty}3t^2P(|\bm{V}|>t)dt= \int_{0}^{\infty} O(1)3t\exp\left(-\frac{t^{1+\delta}}{d+1}\right)dt<\infty.
\end{align*}

\end{proof}

\begin{lemma}[Theorem 2.1 in \cite{Yang2021}]\label{lem:g-approx-err}
    Suppose that Conditions \ref{C:XY}-\ref{C:G-class} hold. Then,
    \begin{align}\label{eq:g-approx-err}
        \min_{g_{\btheta_1}\in\calG_{\btheta_1}} W(P_{\bX, g_{\btheta_1}},P_{\bX,Y})\precsim2^{1/(p+m)}{n_1}^{-1/(p+m)}.
    \end{align}
\end{lemma}

\begin{lemma}\label{lem:f-approx-err}
    Suppose that Conditions \ref{C:XY}-\ref{C:F-class} hold. Then,
    \begin{align}\label{eq:f-approx-err}
    \sup_{f\in\calF^1_\text{Lip}}\inf_{f_{\bphi_1}\in\calF_{\bphi_1}}\|f-f_{\bphi_1}\|_{\infty}\leq 76(p+1)^{3/2}{n_1}^{-1/(p+1)}\log {n_1}.
\end{align}
\end{lemma}

\begin{remark}
Lemmas \ref{lem:g-approx-err}-\ref{lem:f-approx-err} establish the approximation error of the generator network $g_{\btheta_1}$ and the discriminator network $f_{\phi_1}$, respectively. It is noted that they are non-asymptotic. 
That is,  (\ref{eq:g-approx-err}) and (\ref{eq:f-approx-err}) are valid for arbitrary networks $g_{\btheta_1}$ and $f_{\bphi_1}$ satisfying Condition \ref{C:G-class} and Condition \ref{C:F-class}, regardless of $n\to\infty$.
\end{remark}

\begin{proof}
The left hand side of (\ref{eq:f-approx-err}) is the approximation error of discriminator network $f_{\bphi_1}$, which gets smaller with network class $\calF_{\bphi_1}$ getting larger, intuitively. Refer to Corollary 3.1 in \cite{jiao2023AOS}, if the Lipschitz continuous function $f$ is defined on the bounded support $\calX\times\mathcal{Y}=[0,1]^{p+1}$. Then, for any $N_{f_1},M_{f_1}\in \bbN^{+} $, there exists a function $f_{\bphi_1}$ implemented by a ReLu network with width 
$\calW_{\bphi_1}=152 (p+1)^2 3^{p+1}N_{f_1}\lceil\log_2(8N_{f_1})\rceil$ and depth $\calL_{\bphi_1}=84M_{f_1}\lceil \log_2 (8M_{f_1})\rceil+2(p+1)$
such that
\begin{align*}
    \|f-f_{\bphi_1}\|_{\infty}\leq 76 (p+1)^{3/2}(N_{f_1}M_{f_1})^{-2/(p+1)}.
\end{align*}
Based on this result, when Condition \ref{C:XY} holds and network class $\calF_{\bphi_1}$ satisfies Condition \ref{C:F-class},
\begin{align}\label{eq:f-approx-err}
    \sup_{f\in\calF^1_\text{Lip}}\inf_{f_{\bphi_1}\in\calF_{\bphi_1}}\|f-f_{\bphi_1}\|_{\infty}\leq 76(p+1)^{3/2}(N_{f_1}M_{f_1})^{-2/(p+1)}\log n_1.
\end{align}
\end{proof}

\begin{lemma}\label{lem:para-g-theta-ast}
     Under Conditions \ref{C:XY}-\ref{C:F-class}, when $p=o(\log n_1)$, for $ g_{\btheta^{\ast}_1}\in\arg\min_{g_{\btheta_1}\in\calG_{\btheta_1}}W(P_{\bX,g_{\btheta_1}},P_{\bX,Y})$, (i) its first-layer weights $\bm{w}_{0}(\btheta^{\ast})=(\bm{\mu}^{\ast},\bm{\nu}^{\ast},\bm{\upsilon}^{\ast})$ satisfies that $\|\bm{\mu}^{*[,j]}\|_2\neq 0, \ \forall j=1,\ldots,p_s$; (ii) there exists a subset of $g_{\btheta^{\ast}_1}$, whose first-layer parameters $\bm{w}_{0}(\btheta^{\ast})=(\bm{\mu}^{\ast},\bm{0},\bm{\upsilon}^{\ast}).$
\end{lemma}

\begin{remark}
The proof of (i) is carried out through a contradiction.
    We first assume that there exists a $\|\bm{\mu}^{\ast[,j]}\|_2=0$ with $j\in\{1,\ldots,p_s\}$. 
    Then, the proof is accomplished by showing that the approximation error of $g_{\btheta^{\ast}_1}$ with such parameters cannot converge to zero, which contradicts Lemma \ref{lem:g-approx-err}.
    The proof of (ii) follows a constructive approach.
\end{remark}

\begin{proof}
Recall that the noise-outsourcing lemma guarantees that there exists a measurable function $g^{\ast}$ such that $W(P_{\bX,g^{\ast}},P_{\bX,Y})=0$.
Hence, by the definition of $g_{\btheta^{\ast}_1}$ and Lemma \ref{lem:g-approx-err}, under Conditions \ref{C:XY}-\ref{C:G-class}, 
\begin{align}\label{eq:W-dist-g-theta-star}    W(P_{\bX,g_{\btheta^{\ast}_1}},P_{\bX,g^{\ast}})=W(P_{\bX,g_{\btheta^{\ast}_1}},P_{\bX,Y})\precsim2^{1/(p+m)}{n_1}^{-1/(p+m)}\to0,
\end{align}
when $p=o(\log n_1)$ and $n_1\to\infty$.
Then, we prove two statements  sequentially.

(i).  Suppose that there exists a $g_{\btheta^{\ast}_1}$ whose first-layer weights $\bm{w}_{0}(\btheta^{\ast})=(\bm{\mu}^{\ast},\bm{\nu}^{\ast},\bm{\upsilon}^{\ast})$ satisfy $\|\bm{\mu}^{*[,j]}\|_2=0$ for at least one $j\in\{1,\ldots,p_s\}$.
We introduce a random variable $V$, and for any random vector $\bX\in\calX$, we construct the vector $\bX^{[j]}(0)=(X^{[1]},\ldots,X^{[j-1]},0,X^{[j+1]},\ldots,X^{[p]})$, so that $g_{\btheta_1^{\ast}}(\bX^{[j]}(0),\bZ)=g_{\btheta_1^{\ast}}(\bX,\bZ)$.
Under Condition \ref{C:significant}, for the Wasserstein distance between $P_{\bX,g_{\btheta^{\ast}_1}}$ and $P_{\bX,g^{\ast}}$, 
\begin{align*}
    &W(P_{\bX,g_{\btheta^{\ast}_1}},P_{\bX,g^{\ast}})\\
    =&\sup _{f\in\calF_{\text{Lip}}^1}\left\{\bbE_{\bX, \bZ} f(\bX, g_{\btheta^{\ast}_1}(\bX,\bZ))-\bbE_{\bX, \bZ} f(\bX, g^{\ast}(\bX,\bZ))\right\}\\
    \geq&\bbE_{\bX,\bZ}|\tilde{f}(\bX,g_{\btheta^{\ast}_1}(\bX,\bZ))-\tilde{f}(\bX,g^{\ast}(\bX,\bZ))|\\
    =&\bbE_{\bX,\bZ}|g_{\btheta^{\ast}_1}(\bX,\bZ)-g^{\ast}(\bX,\bZ)|\\
    \geq&|\bbE_{\bX,\bZ}g_{\btheta^{\ast}_1}(\bX^{[j]}(0),\bZ)-\bbE_{\bX,\bZ}g^{\ast}(\bX,\bZ)|\\
     =&|\bbE_{\bX,\bZ}g_{\btheta^{\ast}_1}(\bX^{[j]}(0),\bZ)-\bbE_{\bX,\bZ}g^{\ast}(\bX^{[j]}(0),\bZ)+\bbE_{\bX,\bZ}g^{\ast}(\bX^{[j]}(0),\bZ)-\bbE_{\bX,\bZ}g^{\ast}(\bX,\bZ)|\\
     \geq &\left||\bbE_{\bX,\bZ}g_{\btheta^{\ast}_1}(\bX^{[j]}(0),\bZ)-\bbE_{\bX,\bZ}g^{\ast}(\bX^{[j]}(0),\bZ)|-|\bbE_{\bX,\bZ}g^{\ast}(\bX^{[j]}(0),\bZ)-\bbE_{\bX,\bZ}g^{\ast}(\bX,\bZ)|\right|\\
     \geq&c_s,
\end{align*}
where $\tilde{f}:\calX\times\mathcal{Y}\to\mathbb{R}$ is defined as $\tilde{f}(\bX,Y)=\bX+Y$, belonging to $\calF_{\text{Lip}}^1$.  
$c_s>0$ is a positive constant defined in Condition \ref{C:significant}. The fifth line holds by Condition \ref{C:XY}.  It contradicts (\ref{eq:W-dist-g-theta-star}).
Thus, for any $j=1,\ldots,p_s$, $\|\bm{\mu}^{\ast[,j]}\|\neq 0$.

(ii). Pick any $g_{\btheta^{\ast}_1}$ with first-layer weights $\bm{w}_{0}(\btheta^{\ast})=(\bm{\mu}^{\ast},\bm{\nu}^{\ast},\bm{\upsilon}^{\ast})$, shift $\bm{\nu}^{\ast}=\bm{0}$, and denote the shifted function by $\tilde{g}_{\btheta^{\ast}_1}$.
Then, we have 
\begin{align*}
   W(P_{\bX,\tilde{g}_{\btheta^{\ast}_1}},P_{\bX,g^{\ast}})=&\sup _{f\in\calF_{\text{Lip}}^1}\left\{\bbE_{\bX, \bZ} f(\bX, \tilde{g}_{\btheta^{\ast}_1}(\bX,\bZ))-\bbE_{\bX, \bZ} f(\bX, g^{\ast}(\bX,\bZ))\right\}\\
   \leq &\bbE_{\bX, \bZ}\|(\bX,\tilde{g}_{\btheta^{\ast}_1}(\bX,\bZ))^{\top}-(\bX,g^{\ast}(\bX,\bZ))^{\top}\|_2  \\
   =&\bbE_{\bX, \bZ}|\tilde{g}_{\btheta^{\ast}_1}(\bX,\bZ)-g^{\ast}(\bX,\bZ)|\\
   =&\bbE_{\bX, \bZ}|\tilde{g}_{\btheta^{\ast}_1}(\bX,\bZ)-{g}_{\btheta^{\ast}_1}((\bX^s,\bm{0}),\bZ)+{g}_{\btheta^{\ast}_1}((\bX^s,\bm{0}),\bZ)-g^{\ast}(\bX,\bZ)|\\
   \leq&\bbE_{\bX, \bZ}|\tilde{g}_{\btheta^{\ast}_1}(\bX,\bZ)-{g}_{\btheta^{\ast}_1}((\bX^s,\bm{0}),\bZ)|+\bbE_{\bX, \bZ}|{g}_{\btheta^{\ast}_1}((\bX^s,\bm{0}),\bZ)-g^{\ast}(\bX,\bZ)|,
\end{align*}
where the second line holds as $f$ is a 1-Lipschitz function.
The first term in the right hand side equals zero as
\begin{align*}
    &\bbE_{\bX, \bZ}|\tilde{g}_{\btheta^{\ast}_1}(\bX,\bZ)-{g}_{\btheta^{\ast}_1}((\bX^s,\bm{0}),\bZ)|\\
    =&\bbE_{\bX, \bZ}\left|\bm{w}_{\calL_{\calG_{\btheta_1}}}\sigma(\ldots\bm{w}_1\sigma(\bm{\mu}^{\ast}\bX^{s}+\bm{0}\bX^c+\bm{\upsilon}^{\ast}\bZ+\bm{b}_0)+\bm{b}_1)+\ldots)+\bm{b}_{\calL_{\calG_{\btheta_1}}}\right.\\
    & \left.-\bm{w}_{\calL_{\calG_{\btheta_1}}}\sigma(\ldots\bm{w}_1\sigma(\bm{\mu}^{\ast}\bX^{s}+\bm{\nu}^{\ast}\bm{0}+\bm{\upsilon}^{\ast}\bZ+\bm{b}_0)+\bm{b}_1)+\ldots)+\bm{b}_{\calL_{\calG_{\btheta_1}}}\right|=0.
\end{align*}
Next, consider the second term. Similarly for $\tilde{f}(\cdot)$, 
\begin{align*}
    &\bbE_{\bX, \bZ}|{g}_{\btheta^{\ast}_1}((\bX^s,\bm{0}),\bZ)-g^{\ast}(\bX,\bZ)|\\
    =&\bbE_{\bX, \bZ}\left|{g}_{\btheta^{\ast}_1}((\bX^s,\bX^c),\bZ)-g^{\ast}((\bX^s,\bX^c),\bZ)\mid \bX^c=\bm{0} \right|\\
    \leq&\bbE_{\bX, \bZ}|{g}_{\btheta^{\ast}_1}(\bX,\bZ)-g^{\ast}(\bX,\bZ)|\\
    =&\bbE_{\bX,\bZ}|\tilde{f}(\bX,g_{\btheta^{\ast}_1}(\bX,\bZ))-\tilde{f}(\bX,g^{\ast}(\bX,\bZ))|\\
    \leq &\sup _{f\in\calF_{\text{Lip}}^1}\left\{\bbE_{\bX, \bZ} f(\bX, g_{\btheta^{\ast}_1}(\bX,\bZ))-\bbE_{\bX, \bZ} f(\bX, g^{\ast}(\bX,\bZ))\right\}\\
    =&W(P_{\bX,g_{\btheta^{\ast}_1}},P_{\bX,g^{\ast}})\\
    \to& 0, \quad \text{as }n_1\to\infty,
\end{align*}
where the first line holds because {$g^{\ast}(\bX,\bZ)$ is trivial to $\bX^c$} by Definition \ref{def:significance}.
\end{proof}

\begin{lemma}\label{lem:Dinh2020-expand} 
When Conditions \ref{C:XY}-\ref{C:F-class} hold and $p=o(\log n_1)$, $\Theta$ has the following two properties.
        (i)  For all $\btheta\in\Theta$ and $j=1,\ldots,p_s $, there exists $c_{\mu}>0$ such that $\|\bm{\mu}^{[,j]}\|_2\geq c_{\mu}$. 
        (ii) For any $\btheta\in\Theta$, if we set its $\bm{\nu}$-components to zero, it still belongs to $\Theta$.
\end{lemma}

\begin{proof}
Under the same conditions of Lemma \ref{lem:para-g-theta-ast}, for any network function $g_{\btheta_1}$ with $\btheta\in\Theta$,
\begin{align*}
     W(P_{\bX, g_{\btheta_1}},P_{\bX,g^{\ast}})\precsim  (p+1)^{3/2}{n_1}^{-2/(p+1)}\log n_1+2^{1/(p+m)}{n_1}^{-1/(p+m)},
\end{align*}
where $g^{\ast}(\bx,\bz)\sim P_{Y|\bX=\bx}$ for any $\bx\in\calX$ given by the noise-outsourcing lemma.
When $p=o(\log {n_1})$ and ${n_1}\to\infty$, the Wasserstein distance $W(P_{\bX, g_{\btheta_1}},P_{\bX,g^{\ast}})\to0$, implying that $g_{\btheta_1}(\bx,\bz)\to g^{\ast}(\bx,\bz)$, $\forall \bx\in\calX,\bz\in\calZ$.

(i) Assume that no such $c_{\mu}$ exists. Thus, there exists $\tilde{\btheta}_1\in\Theta$ such that its first-layer weights $\|\tilde{\bm{\mu}}^{[,j]}\|_{2}\to0$ as $n_1\to\infty$ for a $j\in\{1,\ldots,p_s\}$.
Since the first $p_s$ components of $\bX$ are assumed to be important, for such $j$, $(\bx^{s[-j]},\bx^{c})=\bx^{[-j]}$. Then, we have
\begin{align*}
     |g_{\tilde{\btheta}_1}(\bx,\bz)-g_{\tilde{\btheta}_1}(\bx^{[-j]},\bz)|\to0, \forall \bx\in\calX,\bz\in\calZ.
\end{align*}
As $\tilde{\btheta}_1\in\Theta$,
\begin{align*}
     g_{\btheta_1}(\bx,\bz)\to g^{\ast}(\bx,\bz),\text{ and }  g_{\btheta_1}(\bx^{[-j]},\bz)\to g^{\ast}(\bx^{[-j]},\bz), \quad \forall \bx\in\calX,\bz\in\calZ.
\end{align*}
Combining the above two expressions leads to the conclusion that 
\begin{align*}
    g^{\ast}(\bx,\bz)=g^{\ast}(\bx^{[-j]},\bz),
\end{align*}
indicating that $P_{Y|\bX=\bx}\sim P_{Y|\bX^{[-j]}=\bx^{[-j]}}$. 
According to Definition \ref{def:significance}, it contradicts the fact that $\tilde{\bm{\mu}}$ connects all the important variables. 

(ii) 
Let $\kappa(\btheta_1)$ be the function that set the $\bm{\nu}$-component of $\btheta$ as a zero matrix.
When the conditions of Lemma \ref{lem:para-g-theta-ast} hold, 
there exists a $g_{\btheta^{\ast}_1}$ with first-layer weights $\bm{w}_{0}(\btheta^{\ast})=(\bm{\mu}^{\ast},\bm{0},\bm{\upsilon}^{\ast}).$
Thus, 
\begin{align*}
    &W(P_{\bX,g_{\btheta_1}(\bX,\bZ)},P_{\bX,Y})=W(P_{\bX,g_{\btheta^{\ast}_1}(\bX,\bZ)},P_{\bX,Y})=W(P_{\bX,g_{\btheta^{\ast}_1}((\bX^s,\bm{0}),\bZ)},P_{\bX,Y})\\
    =&W(P_{\bX,g_{\btheta_1}((\bX^s,\bm{0}),\bZ)},P_{\bX,Y})=W(P_{\bX,g_{\kappa(\btheta_1)}(\bX,\bZ)},P_{\bX,Y}),
\end{align*}
leading to the conclusion that $\kappa(\btheta_1)\in\Theta$.
\end{proof}

\clearpage
\subsection{Proof of Theorem \ref{thm:vary-p}}

The proof is accomplished in four steps. 

{\it Step 1.} Error analysis of $\bbE[d(\hat{\btheta}_1,\Theta)]$.

First, we construct a generator network $g_{\hat{\bm{\vartheta}}_1}$ with parameter $\hat{\bm{\vartheta}}_1$ obtained by 
\begin{align*}
    \hat{\bm{\vartheta}}_1=\arg\min_{\bm{\vartheta}\in\Theta}\|\hat{\btheta}_1-\bm{\vartheta}\|_2,
\end{align*}    
such that $\bbE[d(\hat{\btheta}_1,\Theta)]=\|\hat{\btheta}_1-\hat{\bm{\vartheta}}_1\|_2$.

Under Condition \ref{C:Lojasewica}, there exists a positive constant $a$ such that 
\begin{align}\label{eq:error-analysis}
  d(\btheta,\Theta)^{a} \precsim & \{ W(P_{\bX,g_{\btheta_1}(\bX,\bZ)},P_{\bX, Y})-W(P_{\bX, g_{\btheta^{\ast}_1}(\bX,\bZ)},P_{\bX, Y})\}\\
  \nonumber= & \{ W(P_{\bX,g_{\btheta_1}(\bX,\bZ)},P_{\bX, Y})-W(P_{\bX, g_{\hat{\bm{\vartheta}}_1}(\bX,\bZ)},P_{\bX, Y})\}\\
  \nonumber=& \bbE\left[\sup_{f\in\calF^1_\text{Lip}}\left\{\bbE_{\bX, \bZ} f(\bX,  g_{\hat{\btheta}_1}(\bX,\bZ))-\bbE_{\bX, Y} f(\bX, Y)\right\}\right.\\
  \nonumber &\quad \left.-\sup_{f\in\calF^1_\text{Lip}}\left\{\bbE_{\bX, \bZ } f(\bX,  g_{\hat{\bm{\vartheta}}_1}(\bX,\bZ))-\bbE_{\bX, Y} f(\bX, Y)\right\}\right],
\end{align}
where $\calF_{\text{Lip}}^1=\{f:\mathbb{R}^{p+1}\to\mathbb{R},|f(\bm{u})-f(\bm{v})|\leq \|\bm{u}-\bm{v}\|_2, \bm{u},\bm{v}\in\mathbb{R}^{p+1}\}$ is the 1-Lipschitz function class, and the second line holds by the definition of $\Theta$. For the convenience of presentation, we introduce the following notations: 
\begin{align*}
   &R(\btheta;f)\coloneqq \bbE_{\bX,\bZ} f(\bX,g_{\btheta_1}(\bX,\bZ))-\bbE_{\bX,Y} f(\bX,Y),\\
   &R_{n_1}(\btheta;f)\coloneqq\frac{1}{n_1}\sum_{i=1}^{n_1} f(\bX_i,g_{\btheta_1}(\bX_i,\bZ_i))-\frac{1}{n_1}\sum_{i=1}^{n_1}f(\bX_i,Y_i).
\end{align*}
Then, 
\begin{small}
\begin{align*}
    &d(\btheta,\Theta)^{a}\precsim  \bbE\left[\sup_{f\in\calF^1_\text{Lip}}R(\hat{\btheta}_1;f)-\sup_{f\in\calF^1_\text{Lip}}R(\hat{\bm{\vartheta}}_1;f)\right]\\
        \leq&\bbE\left[\sup_{f\in\calF^1_\text{Lip}}\left\{R(\hat{\btheta}_1;f)-R_{n_1}(\hat{\btheta}_1;f)\right\}+\sup_{f\in\calF^1_\text{Lip}}R_{n_1}(\hat{\btheta}_1;f)-\sup_{f_{\bphi_1}\in\calF_{\bphi_1}}R_{n_1}(\hat{\btheta}_1;f_{\bphi_1})+\sup_{f_{\bphi_1}\in\calF_{\bphi_1}}R_{n_1}(\hat{\btheta}_1;f_{\bphi_1})\right.\\
    &-\left. \sup_{f_{\bphi_1}\in\calF_{\bphi_1}}R_{n_1}(\hat{\bm{\vartheta}}_1;f_{\bphi_1})+\sup_{f_{\bphi_1}\in\calF_{\bphi_1}}R_{n_1}(\hat{\bm{\vartheta}}_1;f_{\bphi_1})-\sup_{f\in\calF^1_\text{Lip}}R_{n_1}(\hat{\bm{\vartheta}}_1;f)+\sup_{f\in\calF^1_\text{Lip}}\left\{R_{n_1}(\hat{\bm{\vartheta}}_1;f)-R(\hat{\bm{\vartheta}}_1;f)\right\}
    \right]\\
    \leq&\bbE\left[\sup_{f\in\calF^1_\text{Lip}}\left\{R(\hat{\btheta}_1;f)-R_{n_1}(\hat{\btheta}_1;f)\right\}+\sup_{f\in\calF^1_\text{Lip}}\inf_{f_{\bphi_1}\in\calF_{\bphi_1}}\left\{R_{n_1}(\hat{\btheta}_1;f)-R_{n_1}(\hat{\btheta}_1;f_{\bphi_1})\right\}+R_{n_1}(\hat{\btheta}_1;\hat{f}_{\bphi_1})\right.\\
    & \left.-R_{n_1}(\hat{\bm{\vartheta}}_1;\hat{f}_{\bphi_1})+\sup_{f_{\bphi_1}\in\calF_{\bphi_1}}\inf_{f\in\calF^1_\text{Lip}}\left\{R_{n_1}(\hat{\bm{\vartheta}}_1;f_{\bphi_1})-R_{n_1}(\hat{\bm{\vartheta}}_1;f)\right\}+\sup_{f\in\calF^1_\text{Lip}}\left\{R_{n_1}(\hat{\bm{\vartheta}}_1;f)-R(\hat{\bm{\vartheta}}_1;f)\right\}\right]\\
    \leq & 2\bbE\left[\sup_{f\in\calF^1_\text{Lip}}\left\{R(\hat{\btheta}_1;f)-R_{n_1}(\hat{\btheta}_1;f)\right\}\right]+2\sup_{f\in\calF^1_\text{Lip}}\inf_{f_{\bphi_1}\in\calF_{\bphi_1}}\|f-f_{\bphi_1}\|_{\infty}+\bbE\left[R_{n_1}(\hat{\btheta}_1;\hat{f}_{\bphi_1})-R_{n_1}(\hat{\bm{\vartheta}}_1;\hat{f}_{\bphi_1})\right]\\
    \coloneqq & \mathcal{E}_1 + \mathcal{E}_2 + \mathcal{E}_3.
\end{align*}
\end{small}

{\it Step 2.} Bound the above three errors.

First, 
$\mathcal{E}_1$ is the stochastic error, which can be bounded by Lemmas \ref{lem:stochastic-bound} and \ref{lem:stoch-f-XY},
and $\mathcal{E}_2$ is the approximation error of the nuisance functional parameter $f$, whose upper bound is given in { (\ref{eq:f-approx-err})}.
Therefore, when Conditions \ref{C:XY} and \ref{C:F-class} hold and $\calF_{\bphi_1}$ further satisfies that $N_{f_1}M_{f_1}\geq\sqrt{n_1}$, we have
\begin{align}\label{eq:thm1-err-result-p1}
    \bbE[d(\hat{\btheta}_1,\Theta)^{a}]\precsim {n_1}^{-\frac{1}{p+1}} \log {n_1} +(p+1)^{\frac{1}{2}}{n_1}^{-\frac{1}{p+1}}+(p+1)^{\frac{3}{2}}{n_1}^{-\frac{2}{p+1}}\log n_1 +\mathcal{E}_3.
\end{align}

Next, consider $\mathcal{E}_3$.
Denote the first-layer weights of $g_{\hat{\btheta}_1}$ and $g_{\hat{\bm{\vartheta}}_1}$ by $\bm{w}_0(\hat{\btheta}_1)=(\hat{\bm{\mu}}_{\btheta_1},\hat{\bm{\nu}}_{\btheta_1},\hat{\bm{\upsilon}}_{\btheta_1})$ and $\bm{w}_0(\hat{\bm{\vartheta}}_1)=(\hat{\bm{\mu}}_{\bm{\vartheta}_1},\hat{\bm{\nu}}_{\bm{\vartheta}_1},\hat{\bm{\upsilon}}_{\bm{\vartheta}_1})$, respectively.
Due to the fact that $(\hat{\btheta}_1,\hat{\bphi_1})=\arg\min_{\btheta}\max_{\bphi_1}l_1(g_{\btheta_1},f_{\bphi_1})$, 
\begin{align}\label{eq:thm1-err-E3-p1}
    \mathcal{E}_3 = \bbE[R_{n_1}(\hat{\btheta}_1;f_{\hat{\bphi}_1})-R_{n_1}(\hat{\bm{\vartheta}}_1;f_{\hat{\bphi}_1})]\leq {\lambda_n} \bbE [J(\hat{\bm{\vartheta}}_1)-J(\hat{\btheta}_1)],
\end{align}
where $J(\btheta_1)= \sum_{j=1}^{p}\|\bm{w}_{0}(\btheta_1)^{[,j]}\|_2$.
Recall the left hand side of (\ref{eq:error-analysis}) $d(\hat{\btheta}_1,\Theta)^{a}=\|\hat{\btheta}_1-\hat{\bm{\vartheta}}_1\|^{a}_2$, where $a$ is a positive constant. We can bound ${\lambda_n}\bbE[J(\hat{\bm{\vartheta}_1})-J(\hat{\btheta}_1)]$ as follows:
 \begin{align*}
       {\lambda_n}\bbE[J(\hat{\bm{\vartheta}_1})-J(\hat{\btheta}_1)]=&{\lambda_n}\bbE\left[\left\{\sum_{j=1}^{p_s}\|\hat{\bm{\mu}}_{\bm{\vartheta}_1}^{[,j]}\|_2+\sum_{k=1}^{p_c}\|\hat{\bm{\nu}}_{\bm{\vartheta}_1}^{[,k]}\|_2\right\}-\left\{\sum_{j=1}^{p_s}\|\hat{\bm{\mu}}_{\btheta_1}^{[,j]}\|_2+\sum_{k=1}^{p_c}\|\hat{\bm{\nu}}_{\btheta_1}^{[,k]}\|_2\right\}\right]\\
         \leq &{\lambda_n}\bbE\left\{\sum_{j=1}^{p_s}\|\hat{\bm{\mu}}_{\bm{\vartheta}_1}^{[,j]}-\hat{\bm{\mu}}_{\btheta_1}^{[,j]}\|_2+\sum_{k=1}^{p_c}\|\hat{\bm{\nu}}_{\bm{\vartheta}_1}^{[,k]}-\hat{\bm{\nu}}_{\btheta_1}^{[,k]}\|_2\right\}\\
        \leq & {\lambda_n}\bbE [\sqrt{p_s}\|\hat{\bm{\mu}}_{\bm{\vartheta}_1}-\hat{\bm{\nu}}_{{\btheta}_1}\|_2+\sqrt{p_c}\|\hat{\bm{\mu}}_{\bm{\vartheta}_1}-\hat{\bm{\nu}}_{{\btheta}_1}\|_2]\\
        \leq& \sqrt{2}{\lambda_n} \bbE [\sqrt{p}\|\hat{\btheta}_1-\hat{\bm{\vartheta}}_1\|_2],
    \end{align*}
where the first inequality holds by the triangle inequality, and the second inequality holds by Cauchy-Schwarz inequality.
Then, as $a>2$, through Young's inequality
\begin{align}\label{eq:thm1-err-E3-p2}
    \bbE [{\lambda_n} \sqrt{p}\|\hat{\btheta}_1-\hat{\bm{\vartheta}}_1\|_2]\leq \bbE\left\{\frac{\|\hat{\btheta}_1-\hat{\bm{\vartheta}}_1\|_2^{a}}{a}+\frac{({\lambda_n} \sqrt{p})^{a/(a-1)}}{(a-1)/a}\right\}.
\end{align}

Finally, combining (\ref{eq:thm1-err-result-p1}), (\ref{eq:thm1-err-E3-p1}) and (\ref{eq:thm1-err-E3-p2}),
\begin{align*}
    \bbE\|\hat{\btheta}_1-\hat{\bm{\vartheta}}_1\|_2^{a}\precsim& {n_1}^{-\frac{1}{p+1}} \log {n_1} +(p+1)^{\frac{1}{2}}{n_1}^{-\frac{1}{p+1}}+(p+1)^{\frac{3}{2}}{n_1}^{-\frac{2}{p+1}}\log n_1 \\
    &+\bbE\left\{\frac{\|\hat{\btheta}_1-\hat{\bm{\vartheta}}_1\|_2^{a}}{a}+\frac{({\lambda_n} \sqrt{p})^{a/(a-1)}}{(a-1)/a}\right\},
\end{align*}
leading to the conclusion
\begin{align*}%\label{eq:rate-mu-1}
    \bbE[d(\hat{\btheta}_1,\Theta)^{a}]\precsim &{n_1}^{-\frac{1}{p+1}} \log {n_1} +(p+1)^{\frac{1}{2}}{n_1}^{-\frac{1}{p+1}}+(p+1)^{\frac{3}{2}}{n_1}^{-\frac{2}{p+1}}\log n_1 +({\lambda_n}\sqrt{p})^{\frac{a}{a-1}},\\
    \precsim& {n_1}^{-\frac{1}{p+1}} \log {n_1} +(p+1)^{\frac{3}{2}}{n_1}^{-\frac{2}{p+1}}\log n_1 +({\lambda_n}\sqrt{p})^{\frac{a}{a-1}}.
\end{align*}
The last line holds because {the second term is dominated} when $p=o(\log n_1)$, as
\begin{align*}
    \frac{{n_1}^{-\frac{1}{p+1}} \log {n_1}}{{n_1}^{-\frac{1}{p+1}}(p+1)^{\frac{1}{2}}}=\frac{\log n_1}{(p+1)^{\frac{1}{2}}}=\frac{\log n_1}{(o(\log n_1)+1)^{1/2}}\to\infty, \quad \text{as }n_1\to\infty. 
\end{align*}

{\it Step 3.} Bound $\bbE\|\hat{\bm{\nu}}_{\btheta_1}\|_2$. 

Let $\kappa(\hat{\bm{\vartheta}}_1)$ be the function that set the $\bm{\nu}$-component of $\bm{\vartheta}_1$ as a zero matrix.
Because $\hat{\bm{\vartheta}}_1\in\Theta$ and $p=o(\log n_1)$, by
Lemma \ref{lem:Dinh2020-expand}, 
$\kappa(\hat{\bm{\vartheta}}_1)$ belongs to $\Theta$ as well.
Thus,
\begin{align*}
    \bbE[R_{n_1}(\hat{\btheta}_1;f_{\hat{\bphi}_1})-R_{n_1}(\kappa(\hat{\bm{\vartheta}}_1);f_{\hat{\bphi}_1})]\leq {\lambda_n} \bbE [J(\kappa(\hat{\bm{\vartheta}}_1))-J(\hat{\btheta}_1)].
\end{align*}
Since $J(\kappa(\hat{\bm{\vartheta}}_1))=\sum_{j=1}^{p_s}\|\hat{\bm{\mu}}_{\bm{\vartheta}_1}^{[,j]}\|_2$, following the same error analysis in Step 1, we obtain
\begin{align*}
    \bbE\left[\sum_{k=1}^{p_c}\|\hat{\bm{\nu}}_{\btheta_1}^{[,k]}\|_2 \right]
   \leq & {\lambda_n}^{-1}\bbE[R_{n_1}(\kappa(\hat{\bm{\vartheta}}_1);f_{\hat{\bphi}_1})-R_{n_1}(\hat{\btheta}_1;f_{\hat{\bphi}_1})] +\bbE\left[\sum_{j=1}^{p_s}\|\hat{\bm{\mu}}_{\bm{\vartheta}_1}^{[,j]}\|_2-\sum_{j=1}^{p_s}\|\hat{\bm{\mu}}_{\btheta}^{[,j]}\|_2\right]\\
   \precsim& {\lambda_n}^{-1}\{{n_1}^{-\frac{1}{p+1}} \log {n_1}+(p+1)^{\frac{3}{2}}{n_1}^{-\frac{2}{p+1}}\log n_1\} +  \bbE [\sqrt{p}\|\hat{\btheta}_1-\hat{\bm{\vartheta}}_1\|_2]\\
   =&{\lambda_n}^{-1}\{{n_1}^{-\frac{1}{p+1}} \log {n_1} +(p+1)^{\frac{3}{2}}{n_1}^{-\frac{2}{p+1}}\log n_1\} +  \sqrt{p} \bbE [d(\hat{\btheta}_1,\Theta)].
\end{align*}

Combining the results of the previous three steps, we reach the conclusion that, under Conditions \ref{C:XY}-\ref{C:Lojasewica},  
there exists a positive constant $a> 2$ such that 
\begin{align*}
    &\bbE[d(\hat{\btheta}_1,\Theta)]\precsim \left\{{n_1}^{-\frac{1}{p+1}} \log {n_1}+(p+1)^{\frac{3}{2}}{n_1}^{-\frac{2}{p+1}}\log n_1 +({\lambda_n}\sqrt{p})^{\frac{a}{a-1}}\right\}^{\frac{1}{a}},\\
    &\bbE\|\hat{\bm{\nu}}_1\|_2\precsim{\lambda_n}^{-1}\{{n_1}^{-\frac{1}{p+1}} \log {n_1}+(p+1)^{\frac{3}{2}}{n_1}^{-\frac{2}{p+1}}\log n_1\} +  \sqrt{p} \bbE[d(\hat{\btheta}_1,\Theta)].
\end{align*}

{\it Step 4.} Determine tuning parameter $\lambda$. 

Note that the convergence rate of the first two terms in the right hand side of the above inequalities cannot be determined when $p=o(\log n_1)$, because 
\begin{align}\label{eq:domate-1}
     \log\left(\frac{{n_1}^{-\frac{1}{p+1}} \log {n_1}}{(p+1)^{\frac{3}{2}}{n_1}^{-\frac{2}{p+1}}\log n_1}\right)
    =\log\left(\frac{{n_1}^{\frac{1}{p+1}}}{(p+1)^{\frac{3}{2}}}\right)=\frac{\log {n_1}}{p+1}-\frac{3}{2}\log(p+1),
\end{align}
where the two terms both converge to positive infinity as $n_1\to\infty$.
To simplify calculation, we assume that $p=o(\log^{c} n_1)$ with $0<c<1$, so that (\ref{eq:domate-1}) goes to positive infinity as
\begin{align*}
    \frac{\log {n_1}}{p+1}-\frac{3}{2}\log(p+1)=&\frac{\log n_1}{o(1)\log^c n_1+1}-\frac{3}{2}\log(o(\log^c n_1)+1)\\
    =&\frac{\log^{1-c} n_1}{o(1)+\log^{-c}n_1}-\frac{3}{2}\log(o(1)\log^c n_1+1)\to\infty, \text{ as }n_1\to\infty,
\end{align*}
where the last equality holds because $\log^{1-c}n_1\gg\log\log n_1$ for any $c\in(0,1)$.
Thus, the result can be reduced to 
\begin{align*}
    \bbE[d(\hat{\btheta}_1,\Theta)]\precsim \left\{{n_1}^{-\frac{1}{p+1}} \log {n_1} +({\lambda_n}\sqrt{p})^{\frac{a}{a-1}}\right\}^{\frac{1}{a}},\ \bbE\|\hat{\bm{\nu}}_1\|_2\precsim{\lambda_n}^{-1}{n_1}^{-\frac{1}{p+1}} \log {n_1} +  \sqrt{p} \bbE[d(\hat{\btheta}_1,\Theta)],
\end{align*}
since the second term can be dominated by ${n_1}^{-\frac{1}{p+1}} \log {n_1}$.
Further, when $\lambda=O({n_1}^{-\frac{1}{2(p+1)}})$,% and $p=o(\log^c n_1 )$,
\begin{align*}
    \bbE[d(\hat{\btheta}_1,\Theta)]\precsim& \left\{{n_1}^{-\frac{1}{p+1}} \log {n_1} + {n_1}^{-\frac{a}{2(p+1)(a-1)}}\log^{\frac{ac}{2(a-1)}} n_1\right\}^{\frac{1}{a}}\\
    \precsim& \left({n_1}^{-\frac{1}{p+1}} \log^c {n_1}\right)^{\frac{1}{2(a-1)}},\\
    \precsim& \left({n_1}^{-\frac{1}{p+1}} \log {n_1}\right)^{\frac{1}{2(a-1)}}\\
    \to&0, \text{ as }n_1\to\infty, 
\end{align*}  
where the third line holds because $\frac{a}{2(a-1)}<1$ when $a>2$, and the last line holds because
 its logarithm converges to negative infinity when $c\in(0,1)$, which is 
\begin{align*}
    \log \left({n_1}^{-\frac{1}{p+1}} \log {n_1}\right)=&\left\{\log\log n_1 - \frac{\log n_1 }{p+1}\right\}\precsim\left\{\log\log n_1 - \frac{\log n_1 }{o(1)\log^{c} n_1+1}\right\}\\ \precsim&\left\{\log\log n_1 - \frac{\log^{1-c} n_1}{o(1)+\log^{-c}n_1}\right\} \to -\infty, \text{ as }n_1\to\infty.
\end{align*}
Similarly, we can obtain
   \begin{align*}
    \bbE\|\hat{\bm{\nu}}_1\|_2\precsim&{\lambda_n}^{-1}{n_1}^{-\frac{1}{p+1}} \log {n_1} +  \sqrt{p} \bbE[d(\hat{\btheta}_1,\Theta)]\\
    \precsim& {n_1}^{\frac{1}{2(p+1)}}{n_1}^{-\frac{1}{p+1}} \log {n_1}+ \log ^{c/2} n_1 \left({n_1}^{-\frac{1}{p+1}} \log^c {n_1}\right)^{\frac{1}{2(a-1)}}\\
    \precsim & {n_1}^{-\frac{1}{2(p+1)}} \log {n_1} + {n_1}^{-\frac{1}{2(a-1)(p+1)}}\log n_1\\
    \precsim& {n_1}^{-\frac{1}{2(a-1)(p+1)}} \log {n_1}.
   \end{align*}
It converges to zero since its logarithm converges to negative infinity as well, which is
\begin{align*}
\log \left({n_1}^{-\frac{1}{2(a-1)(p+1)}} \log {n_1}\right)  =&\log\log n_1 - \frac{\log n_1}{2(a-1)(p+1)}\precsim \log\log n_1 - \frac{1}{2(a-1)}\frac{\log n_1 }{o(1)\log^{c} n_1+1} \\
\precsim& \log\log n_1 - \frac{1}{2(a-1)}\frac{\log^{1-c} n_1}{o(1)+\log^{-c}n_1} \to -\infty, \text{ as }n_1\to\infty.
\end{align*}

\clearpage
\subsection{Proof of Corollary \ref{cor:fixed-p}}

When $p$ is fixed and $\lambda_n=O(n_1^{-\frac{1}{2(p+1)}})$, 
\begin{align*}
   \bbE[d(\hat{\btheta}_1,\Theta)]\precsim& \left\{{n_1}^{-\frac{1}{p+1}} \log {n_1}  +(p+1)^{\frac{3}{2}}{n_1}^{-\frac{2}{p+1}}\log n_1 +({\lambda_n}\sqrt{p})^{\frac{a}{a-1}}\right\}^{\frac{1}{a}}\\
   \precsim &\left\{{n_1}^{-\frac{1}{p+1}} \log {n_1} +{\lambda_n}^{\frac{a}{a-1}}\right\}^{\frac{1}{a}}\\
   =&\left\{{n_1}^{-\frac{1}{p+1}} \log {n_1} +{n_1^{-\frac{a}{2(p+1)(a-1)}}}\right\}^{\frac{1}{a}}\\
   \precsim& \left(\frac{\log n_1}{n_1}\right)^{\frac{1}{2(p+1)(a-1)}}.
\end{align*}
Similarly, we can obtain 
\begin{align*}
    \bbE\|\hat{\bm{\nu}}_1\|_2\precsim &{\lambda_n}^{-1}{n_1}^{-\frac{1}{p+1}} \log {n_1} + \bbE[d(\hat{\btheta}_1,\Theta)]\\
    \precsim & {n_1}^{\frac{1}{2(p+1)}}{n_1}^{-\frac{1}{p+1}} \log {n_1} + \left(\frac{\log n_1}{n_1}\right)^{\frac{1}{2(p+1)(a-1)}}\\
    = & {n_1}^{-\frac{1}{2(p+1)}} \log {n_1} + \left(\frac{\log n_1}{n_1}\right)^{\frac{1}{2(p+1)(a-1)}}\\
    \precsim&\left(\frac{\log n_1}{n_1}\right)^{\frac{1}{2(p+1)(a-1)}}.
\end{align*}

\clearpage
\subsection{Proof of Corollary \ref{cor:threshold}}
Recall that for the estimate $\hat{\btheta}_1$ obtained in (\ref{eq:emp-stage1}), $\bbE[d(\hat{\btheta}_1,\Theta)]=\|\hat{\btheta}_1-\hat{\bm{\vartheta}}_1\|_2$, where $\hat{\bm{\vartheta}}_1=\arg\min_{\bm{\vartheta}\in\Theta}\|\hat{\btheta}_1-\bm{\vartheta}\|_2$.
Let $\bm{w}_0(\hat{\btheta}_1)=(\hat{\bm{\mu}}_{\btheta_1},\hat{\bm{\nu}}_{\btheta_1},\hat{\bm{\upsilon}}_{\btheta_1})$ and $\bm{w}_0(\hat{\bm{\vartheta}}_1)=(\hat{\bm{\mu}}_{\bm{\vartheta}_1},\hat{\bm{\nu}}_{\bm{\vartheta}_1},\hat{\bm{\upsilon}}_{\bm{\vartheta}_1})$ be the first-layer weights of the generator networks $g_{\hat{\btheta}_1}$ and $g_{\hat{\bm{\vartheta}}_1}$, respectively.
Then, from Corollary \ref{thm:vary-p}, when  ${\lambda_n}=O({n_1}^{-\frac{1}{2(p+1)}})$ and $p=o(\log^c n_1)$ with $c\in(0,1)$,
\begin{align*}
   \bbE \|\hat{\bm{\mu}}_{\btheta_1}-\hat{\bm{\mu}}_{\bm{\vartheta}_1}\|_2\leq \bbE\|\hat{\btheta}_1-\hat{\bm{\vartheta}}_1\|_2=\bbE[d(\hat{\btheta}_1,\Theta)]\precsim \left({n_1}^{-\frac{1}{p+1}} \log {n_1}\right)^{\frac{1}{2(a-1)}}. 
\end{align*}
Lemma \ref{lem:Dinh2020-expand} shows that there exists a positive constant $c_{\mu}>0$ such that $\|\hat{\bm{\mu}}_{\bm{\vartheta}}^{[,j]}\|_2\geq c_{\mu}$ for any $j=1,\ldots,p_s$ by the fact $\hat{\bm{\vartheta}}\in\Theta$.
In other words, $\min_{1\leq j\leq p_s}\|\hat{\bm{\mu}}_{\bm{\vartheta}}^{[,j]}\|_2\geq c_{\mu}$.
Then, for each $j=1,\ldots,p_s$,
\begin{align*}
    &c_{\mu}  -\|\hat{\bm{\mu}}_{\btheta_1}^{[,j]}\|_2\leq  \min_{1\leq k\leq p_s}\|\hat{\bm{\mu}}_{\bm{\vartheta}}^{[,k]}\|_2-\|\hat{\bm{\mu}}_{\btheta_1}^{[,j]}\|_2\leq\|\hat{\bm{\mu}}_{\bm{\vartheta}}^{[,j]}\|_2-\|\hat{\bm{\mu}}_{\btheta_1}^{[,j]}\|_2\leq \left|\|\hat{\bm{\mu}}_{\btheta_1}^{[,j]}\|_2-\|\hat{\bm{\mu}}_{\bm{\vartheta}_1}^{[,j]}\|_2\right|\\
    \leq& \|\hat{\bm{\mu}}_{\btheta_1}^{[,j]}-\hat{\bm{\mu}}_{\bm{\vartheta}_1}^{[,j]}\|_2
    \leq\sum_{j=1}^{p_s}\|\hat{\bm{\mu}}_{\btheta_1}^{[,j]}-\hat{\bm{\mu}}_{\bm{\vartheta}_1}^{[,j]}\|_2\leq \sqrt{p_s}\|\hat{\bm{\mu}}_{\btheta_1}-\hat{\bm{\mu}}_{\bm{\vartheta}_1}\|_2      
   \leq \sqrt{p_s}\|\hat{\btheta}_1-\hat{\bm{\vartheta}}_1\|_2,
   %\leq\sqrt{p_s}\left({n_1}^{-\frac{1}{p+1}}\log {n_1}\right)^{\frac{1}{a}}.
\end{align*}
indicating that $\|\hat{\bm{\mu}}_{\btheta_1}^{[,j]}\|_2\geq c_{\mu}- \sqrt{p_s}\|\hat{\btheta}_1-\hat{\bm{\vartheta}}_1\|_2$.

By Markov's inequality, for any positive constant $0<c_0<c_{\mu}/\sqrt{p_s}$, 
\begin{align*}
    P(\|\hat{\btheta}_1-\hat{\bm{\vartheta}}_1\|_2\leq c_0)=1-P(\|\hat{\btheta}_1-\hat{\bm{\vartheta}}_1\|_2\geq c_0)\geq 1-\frac{\bbE\|\hat{\btheta}_1-\hat{\bm{\vartheta}}_1\|_2}{c_0}.
\end{align*} 
Theorem \ref{thm:vary-p} provides that there exists a positive constant $c_1$ such that
\begin{align*}
    \bbE\|\hat{\btheta}_1-\hat{\bm{\vartheta}}_1\|_2\leq c_1\left({n_1}^{-\frac{1}{p+1}} \log {n_1}\right)^{\frac{1}{2(a-1)}},
\end{align*}
so that for any $j=1,\ldots,p_s$,  
\begin{align}\label{eq:prob-mu}
    P\left(\|\hat{\bm{\mu}}_{\btheta_1}^{[,j]}\|_2\geq c_{\mu}-\sqrt{p_s}c_0\right)\geq 1-\frac{\bbE\|\hat{\btheta}_1-\hat{\bm{\vartheta}}_1\|_2}{c_0}\geq 1-\frac{c_1}{c_0}\left({n_1}^{-\frac{1}{p+1}} \log {n_1}\right)^{\frac{1}{2(a-1)}},
\end{align}
where $a>2$ is a constant, independent of $p$ and $n_1$. 
For the part connecting the unimportant variables,
by Theorem \ref{thm:vary-p}, for each $j=1,\ldots,p_c$,
\begin{align}\label{eq:prob-nu}
    P\left(\|\hat{\bm{\nu}}_{\btheta_1}^{[,j]}\|_2\geq c_{\mu}-\sqrt{p_s}c_0\right)\leq \frac{\bbE\|\hat{\bm{\nu}}_{\btheta_1}^{[,j]}\|_2}{c_{\mu}-\sqrt{p_s}c_0}\leq \frac{c_2{n_1}^{-\frac{1}{2(a-1)(p+1)}} \log {n_1}}{c_{\mu}-\sqrt{p_s}c_0},
\end{align}
where $c_{\mu}-\sqrt{p_s}c_0>0$ by the constriction of Markov's inequality, and $c_2>0$ is a positive constant independent of $p$ and $n$.
Taking (\ref{eq:prob-mu}) and (\ref{eq:prob-nu}) together into consideration, for any $j=1,\ldots,p_s$ and $k=1,\ldots,p_c$,
\begin{align*}
&P\left(\{\|\hat{\bm{\mu}}_{\btheta_1}^{[,j]}\|_2\geq c_{\mu}-\sqrt{p_s}c_0\}  \cap \{\|\hat{\bm{\nu}}_{\btheta_1}^{[,k]}\|_2\leq c_{\mu}-\sqrt{p_s}c_0\}\right)\\
=&P\left( \|\hat{\bm{\mu}}_{\btheta_1}^{[,j]}\|_2\geq c_{\mu}-\sqrt{p_s}c_0  \right)-P\left(\{\|\hat{\bm{\mu}}_{\btheta_1}^{[,j]}\|_2\geq c_{\mu}-\sqrt{p_s}c_0\} \cap\{\|\hat{\bm{\nu}}_{\btheta_1}^{[,k]}\|_2\geq c_{\mu}-\sqrt{p_s}c_0\} \right)\\
\geq&P\left( \|\hat{\bm{\mu}}_{\btheta_1}^{[,j]}\|_2\geq c_{\mu}-\sqrt{p_s}c_0  \right)-P\left( \|\hat{\bm{\nu}}_{\btheta_1}^{[,k]}\|_2\geq c_{\mu}-\sqrt{p_s}c_0  \right)\\
\geq & 1-\frac{c_1}{c_0}\left({n_1}^{-\frac{1}{p+1}} \log {n_1}\right)^{\frac{1}{2(a-1)}}-\frac{c_2}{c_{\mu}-\sqrt{p_s}c_0}{n_1}^{-\frac{1}{2(a-1)(p+1)}} \log {n_1}\\
\geq &1-2\max\left\{\frac{c_1}{c_0},\frac{c_2}{c_{\mu}-\sqrt{p_s}c_0}\right\}{n_1}^{-\frac{1}{2(a-1)(p+1)}} \log {n_1}\\
\to&1,\ \text{ as }n_1\to\infty,
\end{align*}
since ${n_1}^{-\frac{1}{2(a-1)(p+1)}} \log {n_1}$ converges to zero as $n_1\to\infty$.

\clearpage
\subsection{Proof of Theorem \ref{thm2:convergence-rate-stage2}}

Consider the Wasserstein distance between $P_{\hat{\bX}^2,g_{\hat{\btheta}_2}}$ and $P_{\bX^s,Y}$,
\begin{align*}
    &\mathbb{E}W(P_{\hat{\bX}^s,g_{\hat{\btheta}_2}},P_{\bX^s,Y})\\
    %=&\bbE \sup _{f\in\calF_{\text{Lip}}^1}\left\{\bbE_{\hat{\bX}^s, \bZ} f(\hat{\bX}^s, g_{\hat{\btheta}_2}(\hat{\bX}^s,\bZ))-\bbE_{\bX^s, Y} f(\bX^s, Y)\right\}\\
    =&\bbE \sup _{f\in\calF_{\text{Lip}}^1}\left\{\bbE_{\hat{\bX}^s, \bZ} f(\hat{\bX}^s, g_{\hat{\btheta}_2}(\hat{\bX}^s,\bZ))-\bbE_{\hat{\bX}^s, Y} f(\hat{\bX}^s, Y)+\bbE_{\hat{\bX}^s, Y} f(\hat{\bX}^s, Y)-\bbE_{\bX^s, Y} f(\bX^s, Y)\right\}\\
    \leq& \bbE \sup _{f\in\calF_{\text{Lip}}^1}\left\{\bbE_{\hat{\bX}^s, \bZ} f(\hat{\bX}^s, g_{\hat{\btheta}_2}(\hat{\bX}^s,\bZ))-\bbE_{\hat{\bX}^s, Y} f(\hat{\bX}^s, Y)\right\}+ \bbE \sup _{f\in\calF_{\text{Lip}}^1}\left\{\bbE_{\hat{\bX}^s, Y} f(\hat{\bX}^s, Y)-\bbE_{\bX^s, Y} f(\bX^s, Y)\right\}.
\end{align*}
Corollary \ref{cor:threshold} shows that 
$\hat{\bX}^s=\bX^s$ 
with probability $1-c_{p_s}{n_1}^{-\frac{1}{2(a-1)(p+1)}} \log {n_1}$, so that in the same probability, $\bbE \sup _{f\in\calF_{\text{Lip}}^1}\left\{\bbE_{\hat{\bX}^s, Y} f(\hat{\bX}^s, Y)-\bbE_{\bX^s, Y} f(\bX^s, Y)\right\}=0$.
Therefore, we only focus on the first term, which can be decomposed as:
\begin{align*}
    &\bbE \sup _{f\in\calF_{\text{Lip}}^1}\left\{\bbE_{\hat{\bX}^s, \bZ} f(\hat{\bX}^s, g_{\hat{\btheta}_2}(\hat{\bX}^s,\bZ))-\bbE_{\hat{\bX}^s, Y} f(\hat{\bX}^s, Y)\right\}\\
    =&\bbE \sup _{f\in\calF_{\text{Lip}}^1}\left\{\bbE_{\hat{\bX}^s, \bZ} f(\hat{\bX}^s, g_{\hat{\btheta}_2}(\hat{\bX}^s,\bZ))-\frac{1}{n_2}\sum_{i=1}^{n_2}f(\hat{\bX}_i^s, g_{\hat{\btheta}_2}(\hat{\bX}_i^s,\bZ_i))\right.\\
    &\left.+\frac{1}{n_2}\sum_{i=1}^{n_2}f(\hat{\bX}_i^s, g_{\hat{\btheta}_2}(\hat{\bX}_i^s,\bZ_i))-\frac{1}{n_2}\sum_{i=1}^{n_2}f(\hat{\bX}_i^s, Y_i)+\frac{1}{n_2}\sum_{i=1}^{n_2}f(\hat{\bX}_i^s, Y_i)-\bbE_{\hat{\bX}^s, Y} f(\hat{\bX}^s, Y)\right\}\\
    \leq&\bbE \sup _{f\in\calF_{\text{Lip}}^1}\left\{\bbE_{\hat{\bX}^s, \bZ} f(\hat{\bX}^s, g_{\hat{\btheta}_2}(\hat{\bX}^s,\bZ))-\frac{1}{n_2}\sum_{i=1}^{n_2}f(\hat{\bX}_i^s, g_{\hat{\btheta}_2}(\hat{\bX}_i^s,\bZ_i))\right\}\\
    &+\bbE \sup _{f\in\calF_{\text{Lip}}^1}\left\{\frac{1}{n_2}\sum_{i=1}^{n_2}f(\hat{\bX}_i^s, g_{\hat{\btheta}_2}(\hat{\bX}_i^s,\bZ_i))-\frac{1}{n_2}\sum_{i=1}^{n_2}f(\hat{\bX}_i^s, Y_i)\right\}\\
    &+\bbE \sup _{f\in\calF_{\text{Lip}}^1}\left\{\frac{1}{n_2}\sum_{i=1}^{n_2}f(\hat{\bX}_i^s, Y_i)-\bbE_{\hat{\bX}^s, Y} f(\hat{\bX}^s, Y)\right\}.
\end{align*}
Similar to the results in Lemmas \ref{lem:stochastic-bound}-\ref{lem:stoch-f-XY}, the first and third terms above have their upper bounds as follows:
\begin{align*}
    &\bbE \sup _{f\in\calF_{\text{Lip}}^1}\left\{\bbE_{\hat{\bX}^s, \bZ} f(\hat{\bX}^s, g_{\hat{\btheta}_2}(\hat{\bX}^s,\bZ))-\frac{1}{n_2}\sum_{i=1}^{n_2}f(\hat{\bX}_i^s, g_{\hat{\btheta}_2}(\hat{\bX}_i^s,\bZ_i))\right\}\precsim n_2^{\frac{1}{p_s+1}}\log n_2,\\
    &\bbE \sup _{f\in\calF_{\text{Lip}}^1}\left\{\frac{1}{n_2}\sum_{i=1}^{n_2}f(\hat{\bX}_i^s, Y_i)-\bbE_{\hat{\bX}^s, Y} f(\hat{\bX}^s, Y)\right\}\precsim\sqrt{p_s+1}n_2^{-\frac{1}{p_s+1}}.
\end{align*}
Then, when the second stage's network classes $\calF_{\bphi_2}$ and $\calG_{\btheta_2}$ satisfy Conditions \ref{C:F-class} and \ref{C:G-class},
\begin{align*}
    &\bbE \sup _{f\in\calF_{\text{Lip}}^1}\left\{\frac{1}{n_2}\sum_{i=1}^{n_2}f(\hat{\bX}_i^s, g_{\hat{\btheta}_2}(\hat{\bX}_i^s,\bZ_i))-\frac{1}{n_2}\sum_{i=1}^{n_2}f(\hat{\bX}_i^s, Y_i)\right\}\\
    \leq & 2\sup_{f\in\calF_{\text{Lip}}^1}\inf_{f_{\bphi_2\in\calF_{\bphi_2}}}\|f-f_{\bphi_2}\|+\inf_{g_{\btheta_2}\in\mathcal{G}_{\btheta_2}}\sup_{f_{\bphi_2\in\calF_{\bphi_2}}}\left\{\frac{1}{n_2}\sum_{i=1}^{n_2}f(\hat{\bX}_i^s, g_{\hat{\btheta}_2}(\hat{\bX}_i^s,\bZ_i))-\frac{1}{n_2}\sum_{i=1}^{n_2}f(\hat{\bX}_i^s, Y_i)\right\}\\
    \leq &2\sup_{f\in\calF_{\text{Lip}}^1}\inf_{f_{\bphi_2\in\calF_{\bphi_2}}}\|f-f_{\bphi_2}\| 
    \leq  76(p_s+1)^{3/2}{n_2}^{-\frac{2}{p_s+1}}\log n_2 
    \precsim  {n_2}^{-\frac{2}{p_s+1}}\log n_2,
\end{align*}
where the second and third lines hold by Lemma C.1 and Lemma C.3 in \cite{liu2021}, the fourth line holds by (\ref{eq:f-approx-err}), and the last line holds as $p_s$ is fixed.

\clearpage
\subsection{Proof of Proposition \ref{prop:surivival}}
Suppose that Conditions \ref{C6:censor-indep} and \ref{C7:integrability} hold. Then, 
according to Theorem 1.1 in \cite{stute1996}, for any $(x,t)\in\bbR^{d\times 1}$,
\begin{align*}
    \bbE|\hat{F}_{XT}(x,t)-F_{XT}(x,t)|\leq 2F_{XT}(x,t).
\end{align*}
Condition \ref{C:XY} guarantees that $M_3=\bbE_{X,T}|(X,T)|^3<\infty$. 
Last, following Theorem 3.1 in \cite{lei2020}, 
\begin{align*}
    \bbE W(\hat{F}_{X,T},F_{X,T})\leq C n^{-\frac{1}{d+1}},
\end{align*}
where $C$ is a positive constant depending on $M_3$.

\clearpage
\subsection{Computational algorithm for the penalized WGAN for censored survival data}
 \begin{algorithm}[h]
	\caption{Penalized WGAN for censored survival data}
	\label{alg:survival}
	\begin{algorithmic}[1]
	\Require{ Tuning parameter ${\lambda_n}$;  Minibatch size $v\leq n_1$; Clipping parameter $c$; Noise dimension $m$; Learning rates $\alpha_f$, $\alpha_g$. %; $\tilde{n}$, the batch size.
	}
	\For {number of training iterations in stage 1}
        \State Sample $\{(\bx_{bj},y_{bj},\Delta_{bj})\}_{j=1}^{n_b}$ from $\{(\bx_i,y_i,\Delta_i)\}_{i=1}^{n_1}$,
        and noise $\{\bz_j\}_{j=1}^{n_b}$ from $N(\bm{0},\bm{I}_{m})$.
        \State Order the samples as $y_{(b1)}\leq y_{(b2)}\leq\ldots\leq y_{(bn_b)}$, and denote $\bx_{(b1)},\ldots,\bx_{(bn_b)}$ and $\Delta_{(b1)},\ldots,\Delta_{(bn_b)}$ as the predictors and indicators corresponding to the ordered $y_{bj}$'s.
        \State Compute the Kaplan-Meier weights as
        \begin{align*}
            \omega_{(b1)}=\frac{\Delta_{(b1)}}{n_b},\quad \omega_{(bj)}=\frac{\Delta_{(j)}}{n_b-j+1}\prod_{k=1}^{j-1}\left(\frac{n_b-k}{n-k+1}\right)^{\Delta_{(k)}},\ j=2,\ldots,n_b.
        \end{align*}
		\State Update the discriminator $f_{\bphi_1}$ by ascending its stochastic gradient:
	\begin{align*}%\label{Alg-D}
	 & f_{\bphi_1}\gets\nabla_{\bphi_1}\left[ \frac{1}{n_b} \sum_{j=1}^{n_b} f_{\bphi_1}\left(\bx_{bj}, {g}_{\btheta}\left( \bx_{bj},\bz_{j}\right)\right)- \sum_{j=1}^{n_b}\omega_{(bj)}f_{\bphi_1}\left(\bx_{(bj)}, y_{(bj)}\right)\right],\\
    &\bphi_1 \gets \bphi_1 + \alpha_f\cdot \text{RMSProp}(\bphi_1,f_{\bphi_1}), \quad \bphi_1 \gets \text{clip}(\bphi_1,-c,c).
	\end{align*}
		\State Update the generator $g_{\btheta_1}$ by descending its stochastic gradient:
	\begin{align*}%\label{Alg-F}
	 &g_{\btheta_1}\gets \nabla_{\btheta} \left[\frac{1}{n_b} \sum_{j=1}^{n_b} {f}_{\bphi_1}(\bx_{bj}, g_{\btheta_1}( \bx_{bj},\bz_{j}))+{\lambda_n}\sum_{l=1}^p\|\bm{w}_{0}(\btheta_1)^{[,j]}\|_2 \right],\\
    &{\btheta} \gets {\btheta} - \alpha_g \cdot \text{RMSProp}({\btheta},g_{\btheta_1}).
	\end{align*}
	\EndFor
	\end{algorithmic}
\end{algorithm}

\clearpage
\subsection{Additional simulation study details}

\subsubsection{Additional simulation results for $p_s =30$: Distribution estimation}

Beyond those evaluated in the main text, as a ``byproduct’’, the proposed approach can also provide the conditional distribution estimation. For comparison, we also apply the conditional WGAN \citep[cWGAN,][]{liu2021} with the original high-dimensional predictors and the set of important predictors, respectively -- here the latter is an oracle approach and is not feasible in practice. 
To evaluate the conditional distribution estimation, we consider the mean squared error (MSE) of the estimated mean $\mathbb{E}(Y|\bX)$, the estimated conditional standard deviation $SD(Y|\bX)$, and the estimated conditional quantile of level $\tau$, $F^{-1}_{Y|\bX}(\tau|\bX)$, which are defined as:
\begin{align*}
    \text{MSE(mean)}=&\frac{1}{T}\sum_{i=1}^T \{\widehat{\mathbb{E}}(Y \mid \bX=\bx_i)-\mathbb{E}(Y \mid \bX=\bx_i) \}^2,\\
    \text{MSE(sd)}=&\frac{1}{T}\sum_{i=1}^T \{\widehat{SD}(Y \mid \bX=\bx_i)-SD(Y \mid \bX=\bx_i) \}^2,\\
    \operatorname{MSE}(\tau)=&\frac{1}{T} \sum_{i=1}^{T}\left\{\hat{F}_{Y \mid \bX}^{-1}\left(\tau \mid \bX=\bx_{i}\right)-F_{Y \mid X}^{-1}\left(\tau \mid \bX=\bx_{i}\right)\right\}^{2},
\end{align*}
where $T$ is the size of testing data, %{\color{red} T is previously used for time. change to a different notation} 
$\widehat{\mathbb{E}}(Y|\bX=\bx_i)=J^{-1}\sum_{j=1}^J \hat{g}_{\bm{\gamma}}(\bx_i,\bZ_{ij})$, $\widehat{SD}(Y|\bX=\bx_i)=[J^{-1}\sum_{j=1}^J\{\hat{g}_{\bm{\theta}}(\bx_i,\bZ_{ij})-\hat{\mathbb{E}}(Y|\bX=\bx_i)\}^2]^{1/2}$, and $F_{Y|\bX}^{-1}(\tau|\bX=\bx_i)$ is the $\tau$-th conditional quantile for a given $\bX=\bx_i$. The noise vector $\bZ$ is generated from $N(\bm{0},\bm{I}_{5})$, and $J=500$. Table \ref{tab:MSEs for d30} summarizes the average MSEs and the corresponding standard errors in parentheses. It can be seen that the MSEs of the proposed approach are smaller than cWGAN in all models. Further, they remain stable with the increasing dimension of $\bX$. %{\color{red} one sentence on relative performance comparing to oracle}

\begin{table}[h]
    \centering
        \caption{Mean squared error (MSE) of the estimated conditional mean, the estimated conditional standard deviation, and the estimated conditional quantile of level $\tau$ for M\ref{M1} to M\ref{M4} with $p_s =30$.}
        \begin{threeparttable}
        \resizebox{\textwidth}{!}{\begin{minipage}{\textwidth}
        \begin{tabular}{ccc| cc ccc}
        \hline
         & & & & & \multicolumn{3}{c}{$\tau$} \\
         \cline{6-8}
        & Method & $p$ & Mean & Sd &  0.25 & 0.50 & 0.75\\
        \hline
        \multirow{9}{*}{M\ref{M1}} & Oracle & 30 & 0.34(0.07) & 0.13(0.07) & 0.40(0.08) & 0.35(0.07) & 0.41(0.09) \\
         \cline{2-8}
         & \multirow{4}{*}{Proposed} & 100 & 0.52(0.08) & 0.21(0.09) &  0.63(0.13) & 0.53(0.08) & 0.61(0.12) \\
          & & 300 & 0.50(0.08) & 0.28(0.15) & 0.61(0.12) & 0.51(0.08) & 0.66(0.11)  \\
         & & 500 & 0.56(0.11) & 0.29(0.12) & 0.64(0.18) & 0.49(0.11) & 0.59(0.14) \\
         & & 1000 & 0.56(0.14) & 0.19(0.13) & 0.65(0.16) & 0.57(0.15) & 0.66(0.18)\\
         \cline{2-8}
         & \multirow{4}{*}{cWGAN} & 100 & 0.66(0.11) & 0.86(0.14) & 1.09(0.15) & 0.66(0.11) & 1.02(0.12) \\
          & & 300 & 2.23(0.28) & 0.91(0.03) & 2.52(0.29) & 2.23(0.28) & 2.77(0.39)\\
         & & 500 & 4.79(0.49) & 0.90(0.02) & 5.28(0.64) & 4.80(0.56) & 5.14(0.53) \\
         & & 1000 &12.21(0.80) & 0.89(0.02) & 12.67(0.88) & 12.21(0.80) & 12.56(0.90)\\     
    \hline 
    \multirow{9}{*}{M\ref{M2}} & Oracle & 30& - & - & 2.75(3.02) & 2.52(2.70) & 2.44(2.18) \\
         \cline{2-8}
         & \multirow{4}{*}{ Proposed } & 100 & - & - & 3.62(3.37) & 2.96(3.19) & 4.55(5.79) \\
          & & 300 & - & - & 5.92(5.04) & 3.31(2.81) & 3.88(3.31)\\
         & & 500 & - & - & 5.22(4.97) & 2.70(2.49) & 4.02(3.58)\\
         & & 1000 & - & - & 7.15(6.65) & 2.74(1.94) & 2.74(2.64)\\
         \cline{2-8}
         & \multirow{4}{*}{cWGAN} & 100 & - & - & 26.86(12.34) & 26.38(10.78) & 26.49(10.05)\\
          & & 300 & - & - & 32.66(15.01) & 30.73(12.93) & 32.66(13.90) \\
         & & 500 & - & - & 32.13(19.67) & 30.29(19.07) & 30.08(18.86) \\
         & & 1000 & - & - & 32.40(11.78) & 32.35(10.97) & 34.05(10.56)\\ 
         \hline
         \multirow{9}{*}{M\ref{M3}} & Oracle & 30 & 0.97(0.35) & 0.78(0.16) & 1.07(0.34) & 0.97(0.35) & 1.05(0.37) \\
         \cline{2-8}
         & \multirow{4}{*}{ Proposed } & 100 & 1.19(0.28) & 0.85(0.12) & 1.32(0.27) & 1.19(0.28) &1.27(0.28)\\
          & & 300 &1.15(0.28) & 0.85(0.13) & 1.28(0.26) & 1.15(0.28) & 1.23(0.29)\\
         & & 500 & 1.16(0.31) & 0.83(0.19) & 1.27(0.31) & 1.17(0.32) & 1.27(0.31)\\
         & & 1000 & 1.22(0.36) & 0.85(0.11) & 1.34(0.33) & 1.22(0.36) & 1.31(0.37)\\
         \cline{2-8}
         & \multirow{4}{*}{cWGAN} & 100 & 1.87(0.24) & 0.91(0.04) & 1.99(0.24) & 1.87(0.24) & 1.97(0.24)\\
          & & 300 & 5.53(0.43) & 0.84(0.06) & 5.61(1.36) & 5.53(1.43) & 5.65(1.60)\\
         & & 500 & 8.32(0.90) & 0.85(0.05) & 8.38(0.85) & 8.32(0.90) & 8.47(0.98)\\
         & & 1000 & 17.62(0.68) & 0.88(0.03) & 17.62(0.72) & 17.62(0.72) & 17.82(0.69)\\  
         \hline
         \multirow{9}{*}{M\ref{M4}} & Oracle & 30 & 0.29(0.14) & 0.29(0.07) & 0.56(0.16) & 0.38(0.17) & 0.56(0.18) \\
         \cline{2-8}
         & \multirow{4}{*}{ Proposed } & 100 & 0.40(0.22) & 0.71(0.58) & 1.09(0.44) & 0.53(0.29) & 1.08(0.78) \\
          & & 300 & 0.50(0.24) & 0.56(0.76) & 1.19(0.57) & 0.73(0.39) & 1.05(0.46)\\
         & & 500 & 0.54(0.30) & 0.86(0.62) & 1.21(0.45) & 0.84(0.44) & 1.35(0.55) \\
         & & 1000 & 0.41(0.15) & 0.70(0.73) & 0.93(0.26) & 0.65(0.27) & 1.08(0.33)\\
         \cline{2-8}
         & \multirow{4}{*}{cWGAN} & 100 & 0.57(0.17) & 0.44(0.15) & 1.10(0.35) & 0.91(0.36) & 1.23(0.53) \\
          & & 300 & 0.85(0.56) & 1.58(0.49) & 3.21(0.74) & 0.85(0.44) & 3.06(0.86)\\
         & & 500 & 0.94(0.82) & 8.55(0.67) & 4.33(0.36) & 0.80(0.31) & 4.34(0.38)\\
         & & 1000 & 0.90(0.87) & 8.19(0.34) &  4.19(0.65) & 0.68(0.45) & 3.93(0.32)\\     
      \hline
        \end{tabular}
        \footnotesize
        Note: In M\ref{M2}, the conditional distribution $P_{Y|\bX}$ is a $t$-distribution with degree of freedom 1, and the mean and standard deviation are not well defined.
        \end{minipage}}
        \end{threeparttable}\label{tab:MSEs for d30}
\end{table}

\clearpage
\subsubsection{Simulation results for $p_s =5$}

We consider the same models as in Section \ref{sec:simulation} but set $\beta=(\bm{I}_{5},\bm{I}_{p_c})$. 
The predictors $\bX$ are generated from $N(\bm{0},\bm{I}_{p})$ with $p\in\{100,300,500,1000\}$, and the noise $\bZ\sim N(\bm{0},\bm{I}_{5})$. All methods adopt FNNs with two hidden layers, and the numbers of nodes are set as 64 and 32. 
For M\ref{M1} and M\ref{M3}, $n=1,000$; For M\ref{M2}, $n=10,000$;
For M\ref{M4}, $n=2,000$; And for M\ref{M5} and M\ref{M6}, $n=5,000$. %{\color{red} in the tables, no M5 and M6. need to clarify. if M5 and M6 not done, need to explain}
Table \ref{tab:d5} summarizes the variable selection and prediction results, and Table \ref{tab:MSEs for d5} summarizes the condition distribution estimation results. The observed patterns and superiority of the proposed approach are similar to those for $\bX^{s}\in\bbR^{30}$.

\begin{table}[h]
    \centering
    \caption{Simulation results for variable selection and prediction with $p_s =5$.}
    \begin{threeparttable}
    \resizebox{\textwidth}{!}{\begin{minipage}{\textwidth}
    \begin{tabular}{cc ccc c ccc c ccc c ccc c ccc }
      \hline
      &&\multicolumn{3}{c}{Proposed} & &\multicolumn{3}{c}{P-LS}& & \multicolumn{3}{c}{DFS} & & \multicolumn{3}{c}{LassoNet}& & \multicolumn{3}{c}{GCRNet}\\
      \cline{3-5}\cline{7-9} \cline{11-13} \cline{15-17} \cline{19-21}
      % \cline{4-6}\cline{8-10} \cline{12-14} \cline{15-18} \cline{20-22}
      & $p$ & MSE/C-idx & TPR & FPR && MSE/C-idx & TPR & FPR & & MSE/C-idx & TPR & FPR  && MSE/C-idx & TPR & FPR  && MSE/C-idx & TPR & FPR\\
      \hline
      \multirow{4}{*}{M\ref{M1}} & 100 & 1.08(0.04) & 1.00 & 0.00 &   & 1.04(0.03) & 1.00 & 0.00 & & 1.10(0.07) & 1.00 & 0.00 & &1.03(0.03) & 1.00 & 0.00 & & 1.06(0.03) & 1.00 & 0.00 \\
      & 300 & 1.08(0.07) & 1.00 & 0.00 & & 1.08(0.07) &  1.00 & 0.00 & & 1.22(0.30) & 1.00 & 0.00 & &  1.04(0.04) & 1.00 & 0.00 & & 1.08(0.11) & 1.00 & 0.00   \\
      & 500 & 1.05(0.10) & 1.00 & 0.00 & & 1.11(0.07) &  1.00 & 0.00 & & 1.66(0.64) & 0.94 & 0.00 & & 1.05(0.05) & 1.00 & 0.00 & & 1.12(0.20) & 1.00 & 0.00   \\
      & 1000 & 1.07(0.05) & 1.00 & 0.00 & & 1.11(0.07) &  1.00 & 0.00 & & 1.59(0.65) & 0.93 & 0.00 & & 1.05(0.05) & 1.00 & 0.00 & & 1.10(0.13) & 1.00 & 0.00  \\
      \hline
      \multirow{4}{*}{M\ref{M2}} & 100 & - & 1.00 & 0.00 & & - & 0.56 & 0.48 & & - & 0.10 & 0.05 && - & 1.0 & 0.19 & &  - & 0.20 & 0.04  \\
      & 300 & - & 1.00 & 0.00 & & - & 0.49 & 0.49 & & - & 0.04 & 0.02 && - & 1.0 & 0.20 & & - & 0.16 & 0.19 \\
      & 500 & - & 1.00 & 0.00 & & - & 0.46 & 0.51 & & - & 0.00 & 0.00 && - & 1.0 & 0.14 & &  - & 0.00 & 0.00\\
      & 1000 & - & 0.98 & 0.00 & & - & 0.64 & 0.50 & & - & 0.00 & 0.00 && - & 1.0 & 0.24 & & - & 0.00 & 0.00 \\
      \hline
      \multirow{4}{*}{M\ref{M3}} & 100 & 1.13(0.36) & 1.00 & 0.00 & & 0.54(0.29) &  1.00 & 0.00 && 1.40(0.26) & 1.00 & 0.00 & & 0.60(0.12) & 1.00 & 0.00 & & 0.69(0.34) & 1.00 & 0.00 \\
      & 300 & 1.17(0.32) & 1.00 & 0.00 && 0.64(0.34) & 1.00 & 0.00 && 1.47(0.40) & 1.00 & 0.00 & & 0.72(0.28) & 1.00 & 0.00 & & 0.78(0.47) & 1.00 & 0.00 \\
      & 500 & 1.04(0.29) &  1.00 & 0.00 &&  0.96(0.79) & 1.00 & 0.00 & & 1.69(0.54) & 0.96 & 0.00 & & 0.61(0.08) & 1.00 & 0.00 & &  0.70(0.47) & 1.00 & 0.00\\
      & 1000 & 1.10(0.28) &  1.00 & 0.00 && 0.97(0.94) & 1.00 & 0.00 & & 1.55(0.43) & 0.96 & 0.00 & & 0.66(0.14) & 1.00 & 0.00 & &  0.63(0.39) & 1.00 & 0.00\\
      \hline
      \multirow{4}{*}{M\ref{M4}} & 100 & 2.42(0.17) & 1.00 & 0.00 && 2.09(0.18) & 0.62 & 0.50 & & 2.41(0.22) & 0.04 & 0.05 & & 2.00(0.14) & 1.00 & 0.00 & & 2.19(0.27) & 0.04 & 0.00 \\
      & 300 & 2.59(0.23)&  1.00 & 0.00 && 2.46(0.33) &  0.49 & 0.49 &  & 2.42(0.27) & 0.02 & 0.02 & & 2.11(0.11) & 1.00 & 0.00 & & 2.18(0.19) & 0.02 & 0.00 \\
      & 500 & 2.50(0.16) & 1.00 & 0.00 && 2.42(0.25) & 0.46 & 0.51 & &  2.42(0.24) & 0.04 & 0.01 & & 2.06(0.18) & 1.00 & 0.00 & & 2.78(0.30) & 0.00 & 0.00 \\
      & 1000 & 2.41(0.27) & 1.00 & 0.00 &&  2.65(0.35) & 0.64 & 0.50 & & 2.74(0.24) & 0.04 & 0.00 & & 2.70(0.16) & 1.00 & 0.00 & & 2.74(0.31) & 0.00 & 0.00  \\
    \hline
    \multirow{4}{*}{M\ref{M5}} & 100 & 0.79(0.01) & 1.00 & 0.00 &&  0.79(0.01) & 1.00 & 0.00 & & - & - & - & &0.87(0.01)  & 1.00 & 0.00 & & 0.65(0.01) & 1.00 & 0.36  \\
    & 300 & 0.79(0.01) & 1.00 & 0.00 && 0.79(0.01) & 1.00 & 0.00 & & - & - & - & & 0.83(0.06) & 0.98 & 0.00 & &  0.77(0.10) & 1.00 & 0.30 \\
    & 500 & 0.79(0.01) & 1.00 & 0.00 && 0.79(0.01) & 1.00 & 0.00 & & - & - & - & &  0.77(0.04) & 0.88 & 0.01 & & 0.65(0.01) & 1.00 & 0.07 \\ 
    & 1000 & 0.79(0.01) & 1.00 & 0.00 && 0.79(0.01) & 1.00 & 0.00 & & - & - & - & & 0.74(0.05) & 0.62 & 0.00 & & 0.80(0.06) & 1.00 & 0.03 \\
    \hline
    \multirow{4}{*}{M\ref{M6}} & 100 & 0.73(0.01) & 1.00 & 0.00 && 0.64(0.01) & 1.00 & 0.00 & & - & - & - & & 0.71(0.01) & 0.90 & 0.00 & & 0.60(0.02) & 0.20 & 0.00\\
    & 300 & 0.73(0.01) & 0.98 & 0.00 &&  0.62(0.02) & 0.64 & 0.00 & & - & - & - & &0.71(0.01) & 0.44 & 0.00 &&  0.59(0.01) & 0.20 & 0.00\\
    & 500 & 0.73(0.02) & 1.00 & 0.02 && 0.62(0.01) & 0.52 & 0.00 & & - & - & - & & 0.71(0.02) & 0.40 & 0.00 & & 0.59(0.01) & 0.20 & 0.00 \\ 
    & 1000 & 0.71(0.09) & 1.00 & 0.09 &&  0.60(0.02) & 0.50 & 0.00 & & - & - & - & & 0.70(0.02) & 0.52 & 0.00 & & 0.50(0.03) & 0.00 & 0.00 \\
    \hline
    \end{tabular}
    \footnotesize
    Note: The standard errors for MSE are given in parentheses; With DFS, the number of variables to be selected is set as 5, which is required during the training process.
    \end{minipage}}
    \end{threeparttable}
    \label{tab:d5}
\end{table}

\begin{table}[h]
    \centering
        \caption{Mean squared error (MSE) of the estimated conditional mean, the estimated conditional standard deviation, and the estimated conditional quantile of level $\tau$ for M\ref{M1} to M\ref{M4} with $p_s =5$.}
        \begin{threeparttable}
        \resizebox{\textwidth}{!}{\begin{minipage}{\textwidth}
        \begin{tabular}{ccc| cc ccc}
        \hline
         & & & & & \multicolumn{3}{c}{$\tau$} \\
         \cline{6-8}
        & Method & $p$ & Mean & Sd &  0.25 & 0.50 & 0.75\\
        \hline
        \multirow{9}{*}{M\ref{M1}} & Oracle & 5 & 0.07(0.03) & 0.03(0.03) & 0.10(0.06) & 0.09(0.06) & 0.11(0.07) \\
         \cline{2-8}
         & \multirow{4}{*}{Proposed} & 100 & 0.08(0.03) & 0.04(0.03) & 0.11(0.05) & 0.10(0.04) & 0.12(0.05) \\
          & & 300 & 0.07(0.03) & 0.03(0.01) & 0.08(0.02) & 0.08(0.03) & 0.10(0.04)\\
         & & 500 & 0.06(0.02) & 0.03(0.01) & 0.08(0.02) & 0.07(0.03) & 0.09(0.04)\\
         & & 1000 & 0.06(0.01) & 0.02(0.01) & 0.08(0.02) & 0.08(0.02) & 0.09(0.03)\\
         \cline{2-8}
         & \multirow{4}{*}{cWGAN} & 100 & 0.45(0.08) & 0.39(0.37) & 0.65(0.24) & 0.46(0.08) & 0.62(0.20) \\
          & & 300 & 0.75(0.08) & 0.95(0.02) & 1.20(0.11) & 0.75(0.08) & 1.16(0.14)\\
         & & 500 & 1.25(0.11) & 0.95(0.01) & 1.73(0.20) & 1.25(0.11) & 1.63(0.26)\\
         & & 1000 & 2.35(0.11) & 0.94(0.02) & 2.79(0.34) & 2.35(0.11) & 2.77(0.27)\\     
    \hline 
    \multirow{9}{*}{M\ref{M2}} & Oracle & 5 & - & - &2.69(2.41) & 2.07(1.72) & 2.75(1.44) \\
         \cline{2-8}
         & \multirow{4}{*}{Proposed} & 100 & - & - & 2.28(2.36) & 2.24(2.21) & 2.14(2.23)\\
          & & 300 & - & - & 2.87(2.94) & 2.12(2.06) & 2.67(1.74) \\
         & & 500 & - & - &  3.36(1.96) & 1.90(2.13) & 4.44(3.44)\\
         & & 1000 & - & - & 3.83(1.51) & 1.93(2.17) & 4.46(5.78) \\
         \cline{2-8}
         & \multirow{4}{*}{cWGAN} & 100 & - & - & 4.07(2.05) & 2.76(1.08) & 4.35(2.32)\\
          & & 300 & - & - & 5.10(2.32) & 4.04(2.10) &4.67(2.06) \\
         & & 500 & - & - & 5.87(2.30) & 4.70(2.63) & 5.38(2.74) \\
         & & 1000 & - & - & 5.62(2.72) & 5.50(2.89) & 5.39(4.37) \\   
         \hline
         \multirow{9}{*}{M\ref{M3}} & Oracle & 5 & 0.73(0.23) & 0.71(0.13) & 0.80(0.22) & 0.73(0.23) & 0.83(0.23)\\
         \cline{2-8}
         & \multirow{4}{*}{Proposed} & 100 & 0.73(0.23) & 0.71(0.13) & 0.80(0.22) & 0.73(0.23) & 0.83(0.23) \\
          & & 300 & 0.71(0.22) & 0.78(0.11) & 0.78(0.22) & 0.71(0.22) & 0.78(0.23)\\
         & & 500 & 0.77(0.29) & 0.78(0.23) & 0.89(0.23) & 0.83(0.23) & 0.90(0.24)\\
         & & 1000 & 0.79(0.29) & 0.73(0.13) & 0.99(0.31) & 0.75(0.32) & 0.87(0.32)\\
         \cline{2-8}
         & \multirow{4}{*}{cWGAN} & 100 & 1.37(0.23) & 0.94(0.12) & 1.51(0.24) & 1.37(0.23) & 1.46(0.22) \\
          & & 300 & 3.49(0.99) & 0.95(0.22) & 3.59(0.97) & 3.49(0.99) & 3.61(1.02)\\
         & & 500 & 4.96(0.83) & 0.92(0.18) & 5.01(1.44) & 4.97(1.43) & 5.14(1.44)\\
         & & 1000 & 7.87(0.76) & 0.90(0.33) & 7.83(0.66) & 7.87(0.67) & 8.08(0.71) \\  
         \hline
         \multirow{9}{*}{M\ref{M4}} & Oracle & 5 & 0.03(0.01) & 0.03(0.01) & 0.06(0.03) & 0.04(0.01) & 0.06(0.02)\\
         \cline{2-8}
         & \multirow{4}{*}{Proposed} & 100 & 0.05(0.04) & 0.07(0.09) & 0.13(0.06) & 0.09(0.06) & 0.14(0.07) \\
          & & 300 &0.03(0.01) & 0.11(0.07) & 0.14(0.07) & 0.10(0.05) & 0.12(0.05)\\
         & & 500 & 0.08(0.05) & 0.08(0.05) & 0.12(0.07) & 0.09(0.05) & 0.11(0.06)\\
         & & 1000 & 0.07(0.04) & 0.08(0.09) & 0.12(0.04) & 0.09(0.05) & 0.13(0.05)\\
         \cline{2-8}
         & \multirow{4}{*}{cWGAN} & 100 & 0.08(0.05) & 2.23(0.11) & 1.03(0.07) & 0.13(0.01) & 1.06(0.07) \\
          & & 300 & 0.07(0.05) & 2.44(0.16) & 1.18(0.10) & 0.14(0.02) & 1.13(0.07)\\
         & & 500 & 0.08(0.05) & 2.47(0.13) & 1.18(0.11) & 0.13(0.02) & 1.12(0.09)\\
         & & 1000 & 0.08(0.04) & 2.67(0.09) & 1.27(0.08) & 0.15(0.04) & 1.26(0.14)\\   
      \hline
        \end{tabular}
        \footnotesize
         Note: In M\ref{M2}, the conditional distribution $P_{Y|\bX}$ is a $t$-distribution with degree of freedom 1, and the mean and standard deviation are not well defined.
        \end{minipage}}
        \end{threeparttable}\label{tab:MSEs for d5}
\end{table}

\clearpage
\subsection{Additional details for the data analysis}

\subsubsection{Settings of the neural networks}
For comparability, we adopt similar network structures for all the methods. The neural networks used in the proposed approach, GCRNet, P-LS, LassoNet, and the approximation layers of DFS are two-layer FNNs with widtsh (64, 32) and the ReLu activation function. LassoNet has a single residual connection, and DFS has a selection layer. The noise vector used in the proposed approach is generated from the standard normal distribution with dimension 5. The batch size, learning rate, and value of tuning parameter ${\lambda_n}$ are summarized in Table \ref{tab:network-setting}.

\begin{table}[H]
    \centering
    \caption{Summary of network training settings.}
    \begin{threeparttable}
    \resizebox{\textwidth}{!}
    {\begin{minipage}{\textwidth}
    \begin{tabular}{c|ccc c cccccc c ccc}
    \hline
         & \multicolumn{3}{c}{TCGA} && \multicolumn{6}{c}{MIMIC-III} && \multicolumn{3}{c}{HIV-1}\\
         &&&&& \multicolumn{3}{c}{Albumin} & \multicolumn{3}{c}{Survival} & & & \\
         & batch & lr & ${\lambda_n}$ & & batch & lr & ${\lambda_n}$ &batch & lr & ${\lambda_n}$ &  & batch & lr & ${\lambda_n}$ \\
        \hline
        Proposed & 200 & 1e-4 & 1e-4 && 50 & 1e-4 & 1e-6 & 128 & 5e-5 & 1e-5 && 25 & 1e-4 & 1e-5 \\
        P-LS & 200 & 1e-5 & 0.1 && 50 & 1e-4 & 1 & 128 & 5e-5 & 0.5 && 25 & 0.01 & 0.1 \\
        DFS & 360 & 1e-4 & 1 && 1933 & 0.1 & 1 & - & - & - && - & 0.1 & 1\\
        LassoNet & 360 & 1e-3 & - && 1933 & 1e-3 & - & 8323 & 1e-3 & - && - & 1e-3 & -\\
        GCRNet & 300 & 1e-3 & - && 50 & 0.01 & - & 8323 & 0.001& -&& 50 & 1e-3 & -\\
        \hline
    \end{tabular}
    \footnotesize
    Note: ${\lambda_n}$ in LassoNet and GCRNet are selected by cross-validation in each replication. The batch sizes in DFS and LassoNet vary since the sample sizes of the three drugs are different.
    \end{minipage}}
    \end{threeparttable}
    \label{tab:network-setting}
\end{table}

\subsubsection{Details of the selected predictors}

For the TCGA-LUAD data, Table \ref{tab:TCGA-selection} presents the identified genes using the proposed and alternative methods. Those in bold are selected by multiple methods. No gene is identified by GCRNet. In Table \ref{tab:TCGA-gene}, we present the genes that are frequently selected (which are defined as being selected in at least half of the 30 replications). It is observed that DFS identifies quite different genes across the replications, and GCRNet does not identify any important genes in 27 out of the 30 replications.

\begin{table}[h]
    \centering
    \caption{Analysis of the TCGA-LUAD data: identified important genes. }
    \begin{threeparttable}
    % \resizebox{\textwidth}{!}{\begin{minipage}{\textwidth}
    \begin{tabular}{c|l}
    \hline
      Method   &  Identified Genes\\
    \hline
    \multirow{2}{*}{Proposed}  & {\bf ST6GAL1}, {\bf ISL2}, {\bf OR10G8}, {\bf WDR89}, {\bf OR4N5}, {\bf TRIM16L},\\
    &{\bf C1orf126}, TMEM159, {\bf PIK3C2A}\\
    \hline
    \multirow{2}{*}{P-LS} &  PDE8A, SLC25A21, COL13A1, C18orf54, C1QTNF3, {\bf ST6GAL1}, 
     \\
    & AIG1, MRPL28, RNPC3, {\bf ISL2}, {\bf OR4N5}, DEFB136, {\bf WDR89} \\
    \hline
    \multirow{2}{*}{DFS} &  AAMP, ATIC, C4orf29, CTU2, DCI, FLCN, HSD17B10,\\
    & MADD, MYD88, PRODH, SCAI, SNCG, VSNL1, XAF1   \\
    \hline
    \multirow{3}{*}{LassoNet} &{\bf  ST6GAL1}, {\bf ISL2}, {\bf WDR89}, {\bf OR4N5}, {\bf OR10G8}, {\bf TRIM16L},\\
    & HAPLN2, {\bf C1orf126}, {\bf PIK3C2A}, H1FX, TAF2, ADRA2A, \\
    & KPTN, TMEM63A, PTPRD, JHDM1D, MYL3, LANCL2, TXNDC3 \\
    \hline
    GCRNet & - \\
    \hline
    \end{tabular}
    %\footnotesize
    %Note: No genes are selected by GCRNet. The genes in bold are selected by more than one method.
    % \end{minipage}}
    \end{threeparttable}
    \label{tab:TCGA-selection}
\end{table}

\begin{table}[h]
    \centering
    \caption{Analysis of the TCGA-LUAD data: Frequently identified important genes and their identification frequencies.}
    \begin{threeparttable}
    \resizebox{\textwidth}{!}{\begin{minipage}{\textwidth}
    \begin{tabular}{c|c cc cc cc cc c}
    % \hline
     %  Method & Hugo Symbol of Gene\\
    \hline
    \multirow{2}{*}{Proposed} & Gene & {\bf ISL2} & {\bf OR10G8} & {\bf OR4N5} & SLC47A2& {\bf ST6GAL1} & {\bf WDR89}   \\
        & Frequency & 0.87 & 0.83 & 0.93 & 0.83 & 0.93 & 0.93\\
    \multirow{2}{*}{P-LS} &Gene & AFTPH & C1QTNF3 & {\bf ISL2} & {\bf OR10G8} & {\bf OR4N5}& {\bf ST6GAL1 }& TRIM16L &{\bf  WDR89} \\ 
    & Frequency & 0.50 & 0.67 & 0.83 & 0.50 & 0.73 & 0.83 & 0.77 & 0.80 \\
    \multirow{2}{*}{DFS} & Gene & - \\
    & Frequency & - \\
    \multirow{4}{*}{LassoNet} & Gene &  GLOD4 & H1FX & HAPLN2& {\bf ISL2} & KPTN & LANCL2& MPPED2 & {\bf OR10G8}& {\bf OR4N5}   \\
    & Frequency & 0.57 & 0.60 & 0.70 & 0.90 & 0.77 & 0.50 & 0.50 & 0.77 & 0.90\\
    & Gene &  PIK3C2A & PTPRD& {\bf ST6GAL1} & TAF2& TMEM63A & TRIM16L & VPS13C&{\bf  WDR89 } \\
    & Frequency & 0.57 & 0.77 & 1.00 & 0.87 & 0.90 & 0.60 & 0.90 & 0.90\\
    \multirow{2}{*}{GCRNet} & Gene & - \\
    & Frequency & - \\
    \hline
    \end{tabular}
    %\footnotesize
    %Note: DFS selects different genes across the replications; GCRNet selects no genes over 26 replications.
    \end{minipage}}
    \end{threeparttable}\label{tab:TCGA-gene}
\end{table}

For the analysis of the MIMIC-III data on albumin level,  Table \ref{tab:MIMIC-selection} presents the important predictors identified by each method, and Table \ref{tab:high-MIMIC-selected} presents the frequencies of these predictors identified over the 30 
replications.

\begin{table}[h]
    \centering
        \caption{Analysis of the MIMIC-III data on albumin level: identified important predictors. }
        \label{tab:MIMIC-selection}
    \begin{tabular}{c| ccccc}
    \hline
        Identified predictors & Proposed & P-LS & DFS & LassoNet & GCRNet\\
    \hline
        Number of admissions & \checkmark & \checkmark & \checkmark \\
        White blood cells [max] & \checkmark & \checkmark & \checkmark \\
        Platelets [mean] &  \checkmark & & \checkmark & \checkmark & \checkmark\\
        Platelets [max] & \checkmark\\
        Sodium (135-148) [mean] & & \checkmark \\
        White blood cells (4-11,000) [max] & \checkmark & \checkmark & \checkmark\\
        Temperature C (calc) [max] & & & \checkmark \\
        BUN (6-20) [mean] & & & & \checkmark \\
        BUN [mean] & & & & \checkmark \\
        White blood cells from hematology [mean] & & & & \checkmark  \\
        White blood cells from hematology [max] & \checkmark & \checkmark & \checkmark\\
        Differential-Eos [mean] & & & & \checkmark \\
        Differential-laymphs & \checkmark  & & \checkmark  & \checkmark  & \checkmark \\
        Differential-Polys & & \checkmark\\
        \hline
    \end{tabular}
\end{table} 

\begin{table}[h]
    \centering
    \caption{Analysis of the MIMIC-III data on albumin level: Frequently identified important predictors and their identification frequencies.}
    \label{tab:high-MIMIC-selected}
    \begin{threeparttable}
    \resizebox{\textwidth}{!}{\begin{minipage}{\textwidth}
     \begin{tabular}{c| ccccc}
    \hline
        Identified predictors & Proposed & P-LS & DFS & LassoNet & GCRNet\\
    \hline
        Number of admissions & 0.43 & 0.77 & 0.80 & 0.63 & -  \\
        White blood cells [max] & 1.00 & 1.00 & 0.90 & - & -\\
        Platelets [mean] &  1.00 & 0.87 & 0.97 & 1.00 & 1.00 \\
        Platelets [max] & 0.87 & - & 0.60 & - & -\\
        Sodium (135-148) [mean] & - & 0.57 & 0.37 & - & - \\
        White blood cells (4-11,000) [max] & 1.00 & 1.00 & 0.67 & - &  -\\
        White blood cells from hematology [mean] & 0.23 & - &  - & 0.93 &  -\\
        White blood cells from hematology [max] & 1.00 & 1.00 & 0.93 & -\\
        Differential-laymphs &  0.97 & - & 0.53 & 1.00 & 1.00 \\
        \hline
    \end{tabular}
    \end{minipage}}
    \end{threeparttable}
\end{table}

For the analysis of the MIMIC-III data on survival, Table \ref{tab:MIMIC-III-survival-selection} presents the important predictors identified by each method, and Figure \ref{fig:heatmap-mimic} presents the identification frequencies across the 30 replications. 

\begin{table}[h]
    \centering
     \caption{Analysis of the MIMIC-III data on survival: identified important predictors.
}\label{tab:MIMIC-III-survival-selection}
     \begin{threeparttable} 
     \resizebox{\textwidth}{!}{\begin{minipage}{\textwidth}
    \begin{tabular}{c|l}
    \hline
        Method & Identified predictors \\
        \hline
        \multirow{2}{*}{Proposed} & Monocytes [max], Platelet Count [max], {\bf Bicarbonate} [max], Arterial Blood Pressure systolic [mean], Urea Nitrogen [max], \\
        &  Red Blood Cells [min], Albumin [min], Paw High [min], RDW [max], {\bf Anion Gap} [min], Sodium [min]\\
        \hline
        \multirow{3}{*}{P-LS} & Age, Glucose [min], Respiratory Rate [min], Lactate [mean, min], {\bf Bicarbonate} [max], {\bf Anion Gap} [min], Magnesium [min],\\
        & pO2 [mean], Lactic Acid [min], Non Invasive Blood Pressure systolic [mean], O2 saturation pulseoxymetry [min]\\
        & Chloride, Whole Blood [max], Resp Alarm - High [mean]\\
        \hline 
        DFS &  -\\
        \hline 
    \multirow{3}{*}{LassoNet} & Lactate [mean], O2 saturation pulseoxymetry [min], Urea Nitrogen [min], Heart Rate [min], {\bf Bicarbonate} [mean, max], \\
    & Arterial Blood Pressure systolic [min], {\bf Anion Gap} [min], Glucose [mean], Respiratory Rate (Total) [mean],  \\
    & RDW [min], MCHC [max], Glucose [min], Platelet Count [max], Potassium [max]\\
        \hline
        GCRNet & -\\
        \hline
    \end{tabular}
    %\footnotesize
    Note: DFS is not applicable to survival data. GCRNet does not identify any predictor. Those identified by multiple methods are in bold.
    \end{minipage}}
   \end{threeparttable}
\end{table}

\begin{figure}[h]
    \centering
    \caption{Analysis of the MIMIC-III data on survival: identification frequencies.}
    \includegraphics[width=0.7\textheight]{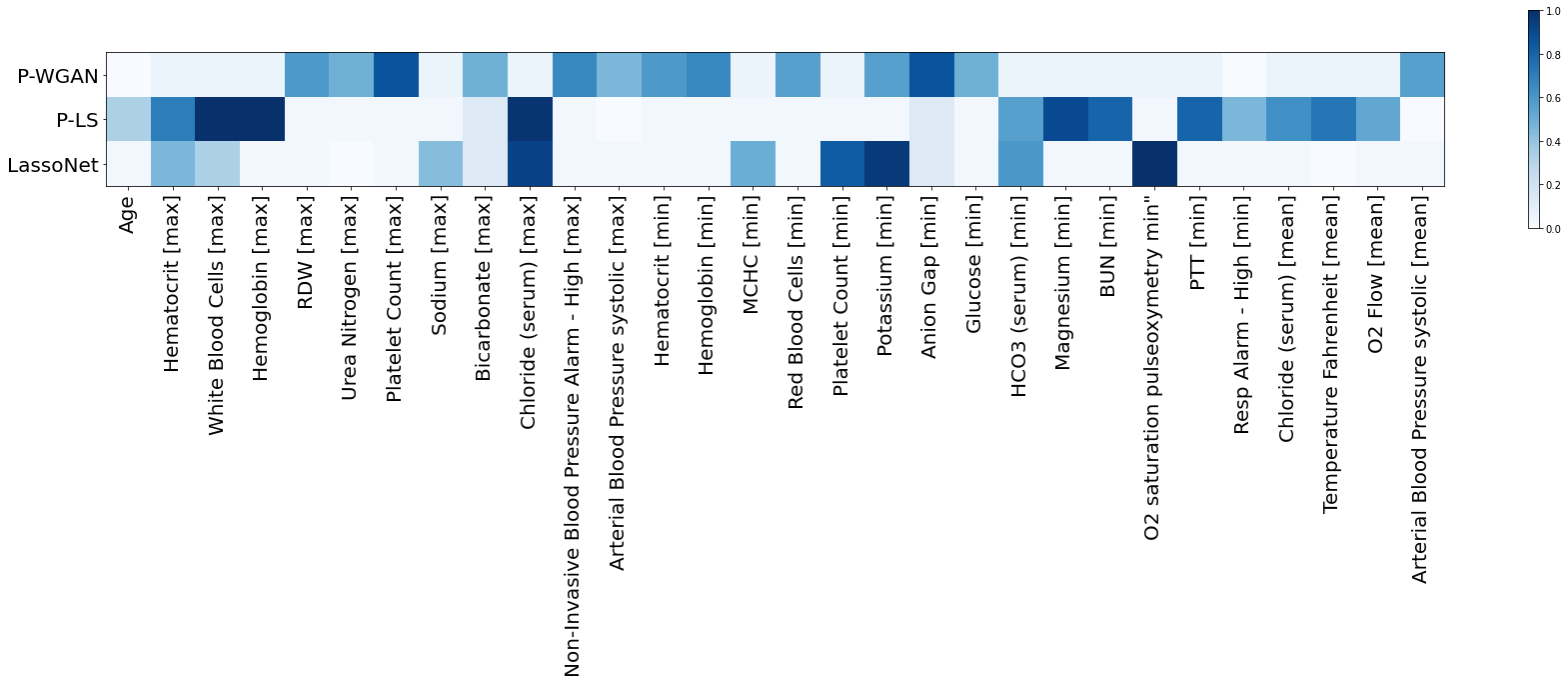}
    \label{fig:heatmap-mimic}
\end{figure}

For the analysis of the HIV-1 data, Table \ref{tab:HIV-selection} presents the mutations identified by each method for APV and SQV. Those in bold are identified in the expert panel in \cite{rhee2006genotypic}. Table \ref{tab:HIV} presents the number of true discoveries and false discoveries of each method with respect to the expert panel. It can be seen that the proposed approach exhibits fewer false discoveries than the alternatives.

\begin{table}[h]
     \caption{Analysis of the HIV dataset: identified important mutations.
}\label{tab:HIV-selection}
     \begin{threeparttable} 
     \resizebox{\textwidth}{!}{%\begin{minipage}{\textwidth}
    \begin{tabular}{cc|l}
    \hline
     Drug   &Method & Identified Mutations \\
        \hline
      \multirow{10}{*}{APV}  &\multirow{2}{*}{Proposed} &   {\bf 10F}, 10I, 10V , 12Z, {\bf 20R}, {\bf20T}, {\bf 33F}, {\bf43T}, {\bf 46L}, {\bf 47V}, {\bf 48M}, {\bf 48V}, \\
      & & {\bf 50L}, {\bf 50V}, {\bf 54M},  64V,  71 , {\bf 82A},  {\bf 82F}, {\bf 84A},  {\bf84C}, {\bf 84V}, {\bf 88S}, {\bf 90M} \\
        \cline{2-3}
        &\multirow{2}{*}{P-LS} &{\bf 50V}, {\bf 88S},  84A,  {\bf 90M}, {\bf 48V}, 20T, {\bf50L}, {\bf 54M}, {\bf 10F},  {\bf 82F}, {\bf 33F},  {\bf 84V},\\ 
        &  & {\bf 54V}, {\bf 47V}, 35D, {\bf 54L},  84C,  {\bf 63P}, {\bf 46I},  76V, 15V, 48M,  83D, 43T, {\bf 32I} \\
        \cline{2-3}
        & \multirow{2}{*}{DFS} &{\bf 10F}, {\bf 10I}, {\bf 10I}, 11I, 20T, {\bf 33F}, 43T, {\bf46I}, {\bf 46L}, {\bf 47V}, {\bf 48V}, {\bf 50L},\\
        & & {\bf 50V}, {\bf 54L}, {\bf 54M}, {\bf 54V}, 58,  76V,  {\bf 82A},  {\bf 82F},  84A,  84C,  {\bf 84V},  88S,  {\bf 90M}  \\
        \cline{2-3}
    &\multirow{2}{*}{LassoNet} & {\bf 10F}, {\bf 10I}, {\bf 10I}, 11I, 20T, {\bf 33F}, 43T, {\bf 46I}, {\bf46L}, {\bf47V}, {\bf48V}, {\bf50L}\\
    & & {\bf50V}, {\bf54L}, {\bf54M},  {\bf54V},   58,  76V,  {\bf82A},  {\bf82F},  84A,  84C,  {\bf84V} \\
        \cline{2-3}
     & \multirow{2}{*}{GCRNet} & {\bf 10F}, {\bf 10I}, {\bf 10I}, {\bf33F}, {\bf46I}, {\bf47V}, 48M, {\bf50L}, {\bf50V}, {\bf54M},  76V,  {\bf82A},\\ 
     & & {\bf82F},  84A,  {\bf84V},  {\bf88S},  89V,  {\bf90M}\\
        \hline
     \multirow{11}{*}{SQV}  &\multirow{2}{*}{Proposed} & 84A, {\bf48V},  84C,  {\bf84V}, {\bf54S},  {\bf90M},  {\bf88D},  {\bf82F}, {\bf24I}, {\bf50L}, {\bf54T}, {\bf53L},\\
      &  & 12P, 76V, {\bf20I}, 48M, {\bf 10I},  {\bf54V},  74S, {\bf36I},  61D, {\bf 10F},  {\bf88S},  {\bf73S} \\
        \cline{2-3}
        &\multirow{2}{*}{P-LS} & 84A, {\bf48V},  84C,  {\bf90M},  {\bf84V}, 48M,  {\bf88D},  76V, {\bf53L}, {\bf54S}, {\bf 10F}, {\bf50L}, \\
        &  & {\bf54V},  67Y,  74S,  {\bf82F}, {\bf20I}, 10IL, {\bf20R}, 35D,  {\bf73S}, {\bf36L}, 85VI, {\bf24I}, {\bf54T}\\
        \cline{2-3}
        &\multirow{2}{*}{DFS} & {\bf 10F}, 10IL,{\bf 20R}, {\bf24I}, 35D, 48M, {\bf48V}, {\bf50L}, {\bf53L}, {\bf54A}, {\bf54S}, {\bf54T}, \\
        &  & 54V,   58,  67Y,  {\bf73S},  74S,  76V,  {\bf82F},  84A,  84C,  {\bf84V},  {\bf88D},  {\bf88S},  {\bf90M} \\
        \cline{2-3}
    &\multirow{2}{*}{LassoNet} & {\bf90M},  {\bf71V}, {\bf 10I},  {\bf84V},  {\bf54V}, {\bf36I}, {\bf 10F}, {\bf48V}, {\bf33F},  {\bf63P},  {\bf88D}, {\bf24I}, \\
    &  & 84A,  {\bf73S}, {\bf20R},  84C, 35D, {\bf53L}, 15V,  {\bf93L}, {\bf20I},  74S, {\bf54S},  76V, 48M \\
        \cline{2-3}
     &   \multirow{3}{*}{GCRNet} & {\bf 10F}, {\bf 10I}, {\bf 10V}, 13V, {\bf24I}, 35D, {\bf36I}, {\bf36L}, 41K, {\bf46I}, 48M, {\bf48V},\\
     &  & {\bf53L}, {\bf54S}, {\bf54T},  {\bf54V},  {\bf62V},  {\bf63P}, 64V,  67Y,  {\bf71V},  {\bf73S},  74S,  {\bf77I}, \\
     &  &  84A,  84C, {\bf84V}, {\bf85VI},  {\bf88D},  {\bf90M},  93L \\
        \hline    
    \end{tabular}}
    %\footnotesize
    %Note: The mutations in bold is the one identified in expert panel in \cite{rhee2006genotypic}.
    %\end{minipage}
   \end{threeparttable}
\end{table}

\begin{table}[t]
    \centering
    \caption{Analysis of the MIMIC-III data on survival: numbers of true (T) and false (F) discoveries, compared against the expert panel.}
    \label{tab:HIV}
    \begin{threeparttable}
    \begin{tabular}{cc cc c cc c cc c cc c cc}
    \hline
    \multirow{2}{*}{Drug}&&\multicolumn{2}{c}{Proposed} & &\multicolumn{2}{c}{P-LS}& & \multicolumn{2}{c}{DFS} & & \multicolumn{2}{c}{LassoNet}& & \multicolumn{2}{c}{GCRNet}\\
         && T & F & &T & F & &T & F & &T & F & &T & F  \\
        \hline
       APV && 19 & 5 & & 16 & 9 & & 17 & 8 & & 16 & 7& & 14 & 4\\
       SQV && 17 & 7 & & 16 & 9 & & 15 & 10 & & 18 & 7 & & 21 & 10\\
       \hline
    \end{tabular}
    %\footnotesize
    %Note: T is the true discovery identified in expert panel in \cite{rhee2006genotypic}; F is the false discovery which is not identified in \cite{rhee2006genotypic}.
    \end{threeparttable}
\end{table}

\end{document}